\documentclass{article}
\PassOptionsToPackage{table,dvipsnames}{xcolor}

\usepackage[preprint]{neurips_2026}
\usepackage{wrapfig}
\usepackage{enumitem}
\usepackage{bbm}
\usepackage{caption}
\usepackage{wrapfig}
\usepackage{subcaption}

\usepackage[utf8]{inputenc}
\usepackage[T1]{fontenc}

\usepackage{amsmath,amssymb,amsfonts,amsthm}
\usepackage{nicefrac}

\usepackage{booktabs}
\usepackage{array}
\usepackage{multirow}
\usepackage{graphicx}

\usepackage[ruled,vlined,linesnumbered]{algorithm2e}
\usepackage{wrapfig}

\makeatletter
\newcommand{\removelatexerror}{\let\@latex@error\@gobble}
\makeatother

\usepackage{microtype}

\usepackage{xcolor}
\definecolor{lightblue}{RGB}{220,235,250}

\usepackage{titletoc}

\usepackage[ruled,vlined,linesnumbered]{algorithm2e}
\SetKwInput{KwIn}{Input}
\SetKwInput{KwOut}{Output}
\SetKw{Return}{return}

\usepackage{url}
\usepackage[colorlinks]{hyperref}
\hypersetup{
  linkcolor=RawSienna,
  urlcolor=ForestGreen,
  citecolor=ForestGreen,
  anchorcolor=ForestGreen
}
\usepackage[nameinlink]{cleveref}

\usepackage[most]{tcolorbox}

\definecolor{lightgray}{RGB}{240,240,245}
\definecolor{darkgreen}{RGB}{0,100,0}

\newtcolorbox[list inside=prompt,auto counter,number within=section]{prompt}[1][]{
  colback=white,
  colbacktitle=black!60,
  coltitle=white,
  fontupper=\footnotesize,
  boxsep=5pt,
  enhanced,
  left=0pt, right=0pt, top=0pt, bottom=0pt,
  boxrule=1pt,
  breakable,
  #1
}

\theoremstyle{definition}

\theoremstyle{plain}
\newtheorem{proposition}{Proposition}
\newtheorem{theorem}{Theorem}
\newtheorem{condition}{Condition}
\newtheorem{lemma}{Lemma}
\newtheorem{corollary}{Corollary}

\theoremstyle{remark}
\newtheorem{remark}{Remark}
\newtheorem{appassumption}{Assumption}
\counterwithin{appassumption}{section}        

\newcommand{\ms}[2]{#1\raisebox{0.15ex}{\tiny$\pm$#2}}

\newcommand{\bms}[2]{\textbf{#1}\raisebox{0.15ex}{\tiny$\pm$#2}}

\title{LARK: Learnability-Grounded Trajectory Selection for Efficient Reasoning Distillation}

\author{%
{\bfseries Tianrun Yu$^{1}$, Kaixiang Zhao$^{1}$, Chih-Chun Chen$^{1}$, Amanda Hughes$^{1}$}\\
{\bfseries Taylor W. Killian$^{1}$, Fenglong Ma$^{2}$, Weitong Zhang$^{3}$, Porter Jenkins$^{1,*}$}\\[0.5em]
{\small
$^{1}$Brigham Young University \quad
$^{2}$The Pennsylvania State University \quad
$^{3}$University of North Carolina at Chapel Hill
}\\[0.4em]
{\small\texttt{*Corresponding author: pjenkins@cs.byu.edu}}
}

\begin{document}
\maketitle

\begin{abstract}
We study trajectory selection for reasoning distillation, where teacher-generated reasoning trajectories are selectively used as supervision for a student model. Existing methods rely on heuristics such as trajectory quality or model confidence, but they often overlook whether a trajectory is learnable by the student. In this paper, we present LARK\footnote{Our code is available at \url{https://github.com/Tianrun-Yu/LARK}.}, a learnability-grounded method for reasoning trajectory selection. LARK selects trajectories that the student can learn efficiently while preserving the generalization of the full training distribution. At the core of LARK is a learnability factor $\rho$, which characterizes the rate at which the student's training loss decreases. To estimate this rate efficiently and maintain generalization, we introduce a learnability proxy and a $\chi^2$-regularized selection policy that balances learnability and distributional coverage, both with strong theoretical guarantees on their estimation error. Empirically, LARK consistently outperforms data selection baselines across multiple base models and reasoning tasks. Diagnostic analyses show that the LARK score predicts downstream training utility and that LARK-selected trajectories induce faster supervised fine-tuning loss reduction.
\end{abstract}


\section{Introduction}
\label{sec:introduction}

Reasoning distillation has emerged as an effective paradigm for transferring chain-of-thought reasoning abilities from larger teacher Large Language Models (LLMs) to smaller student models~\citep{hsieh2023distilling, yuan2023scaling}. In this setting, a teacher model generates reasoning trajectories, and a student model is fine-tuned to imitate them. Recent studies have shown that \emph{a small set of carefully selected reasoning examples} can yield substantial performance gains~\citep{ye2025limo, muennighoff2025s1}, often rivaling the use of much larger training sets. These findings suggest that reasoning distillation requires both \emph{generating} sufficient reasoning data and \emph{identifying} the supervision that is most useful for the student.

Many existing data selection methods for reasoning distillation still inherit the classical formulation of data selection in LLM fine-tuning~\citep{xia2024less,xiao2025unified}, where selection is performed at the question or sample level~\citep{yu2023metamath, zhang2025d3, liu2024selectit}. However, reasoning distillation presents a finer-grained selection problem. For a single question, we often have multiple candidate reasoning trajectories, generated by different teacher models, sampling runs, or reasoning styles. Even when several trajectories lead to the same correct final answer, they may provide very different training signals to the current student model.
A trajectory that appears natural, high-quality, or well aligned with the student model is not necessarily the one that the student can learn from most efficiently. 

Existing trajectory selection methods largely rely on heuristic criteria. Some methods use external verifiers or LLM-as-a-judge scores to assess reasoning quality \citep{zheng2023judging,
yang2025qwen3}, while GRAPE, Local Naturalness, and RSR score candidates using the student model itself~\citep{zhang2025grape, just2025local, yang2026reasoning}.
Although these methods are useful, they do not explicitly measure whether a trajectory is \emph{learnable} by the student. This gap motivates the central question:

\begin{center}
\emph{Can we design a principled trajectory selection criterion that ensures \\ distilled reasoning trajectories are learnable by the student model?}
\end{center}

We answer this question by proposing \textbf{LARK}—\underline{\textbf{L}}earnability-grounded \underline{\textbf{A}}nchor-time \underline{\textbf{R}}an\underline{\textbf{k}}ing—a learnability-grounded method for reasoning trajectory selection, as illustrated in Figure~\ref{fig:lark_pipeline}. LARK identifies the subset of trajectories from which the student model can learn most efficiently. At its core, LARK introduces a principled selection criterion that characterizes trajectory learnability from an optimization perspective, while preserving generalization through $\chi^2$-regularization. Our contributions are as follows:
\begin{itemize}[leftmargin=*]
\item We formulate reasoning trajectory selection as a learnability-grounded policy optimization problem. We introduce the anchor-time learnability rate $\rho$, which characterizes the decay rate of the post-training loss and thereby turns learnability into a principled optimization objective.
\item To estimate the learnability rate $\rho$ efficiently while preventing the selector from hacking the learnability criterion, we use a first-order Taylor expansion around the unselected data distribution and derive a $\chi^2$-regularized policy optimization problem. Theoretically, we show that this policy optimization implicitly increases data learnability and can be solved in closed form under fixed-budget trajectory selection.
\item Empirically, we show that LARK outperforms existing baselines across multiple base models and tasks. Diagnostic analyses further support our theoretical claims and validate the learnability-grounded perspective for trajectory selection.
\end{itemize}

\begin{figure}[!t]
    \centering
    \includegraphics[width=\textwidth]{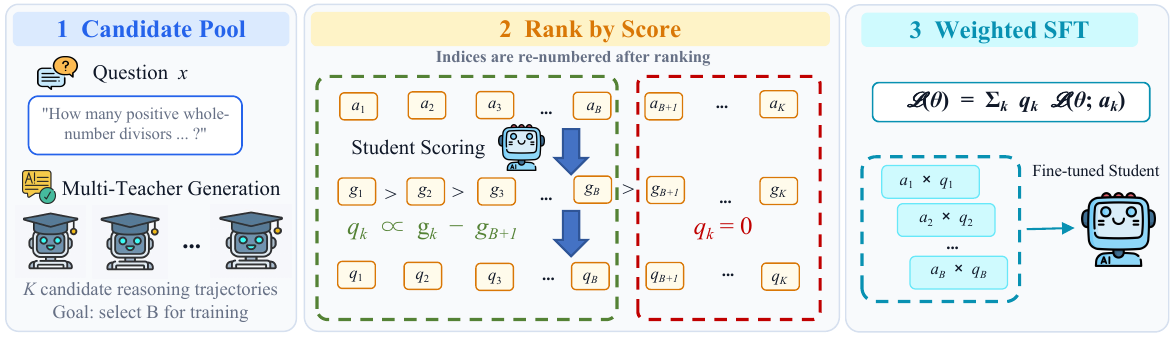}
    \caption{Overview of the LARK pipeline. For each question, multiple
    teacher-generated reasoning trajectories form a candidate pool. LARK uses the
    student model to score these trajectories, ranks them by the score, and selects
    the top-$B$ trajectories for weighted SFT. In the figure, $a_k$ denotes the
    $k$-th trajectory after ranking, $g_k$ denotes the practical LARK score
    $\hat{g}_k$, and $q_k$ denotes its training weight. The indices in the ranking
    panel are re-numbered after sorting.}
    \label{fig:lark_pipeline}
\end{figure}

\paragraph{Notation.}

Vectors are denoted by lowercase boldface letters, e.g., $\mathbf{x}$. For a sequential response, $\mathbf{y}_{<t}$ denotes the sequence of tokens preceding the $t$-th token. We denote by $[n]$ the set $\{1,\cdots,n\}$. For two positive sequences $\{a_n\}_{n\geq 1}$ and $\{b_n\}_{n\geq 1}$, we write $a_n = o(b_n)$ if $a_n/b_n \to 0$ as $n \to \infty$. We denote the $K$-dimensional simplex by $\Delta^K$, and $\|\cdot\|$ denotes the Euclidean norm throughout the paper. For two distributions $p$ and $q$, the $\chi^2$ divergence is defined as $\chi^2(q \parallel p) = \int_x (q(x)-p(x))^2 / p(x) \, \mathrm{d}x$, with the integral replaced by a summation in the discrete case.

\section{Related Work}

\textbf{Reasoning distillation and data selection.}
Reasoning distillation transfers chain-of-thought traces from
stronger teacher models to smaller student models~\citep{hsieh2023distilling,
yuan2023scaling}. Recent work has shown that a small but carefully
chosen set of reasoning examples can rival much larger training
sets~\citep{yu2023metamath,ye2025limo,muennighoff2025s1}. This has
motivated a growing body of work on data-selection methods that filter
examples by diversity, difficulty, uncertainty, or learning
impact~\citep{zhang2025d3,liu2024selectit,li2025limr}. 

Most of these methods operate at the question or example level. In contrast, we study a finer-grained problem that arises in reasoning distillation. For a given question, multiple candidate reasoning trajectories may be available. These trajectories can all lead to correct answers while providing substantially different training signals to the student. LARK therefore asks which trajectory, among candidates for the same question, is most learnable for the current student model.

\textbf{Trajectory selection and optimization-grounded weighting.}
Existing trajectory-selection methods typically score candidates by either
external quality signals or student-side fit. External methods rely on correctness checks, verifiers, or LLM-as-a-judge evaluation to estimate the quality of a reasoning trajectory ~\citep{cobbe2021gsm8k}. Student-side methods instead use the student model's own behavior, such as likelihood (GRAPE~\citep{zhang2025grape}),
local-context likelihood
(Local Naturalness~\citep{just2025local}), or rank-surprisal
alignment (RSR~\citep{yang2026reasoning}), to identify trajectories that appear better aligned with the student. 

LARK differs from both families by deriving its selection criterion from
a local learnability objective. Specifically, it asks whether a
trajectory yields a favorable \emph{optimization signal} for the
current student. This perspective connects our work to gradient-based
data selection rooted in the Polyak--\L{}ojasiewicz
framework~\citep{lojasiewicz1963topological,karimi2016linear,xia2024less},
though unlike those methods we avoid per-trajectory backward passes
by using a forward-pass proxy. For multi-trajectory selection, we
adopt a $\chi^2$-regularized weighting that is related in spirit to
regularized alignment and offline
optimization~\citep{rafailov2023dpo,huang2024correcting,xie2022role},
yielding a closed-form, budgeted rule.

\section{Preliminaries}
\label{sec:prelim}

\subsection{Supervised Fine-Tuning for Reasoning Distillation}
\label{sec:prelim-sft}

We consider the supervised fine-tuning (SFT) of the student model. Let $\mathbf{x}$ be the question and $\mathbf{y}_k{=}(y_1^k, \cdots, y_{|a_k|}^k)^\top$ denote the $k$-th trajectory from the teacher with length $|a_k|$. The SFT loss is defined by
\begin{align}\textstyle{
\ell(\boldsymbol{\theta}, \mathbf{y}_k) = \tfrac{1}{|a_k|}\sum_{t=1}^{|a_k|}
-\log \pi_{\boldsymbol \theta}(y_t \mid \mathbf x,\, \mathbf y_{<t})}, \notag 
\end{align}
where $\pi_{\boldsymbol \theta}$ denotes the student language model with parameter $\boldsymbol{\theta}$. Over a set of response candidates $\{\mathbf y_k\}_{k=1}^K$, we consider the finite-sum loss over a weight $\mathbf q = (q_1, \cdots, q_K)^\top$ written by
\begin{align}\textstyle{
\mathcal L(\boldsymbol{\theta}, \mathbf q) = \sum_{k=1}^K q_k \cdot \ell(\boldsymbol{\theta}, \mathbf y_k),\quad \text{where } \sum_{k=1}^K q_k = 1, q_k \ge 0, \forall k \in [K]},\label{eq:prelim-weighted-sft}
\end{align}
where $\mathbf{q}$ encodes which trajectories should be emphasized during the SFT. For instance, when $q_k = \tfrac1K$ for all $k \in [K]$, Eq.~\eqref{eq:prelim-weighted-sft} reduces to the standard SFT loss over $K$ trajectories. When $B$ trajectories are selected for training, it is equivalent to finding a weight vector $\mathbf{q} $ on the simplex $\Delta^K$ with the number of non-zero elements exactly equal to $B$.

\subsection{Gradient Flow Loss Decay Identity}
\label{sec:prelim-decay}

We characterize the learnability from the lens of optimization theory, in which Polyak--{\L}ojasiewicz (PL) condition of over-parameterized
optimization~\citep{bassily2018exponential, liu2022loss,
chatterjee2022convergence} centers on the ratio of the squared
gradient norm to the loss. We write this ratio as 
\begin{align}
\rho(\boldsymbol \theta; \mathbf{q})
\triangleq
{\|\nabla_{\boldsymbol \theta} \mathcal{L}({\boldsymbol \theta}; \mathbf{q})\|^2} / {\mathcal{L}({\boldsymbol \theta}; \mathbf{q})},
\label{eq:prelim-rho}
\end{align}
which we refer to as the \emph{loss decay rate function}. $\rho(\boldsymbol{\theta}, \mathbf q)$ characterizes the landscape of the optimization process. For instance, when $\rho(\boldsymbol{\theta}, \mathbf q) \ge \mu$, Eq.~\eqref{eq:prelim-rho} implies $\|\nabla_{\boldsymbol \theta} \mathcal{L}\|^2 \ge \mu\mathcal{L}$, which is commonly used in the analysis of over-parameterized optimization~\citep{liu2022loss} (where optimal loss $\mathcal{L}^*{=}0$).

\noindent\textbf{Decay rate of gradient flow.}
Let $T > 0$ denote the total fine-tuning time. For a fixed weight vector $\mathbf{q}$, the \emph{gradient flow} $\phi_{\mathbf{q}}: [0, T] \to \mathbb{R}^{|\boldsymbol \theta|}$ describes the evolution of the parameters when the student is fine-tuned on the weighted objective $\mathcal{L}(\cdot;\, \mathbf{q})$, defined as the solution to the ODE
\begin{align}
\dot\phi_{\mathbf{q}}(s)
=
-\nabla_{\boldsymbol \theta} \mathcal{L}(\phi_{\mathbf{q}}(s); \mathbf{q}),
\qquad
\phi_{\mathbf{q}}(0) = {\boldsymbol \theta}_{\mathrm{ref}},
\qquad s \in [0, T].
\notag 
\end{align}
Here $\phi_{\mathbf{q}}(0) = \theta_{\mathrm{ref}}$ is the starting point of fine-tuning, $\phi_{\mathbf{q}}(T)$ is the parameter at time $T$, and $s$ is the continuous gradient-flow time variable. By the chain rule and Eq.\eqref{eq:prelim-rho}, we have 
\begin{align}
\tfrac{\mathrm d}{\mathrm ds}\mathcal{L}(\phi_{\mathbf{q}}(s); \mathbf{q})
= -\|\nabla_{\boldsymbol{\theta}} \mathcal{L}(\phi_{\mathbf{q}}(s); \mathbf{q})\|^2 = - \rho(\boldsymbol{\theta}, \mathbf q) \mathcal L(\boldsymbol{\theta}, \mathbf q), \quad \phi_{\mathbf{q}}(0) = \theta_{\mathrm{ref}}. \notag 
\end{align}
Solving this ordinary differential equation yields
\begin{align}\textstyle{
\mathcal{L}(\phi_{\mathbf{q}}(T); \mathbf{q})
=
\mathcal{L}(\theta_{\mathrm{ref}}; \mathbf{q}) \cdot
\exp \left(-\int_0^T
\rho(\phi_{\mathbf{q}}(s); \mathbf{q}) \mathrm ds\right),}
\label{eq:prelim-traj-decay}
\end{align}
which implies the integral $\int_0^T
\rho(\phi_{\mathbf{q}}(s); \mathbf{q}) \mathrm ds$ plays an important role in controlling the loss decay.

\section{Proposed Method}
In this section, we present the methodology of LARK. Specifically,
\S\ref{sec:setup} establishes that the time-integrated learnability
objective in Eq.~\eqref{eq:prelim-traj-decay} can be controlled by
the anchor-time learnability rate $\rho(\cdot)$ evaluated at the
initial student model.
\S\ref{sec:method-strategy} then approximately optimizes
$\rho(\boldsymbol{\theta}_{\mathrm{ref}}; \mathbf{q})$ through local
linearization around the uniform prior, develops a forward-pass
proxy $\hat{g}_k$ that estimates the local gradient of $\rho$
without backward propagation, and bounds the residual errors by a
$\chi^2$-regularization term to preserve generalization and prevent
reward hacking of the learnability criterion.
\S\ref{sec:algorithm} combines these pieces into a closed-form,
budget-parameterized selection rule with no hyperparameter to tune.
\subsection{Estimating the Learnability Objective with Anchor-Relative Condition}\label{sec:setup}
As shown in Eq.~\eqref{eq:prelim-traj-decay}, selecting trajectories that are \emph{learnable} by the student model can be formulated as optimizing the time-integrated learnability objective $\max_{\mathbf{q} \in \Delta^K} \int_0^T \rho(\phi_{\mathbf{q}}(s); \mathbf{q})\mathrm ds$, which favors trajectory distributions under which the post-training loss decays rapidly during SFT. However, directly optimizing this objective is infeasible, since it requires tracking the factor $\rho$ along the entire SFT optimization trajectory. To address this challenge, we introduce the following structural condition, which allows the learnability objective to be approximated from a fixed anchor model.
\begin{condition}
\label{ass:anchor-bound}
We say the model satisfies the anchor-relative condition if there exists an absolute constant $\kappa > 0$ such that, for every $\mathbf{q} \in \Delta^K$ and every $s \in [0,T]$,  $\rho(\phi_{\mathbf{q}}(s);\mathbf{q}) \geq \kappa \rho(\boldsymbol{\theta}_{\mathrm{ref}}; \mathbf{q})$.
\end{condition}
Under Condition~\ref{ass:anchor-bound}, the learnability rate evaluated at the initial student parameters $\boldsymbol{\theta}_{\mathrm{ref}}$ provides a tractable proxy for the full time-integrated objective. This anchor-relative condition is consistent with the lazy fine-tuning behavior often observed in over-parameterized LLMs during SFT. In this regime, the model parameters remain close to the initialization $\boldsymbol{\theta}_{\mathrm{ref}}$, and consequently the learnability rate $\rho$ remains close to its anchor-time value $\rho(\boldsymbol{\theta}_{\mathrm{ref}}; \mathbf{q})$. We provide a theoretical justification of this condition under the neural tangent kernel regime~\citep[NTK;][]{jacot2018neural} in Appendix~\ref{app:ntk-justification}, together with empirical validation in Appendix~\ref{app:empirical-verification}.

The following proposition then follows naturally; we defer its proof to Appendix~\ref{app:proof-anchor-control}.
\begin{proposition}
\label{prop:anchor-control}
Under Condition~\ref{ass:anchor-bound}, for every $\mathbf{q} \in \Delta^K$, the post-training loss of the student model satisfies $ \mathcal{L}(\phi_{\mathbf{q}}(T);\mathbf{q}) \leq \mathcal{L}(\boldsymbol{\theta}_{\mathrm{ref}};\mathbf{q}) \cdot \exp\bigl(-\kappa T \cdot \rho(\boldsymbol{\theta}_{\mathrm{ref}}; \mathbf{q})\bigr).$
\end{proposition}
Proposition~\ref{prop:anchor-control} shows that, under the anchor-relative condition, the decay of the SFT loss can be controlled by the learnability rate evaluated at the anchor model $\boldsymbol{\theta}_{\mathrm{ref}}$. Therefore, instead of directly optimizing the intractable trajectory-wide objective $\int_0^T \rho(\phi_{\mathbf{q}}(s); \mathbf{q}) \mathrm ds$, we can optimize the anchor-time surrogate $\max_{\mathbf{q} \in \Delta^K} \rho(\boldsymbol{\theta}_{\mathrm{ref}}; \mathbf{q})$ to encourage faster loss decay for the student model.

\subsection{Approximately Optimizing Learnability $\rho(\boldsymbol{\theta}_{\mathrm{ref}}; \mathbf{q})$ through Local Linearization}\label{sec:method-strategy}
In practice, directly optimizing $\rho(\boldsymbol{\theta}_{\mathrm{ref}}; \mathbf{q})$ is problematic. As shown in Appendix~\ref{app:quasiconvex}, this objective induces a quasi-convex maximization problem over the simplex $\Delta^K$, whose optimum is attained at a vertex of the simplex. Consequently, exact maximization would lead to a degenerate solution that places all probability mass on a single trajectory, thereby ``hacking'' the learnability criterion instead of producing a useful training distribution.

This observation motivates treating the learnability criterion $\rho$ as a \emph{local} indicator around the uniform distribution $\mathbf p = (\tfrac1K, \cdots, \tfrac1K) \in \mathbb R^K$ which corresponds to standard SFT without trajectory selection. Rather than globally maximizing $\rho(\boldsymbol{\theta}_{\mathrm{ref}}; \mathbf{q})$, we restrict attention to distributions $\mathbf{q}$ in a neighborhood of $\mathbf{p}$ and locally linearize the objective around $\mathbf{p}$. The first-order Taylor expansion gives
\begin{align} \rho(\boldsymbol{\theta}_{\mathrm{ref}}; \mathbf{q}) - \rho(\boldsymbol{\theta}_{\mathrm{ref}}; \mathbf{p}) &= \textstyle{\langle \nabla_{\mathbf q} \rho(\boldsymbol{\theta}_{\mathrm{ref}}; \mathbf{q}), \mathbf q - \mathbf p \rangle + o(\|\mathbf q - \mathbf p\|_2)} \notag \\
&= \textstyle{\sum_{k=1}^K} \underbrace{\tfrac{\partial}{\partial q_k} \rho(\boldsymbol{\theta}_{\mathrm{ref}}; \mathbf{q})\mid_{\mathbf q = \mathbf p}}_{g_k^*} \cdot ( q_k -  p_k) + \underbrace{o(\|\mathbf q - \mathbf p\|_2)}_{R_2(\mathbf p, \mathbf q)}. \label{eq:1} 
\end{align}

Two steps remain for optimizing Eq.~\eqref{eq:1}. First, we need to estimate the first-order derivative $g_k^*$ in a computationally efficient manner. Second, we need to bound the residual $R_2(\mathbf p, \mathbf q)$ carefully. 

\subsubsection{Estimating gradient $g_k^*$}
According to the definition of $\rho$, $g_k^* = \tfrac{\partial}{\partial q_k} \rho(\boldsymbol{\theta}_{\mathrm{ref}}; \mathbf{q}) = \tfrac{\partial}{\partial q_k} \tfrac{\|\sum_i q_i \mathbf g_i\|^2}{\sum_i q_i \ell_i}$, where $\mathbf g_i = \nabla_{\boldsymbol{\theta}} \ell_i$ and $\ell_i = \ell(\boldsymbol{\theta}, \mathbf y_i)$. With detailed derivations deferred to Appendix~\ref{app:aux-lemmas}, we can write
\begin{align}
g_k^* = \frac{2\mathbf g_k^\top \sum_iq_i \mathbf g_i}{\sum_i q_i \ell_i} - \frac{\|\sum_iq_i \mathbf g_i\|^2}{\sum_j q_j \ell_k} \cdot \frac{\ell_k}{\sum_i q_i \ell_i} \bigg|_{\mathbf q = \mathbf p} = \frac{2\mathbf g_k^\top \sum_i \mathbf g_i}{\sum_i \ell_i} - \frac{\|\sum_i \mathbf g_i\|^2}{\sum_i \ell_i} \cdot \frac{\ell_k}{\sum_i \ell_i}, \label{eq:2}
\end{align}
where the second equality follows from $q_k = p_k = \tfrac1K$. Based on Eq.~\eqref{eq:2}, we make two observations. First, in many practical overparameterized networks, the per-trajectory gradients $\mathbf g_k$ are nearly orthogonal in high-dimensional parameter space, i.e., $\mathbf g_k^\top \mathbf g_j \approx \delta_{jk} \|\mathbf g_k\|^2$. Second, for the single-trajectory $\rho$-factor defined in Eq.~\eqref{eq:prelim-rho}, denoted by $\rho_k^* = \|\mathbf g_k\|^2 / \ell_k$, the gradient norm can be written as $\|\mathbf g_k\|^2 = \rho_k^* \ell_k$. Plugging these observations into Eq.~\eqref{eq:2}, $g_k^*$ can be approximated by
\begin{align}
    g_k^* \approx \frac{2\|\mathbf g_k\|^2}{\sum_i \ell_i} - \frac{\sum_j \|\mathbf g_j\|^2}{\sum_i \ell_i} \cdot \frac{\ell_k}{\sum_i \ell_i} = \frac{\ell_k}{\sum_i \ell_i} \left(2\rho_k^* - \frac{\sum_i \rho_i^* \ell_i}{\sum_i\ell_i}\right).
\end{align}

In practice, calculating $\rho_k^*$ requires backward propagation to compute $\|\mathbf g_k\|^2$, which is computationally expensive at selection time. Fortunately, we can avoid this gradient computation by leveraging the softmax structure of the last layer in next-token prediction, as stated in the following lemma:
\begin{lemma}[Forward-pass $\rho_k^*$ estimation, informal] \label{lm:1}
For any $k \in [K]$, let $\boldsymbol \pi_t^k \in \Delta^{|\mathcal V|}$ be the probability vector produced by the student model for the $t$-th token, conditioned on the question $\mathbf x$ and the previous response tokens $\mathbf y_{<t}^k$, and let $\boldsymbol \delta(y_t^k) \in \Delta^{|\mathcal V|}$ be the one-hot vector corresponding to the ground-truth token from the teacher model. 
Let 
$\hat \rho_k \triangleq \frac{\sum_t \|\boldsymbol \pi_t^k - \boldsymbol \delta(y_t^k)\|^2}{\ell_k}$,
then under standard structural conditions, there exist absolute positive constants $C_1, C_2$ where $\hat \rho_k \in [C_1 \rho_k^*, C_2 \rho_k^*]$.
\end{lemma}

Lemma~\ref{lm:1} suggests that $\rho_k^*$ can be estimated in a well-controlled manner by the forward-pass quantity $\hat{\rho}_k$, without explicitly computing gradients. We defer the detailed proof to Appendix~\ref{app:proof-lemma1}. Substituting $\hat{\rho}_k$ into Eq.~\eqref{eq:2} yields our practical estimator of the local gradient:
\begin{align}
\hat g_k = \frac{\ell_k}{\sum_i \ell_i} \left(2\hat \rho_k - \frac{\sum_i \hat \rho_i \ell_i}{\sum_i\ell_i}\right), \text{ where }\hat \rho_k \triangleq \frac{\sum_t \|\boldsymbol \pi_t^k - \boldsymbol \delta(y_t^k)\|^2}{\ell_k}.
\end{align}

\subsubsection{Bounding residuals $R(\mathbf p, \mathbf q)$}
We next control the residual errors introduced by the local linearization and the gradient estimation. More precisely, substituting $g_k^*$ with $\hat g_k$ in Eq.~\eqref{eq:1} yields
\begin{align}
    \rho(\boldsymbol{\theta}_{\mathrm{ref}}; \mathbf{q}) - \rho(\boldsymbol{\theta}_{\mathrm{ref}}; \mathbf{p}) = \textstyle{\sum_{k=1}^K} \hat g_k \cdot ( q_k -  p_k) + 
    \underbrace{\textstyle{\sum_{k=1}^K} (g_k^* - \hat g_k) \cdot ( q_k -  p_k)}_{R_1(\mathbf p, \mathbf q)} + \underbrace{o(\|\mathbf q - \mathbf p\|_2)}_{R_2(\mathbf p, \mathbf q)}, \notag 
\end{align}
where $R_1$ captures the error from estimating the true local gradient $g_k^*$ by $\hat g_k$, and $R_2$ captures the higher-order error from the Taylor expansion. The following lemma shows that both terms can be controlled by the $\chi^2$ distance from the uniform distribution.

\begin{lemma}[$\chi^2$ controls both error terms, informal]
\label{lm:chi2_bound}
Under standard structural conditions, there exist constants $\alpha_1, \alpha_2 > 0$ such that for all $\mathbf{q} \in \Delta^K$,
\begin{align}
|R_1(\mathbf{p}, \mathbf{q})| \leq \alpha_1 \chi^2(\mathbf{q}\|\mathbf{p}) + 1, \qquad |R_2(\mathbf{p}, \mathbf{q})| \leq \alpha_2 \chi^2(\mathbf{q}\|\mathbf{p}).
\label{eq:chi2-bound}
\end{align}
\end{lemma}

Lemma~\ref{lm:chi2_bound} shows that the residual errors are controlled by the $\chi^2$ distance between the selected distribution $\mathbf q$ and the uniform distribution $\mathbf p$, with explicit forward-computable constants $\alpha_1, \alpha_2$ given in Appendix~\ref{app:proof-lemma2}. This further supports our intuition that optimization should be performed in a local neighborhood around the original, non-selected data distribution $\mathbf p$, to preserve generalization and prevent the selector from drifting toward degenerate solutions. Combining the linearized objective with the residual bounds gives the following lower bound on the original learnability improvement; we defer its proof to Appendix~\ref{app:proof-thm1}.

\begin{theorem}[LARK objective, informal]
\label{thm:lark-objective}
There exist absolute constants $\alpha > 0$ and $C \in \mathbb R$ such that 
\begin{align}
\rho(\boldsymbol{\theta}_{\mathrm{ref}}; \mathbf{q}) - \rho(\boldsymbol{\theta}_{\mathrm{ref}}; \mathbf{p}) \ge \textstyle{\sum_{k=1}^K} \hat g_k \cdot ( q_k -  p_k) - \alpha \chi^2(\mathbf q \parallel \mathbf p) + C.
\end{align}
\end{theorem}

Theorem~\ref{thm:lark-objective} reduces the intractable learnability improvement to a tractable surrogate that depends on $\mathbf q$ only through (i) the inner product $\langle \mathbf q - \mathbf p,\, \hat{\mathbf g}\rangle$ between the trajectory deviation and the forward-pass score vector, and (ii) the $\chi^2$ distance to the uniform prior. The constant $C$ is independent of $\mathbf q$ and therefore does not affect the maximizer over $\Delta^K$. This reduction sets the stage for the closed-form selection rule developed in \S\ref{sec:algorithm}.
\subsection{LARK: Closed-form Solution for Fixed-Budget Trajectory Selection}
\label{sec:algorithm}

Theorem~\ref{thm:lark-objective} reduces the intractable learnability
improvement to a tractable $\chi^2$-regularized optimization based on
$\{\hat g_k\}$. In practice, many data selection algorithms operate in
the fixed-budget setting, where the agent needs to select $B \le K$
trajectories per question, meaning that the support size of $\mathbf q$
is $B$. This fixed-budget optimization admits a closed-form
solution, as stated in the following lemma.

\begin{lemma}[$B$-parameterized closed form]
\label{prop:b_param}
Sort $\hat g_1 \geq \cdots \geq \hat g_K$ and assume $1 \leq B < K$ with $\hat g_B > \hat g_{B+1}$. 
Under the uniform prior $p_k = 1/K$, the maximizer of optimization objective 
\begin{align}
\hat{\mathbf{q}} = \arg\textstyle{\max_{\mathbf{q}\in\Delta^K, \mathrm{supp}(\mathbf q) = B}} \left\{\textstyle\sum\nolimits_{k=1}^{K} \hat g_k(q_k - p_k) - \tfrac{\tau(B)}{2}\chi^2(\mathbf{q}\|\mathbf{p}) \right\},
\end{align}
with regularizer $\tau(B) = \tfrac{1}{K}\sum_{j=1}^{B}(\hat g_j - \hat g_{B+1})$ takes the closed form $\hat q_i = (\hat g_i - \hat g_{B+1}) / \sum_{j=1}^{B}(\hat g_j - \hat g_{B+1})$ for $i \leq B$, and $\hat q_i = 0$ otherwise.
\end{lemma}

\setlength{\columnsep}{18pt}
\setlength{\intextsep}{5pt}
\begin{wrapfigure}[24]{r}{0.47\textwidth}
\vspace{-\intextsep}
\removelatexerror
\begin{algorithm}[H]
\small
\DontPrintSemicolon
\SetAlCapFnt{\small\bfseries}
\SetInd{0.4em}{0.7em}
\setlength{\algomargin}{1.5em}
\caption{LARK: Learnability-grounded Anchor-time Ranking}
\label{alg:lark}
\KwIn{Candidates $\{\mathbf{y}_k\}_{k=1}^K$; reference $\pi_{\mathrm{ref}}$
with $\boldsymbol{\theta}_{\mathrm{ref}}$; budget $B$.}
\KwOut{
$\mathcal{L}(\boldsymbol{\theta};\hat{\mathbf{q}})
=\textstyle\sum\nolimits_{k=1}^{K}\hat{q}_k\,\ell(\boldsymbol{\theta};\mathbf{y}_k)$.}

\For{$k=1,\ldots,K$}{
  Forward pass of $\pi_{\mathrm{ref}}$ on $\mathbf{y}_k$; record
  $\boldsymbol{\pi}_t^k,p_t$ for $t\in\{1,\ldots,|a_k|\}$\;
  $\ell_k \leftarrow \tfrac{1}{|a_k|}\textstyle\sum\nolimits_t -\log p_t$\;
  $\hat{\rho}_k \leftarrow
   \dfrac{\sum_t\|\boldsymbol{\pi}_t^k-\boldsymbol{\delta}(y_t^k)\|^{2}}
         {\sum_t -\log p_t}$\;
}

\For{$k=1,\ldots,K$}{
  $\hat{g}_k \leftarrow
  \dfrac{\ell_k}{\sum_i \ell_i}
  \!\left(
  2\hat{\rho}_k -
  \dfrac{\sum_i \hat{\rho}_i\,\ell_i}{\sum_i \ell_i}
  \right)$\;
}

Sort and relabel: $\hat g_1 \ge \cdots \ge \hat g_K$\;

\For{$i=1,\ldots,K$}{
  \eIf{$i\leq B$}{
    $\hat{q}_i \leftarrow
    \dfrac{\hat{g}_i-\hat{g}_{B+1}}
          {\sum_{j=1}^{B}(\hat{g}_j-\hat{g}_{B+1})}$\;
  }{
    $\hat{q}_i \leftarrow 0$\;
  }
}

\Return $\mathcal{L}(\boldsymbol{\theta};\hat{\mathbf{q}})
=\textstyle\sum\nolimits_{k=1}^{K}\hat{q}_k\,\ell(\boldsymbol{\theta};\mathbf{y}_k)$\;
\end{algorithm}
\vspace{-\intextsep}
\end{wrapfigure}

Lemma~\ref{prop:b_param} gives a closed-form top-$B$ rule:
each selected trajectory is weighted by its score margin above
$\hat g_{B+1}$, and $\tau^*(B)$ is fully determined by $B$ with no
hyperparameter to tune. Algorithm~\ref{alg:lark} summarizes the full
procedure.

\begin{remark}[Soft vs.\ hard top-$B$]
\label{rem:soft-vs-hard}
LARK is a \emph{soft} top-$B$ rule: weights are proportional to score
margins above the threshold, rather than uniform $1/B$ as in hard
top-$B$ truncation. This hedges against ranking noise near the
boundary; the ablation in \S\ref{sec:rq2_ablation}
(Figure~\ref{fig:ablation_q2}) shows it outperforms uniform,
proportional, and softmax weighting.
\end{remark}

\begin{remark}[Heuristic reading of $\hat g_k$]
\label{rem:ghat-meaning}
$\hat g_k$ favors trajectories on which the student is
\emph{confidently wrong in a structured way}: $\ell_k$ is not small,
indicating that the student does not yet predict it well, yet the
residual mass $\|\boldsymbol\pi_t^k - \boldsymbol\delta(y_t^k)\|^2$
concentrates on a few near-miss tokens, providing a clear local
correction signal. Although $\hat g_k$ is a forward-pass surrogate
for the exact directional derivative $g_k^*$, the two remain tightly
aligned: Appendix~\ref{app:ghat-empirical} gives an explicit error
bound on $\|\hat{\mathbf g} - \mathbf g^*\|_{\mathbf p}$ and
verifies it on real student models, while the token-level signature
is verified in \S\ref{sec:rq4_why} (Figure~\ref{fig:token_level}).
\end{remark}

\section{Experiment}
\label{sec:experiments}

We evaluate LARK on reasoning distillation through three research
questions.
\textbf{RQ1} asks whether LARK outperforms heuristic, quality-based,
and student-side baselines under single-trajectory ($B{=}1$) and
multi-trajectory ($B{=}3$) supervision.
\textbf{RQ2} isolates the contribution of the trajectory score
$\hat g_k$ and the $\chi^2$-$B$ weighting rule by holding one fixed
and varying the other.
\textbf{RQ3} asks \emph{why} LARK works, by linking $\hat g_k$ to the
anchor-time learnability rate $\rho$, to the SFT loss trajectory, and
to downstream training utility, and by analyzing the token-level
structure of the trajectories LARK selects.
Section~\ref{sec:exp_setup} describes the shared experimental setup,
and Sections~\ref{sec:rq1_main}--\ref{sec:rq3_why} answer each
question in turn.

\subsection{Experimental Setup}
\label{sec:exp_setup}

\noindent\textbf{Candidate pool.}
We use a fixed set of $5{,}000$ math problems from
NuminaMath~\citep{li2024numinamath}, each paired with $33$
trajectories from $11$ teacher models ($3$ rollouts each); see
Appendix~\ref{app:exp_pool}. All methods share this pool, so any
performance difference reflects the selection rule alone.

\noindent\textbf{Selection budgets.}
We evaluate each method at $B{=}1$ (single trajectory per problem)
and $B{=}3$ (three trajectories). Baselines take the top-$B$ under
their score with uniform weights $1/B$; LARK applies the
$\chi^2$-$B$ rule of Lemma~\ref{prop:b_param} to obtain
weighted $\hat{\mathbf{q}}$ from $\hat g_k$ computed at the
reference student $\boldsymbol\theta_{\mathrm{ref}}$.

\noindent\textbf{Baselines.}
We compare LARK against seven methods (Appendix~\ref{app:exp_baselines})
spanning heuristics, quality-based selection, LLM-based evaluation,
and student-side forward-pass scoring: Random, Token Length$_{\max}$,
Rule-based Quality$_{\max}$, LLM-judged Quality$_{\max}$, GRAPE$_{\max}$~\citep{zhang2025grape}, Local
Naturalness$_{\max}$~\citep{just2025local}, and
RSR$_{\min}$~\citep{yang2026reasoning}.

\begin{table*}[t!]
\centering
\footnotesize
\caption{
Performance comparison under single-trajectory and multi-trajectory selection.
Values are Acc@5 percentages, with mean$\pm$std taken across three independent
decoding seeds at inference time.
}
\label{tab:main_weighted_math_merged}
\setlength{\tabcolsep}{3pt}

\resizebox{\textwidth}{!}{
\begin{tabular}{cl|ccccc|ccccc}
\toprule
\multirow{2}{*}{\textbf{Model}}
& \multirow{2}{*}{\textbf{Method}}
& \multicolumn{5}{c|}{\textbf{Budget $B=1$}}
& \multicolumn{5}{c}{\textbf{Budget $B=3$}} \\
\cmidrule(lr){3-7}
\cmidrule(lr){8-12}
& & \textbf{AIME} & \textbf{AMC} & \textbf{GPQA} & \textbf{MATH-500} & \textbf{Avg}
& \textbf{AIME} & \textbf{AMC} & \textbf{GPQA} & \textbf{MATH-500} & \textbf{Avg} \\
\midrule

\multirow{9}{*}{\rotatebox[origin=c]{90}{\textbf{Qwen-2.5-7B}}}
& BaseModel
  & \ms{18.89}{1.92} & \ms{47.79}{4.23} & \ms{39.56}{3.29} & \ms{56.47}{4.80} & \ms{40.68}{3.51}
  & \ms{18.89}{1.92} & \ms{47.79}{4.23} & \ms{39.56}{3.29} & \ms{56.47}{4.80} & \ms{40.68}{3.51} \\
& Random
  & \ms{21.11}{3.85} & \ms{51.00}{3.03} & \ms{56.90}{3.55} & \ms{64.67}{4.41} & \ms{48.42}{3.58}
  & \ms{24.44}{3.85} & \ms{61.45}{3.61} & \ms{59.76}{3.29} & \ms{69.67}{4.20} & \ms{53.83}{3.64} \\
& Token Length$_{\max}$
  & \ms{13.33}{3.33} & \ms{50.60}{3.61} & \ms{41.08}{3.04} & \ms{60.93}{4.10} & \ms{41.49}{3.52}
  & \ms{20.00}{3.33} & \ms{57.03}{4.23} & \ms{43.77}{3.29} & \ms{65.73}{4.40} & \ms{46.63}{3.81} \\
& Rule-based Quality$_{\max}$
  & \ms{21.11}{1.92} & \ms{56.63}{3.61} & \ms{48.15}{3.04} & \ms{65.27}{3.70} & \ms{47.79}{3.01}
  & \ms{28.89}{1.92} & \ms{67.07}{3.03} & \ms{53.54}{3.54} & \ms{70.73}{4.10} & \ms{55.06}{3.10} \\
& LLM-judged Quality$_{\max}$
  & \ms{22.22}{1.92} & \ms{56.22}{4.23} & \ms{50.34}{3.29} & \ms{65.80}{3.90} & \ms{48.65}{3.29}
  & \ms{26.67}{3.33} & \ms{62.25}{3.03} & \ms{53.03}{3.54} & \ms{69.60}{3.70} & \ms{52.89}{3.40} \\
& GRAPE$_{\max}$
  & \ms{16.67}{3.33} & \ms{55.02}{3.03} & \ms{49.33}{3.79} & \ms{64.67}{3.90} & \ms{46.42}{3.51}
  & \ms{25.56}{1.92} & \ms{63.05}{3.03} & \ms{55.05}{3.54} & \ms{71.00}{4.20} & \ms{53.66}{3.11} \\
& Local Naturalness$_{\max}$
  & \ms{23.33}{3.33} & \ms{57.83}{3.61} & \ms{51.18}{3.55} & \ms{65.80}{3.80} & \ms{49.54}{3.57}
  & \ms{27.78}{1.92} & \ms{64.26}{3.03} & \ms{56.23}{3.29} & \ms{70.87}{3.81} & \ms{54.78}{2.97} \\
& RSR$_{\min}$
  & \ms{24.44}{3.85} & \ms{59.44}{3.03} & \ms{52.19}{3.04} & \ms{67.33}{3.90} & \ms{50.85}{3.35}
  & \ms{32.22}{3.85} & \ms{66.27}{3.61} & \ms{56.73}{3.29} & \ms{71.47}{3.90} & \ms{56.67}{3.57} \\
\rowcolor{lightblue!100}
& \textbf{LARK (Ours)}
  & \bms{30.00}{3.33} & \bms{65.86}{3.03} & \bms{65.99}{3.55} & \bms{70.20}{3.82} & \bms{58.01}{3.42}
  & \bms{36.67}{3.33} & \bms{76.31}{3.03} & \bms{67.34}{3.29} & \bms{74.67}{3.45} & \bms{63.74}{3.27} \\

\midrule

\multirow{9}{*}{\rotatebox[origin=c]{90}{\textbf{Llama-3.2-3B}}}
& BaseModel
  & \ms{2.22}{1.92} & \ms{11.24}{4.23} & \ms{11.45}{2.04} & \ms{18.13}{3.21} & \ms{10.76}{2.78}
  & \ms{2.22}{1.92} & \ms{11.24}{4.23} & \ms{11.45}{2.04} & \ms{18.13}{3.21} & \ms{10.76}{2.78} \\
& Random
  & \ms{4.44}{3.85} & \ms{14.86}{3.68} & \ms{40.40}{2.81} & \ms{24.13}{3.45} & \ms{20.96}{3.35}
  & \ms{3.33}{3.33} & \ms{17.67}{3.03} & \ms{44.61}{3.29} & \ms{27.07}{3.80} & \ms{23.17}{3.36} \\
& Token Length$_{\max}$
  & \ms{3.33}{3.33} & \ms{16.47}{2.51} & \ms{38.22}{3.55} & \ms{24.53}{4.03} & \ms{20.64}{3.34}
  & \ms{2.22}{1.92} & \ms{16.06}{3.03} & \ms{38.72}{3.55} & \ms{25.07}{3.84} & \ms{20.52}{3.00} \\
& Rule-based Quality$_{\max}$
  & \ms{2.22}{3.85} & \ms{10.44}{3.03} & \ms{14.48}{2.87} & \ms{19.60}{3.93} & \ms{11.69}{3.29}
  & \ms{5.56}{3.85} & \ms{15.66}{3.61} & \ms{23.57}{2.78} & \ms{25.73}{3.23} & \ms{17.63}{3.28} \\
& LLM-judged Quality$_{\max}$
  & \ms{4.44}{1.92} & \ms{18.88}{3.68} & \ms{41.58}{3.04} & \ms{25.93}{3.51} & \ms{22.71}{2.98}
  & \ms{3.33}{3.33} & \ms{18.47}{3.03} & \ms{43.43}{3.03} & \ms{27.00}{3.50} & \ms{23.06}{3.22} \\
& GRAPE$_{\max}$
  & \ms{2.22}{1.92} & \ms{14.46}{3.19} & \ms{14.31}{2.78} & \ms{21.80}{3.90} & \ms{13.20}{2.84}
  & \ms{6.67}{3.33} & \ms{20.48}{3.19} & \ms{21.55}{2.28} & \ms{24.87}{3.80} & \ms{18.39}{3.11} \\
& Local Naturalness$_{\max}$
  & \ms{6.67}{3.33} & \ms{17.67}{4.23} & \ms{41.08}{2.28} & \ms{25.60}{3.22} & \ms{22.75}{3.26}
  & \ms{6.67}{3.33} & \ms{20.08}{3.03} & \ms{42.76}{3.04} & \ms{26.60}{3.34} & \ms{24.03}{3.17} \\
& RSR$_{\min}$
  & \ms{6.67}{3.33} & \ms{15.66}{2.41} & \ms{42.09}{2.54} & \ms{25.73}{3.81} & \ms{22.54}{3.02}
  & \ms{5.56}{1.92} & \ms{17.27}{3.03} & \ms{43.77}{3.04} & \ms{27.87}{3.92} & \ms{23.62}{2.90} \\
\rowcolor{lightblue!100}
& \textbf{LARK (Ours)}
  & \bms{8.89}{1.92} & \bms{19.28}{2.41} & \bms{48.32}{3.04} & \bms{28.67}{3.50} & \bms{26.29}{2.66}
  & \bms{8.89}{1.92} & \bms{22.09}{2.51} & \bms{54.55}{3.54} & \bms{31.00}{3.22} & \bms{29.13}{2.72} \\

\midrule

\multirow{9}{*}{\rotatebox[origin=c]{90}{\textbf{Qwen-2.5-1.5B}}}
& BaseModel
  & \ms{0.00}{0.00} & \ms{24.50}{4.23} & \ms{30.64}{2.04} & \ms{38.53}{3.21} & \ms{23.42}{2.36}
  & \ms{0.00}{0.00} & \ms{24.50}{4.23} & \ms{30.64}{2.04} & \ms{38.53}{3.21} & \ms{23.42}{2.36} \\
& Random
  & \ms{1.11}{1.92} & \ms{33.33}{2.51} & \ms{54.38}{3.55} & \ms{42.33}{4.03} & \ms{32.79}{2.96}
  & \ms{2.22}{1.92} & \ms{32.93}{4.23} & \ms{55.89}{3.04} & \ms{48.53}{4.20} & \ms{34.89}{3.28} \\
& Token Length$_{\max}$
  & \ms{0.00}{0.00} & \ms{25.30}{3.19} & \ms{34.51}{2.78} & \ms{42.40}{3.90} & \ms{25.55}{2.45}
  & \ms{2.22}{1.92} & \ms{25.70}{3.68} & \ms{42.09}{2.59} & \ms{44.20}{4.01} & \ms{28.55}{2.96} \\
& Rule-based Quality$_{\max}$
  & \ms{5.56}{1.92} & \ms{29.72}{3.03} & \ms{41.25}{2.87} & \ms{43.40}{3.93} & \ms{29.98}{2.85}
  & \ms{6.67}{3.33} & \ms{38.55}{3.19} & \ms{46.30}{2.28} & \ms{46.87}{3.80} & \ms{34.60}{3.11} \\
& LLM-judged Quality$_{\max}$
  & \ms{3.33}{3.33} & \ms{29.32}{2.51} & \ms{42.93}{3.31} & \ms{44.80}{3.82} & \ms{30.09}{3.21}
  & \ms{3.33}{3.33} & \ms{35.34}{3.03} & \ms{46.30}{3.04} & \ms{47.47}{3.45} & \ms{33.11}{3.21} \\
& GRAPE$_{\max}$
  & \ms{2.22}{1.92} & \ms{35.34}{3.68} & \ms{42.93}{2.81} & \ms{45.73}{3.45} & \ms{31.56}{2.90}
  & \ms{5.56}{1.92} & \ms{38.15}{1.84} & \ms{46.63}{3.09} & \ms{48.80}{3.74} & \ms{34.79}{2.57} \\
& Local Naturalness$_{\max}$
  & \ms{6.67}{3.33} & \ms{36.14}{2.41} & \ms{45.62}{2.54} & \ms{45.73}{3.81} & \ms{33.54}{3.02}
  & \ms{6.67}{3.33} & \ms{38.96}{4.23} & \ms{46.80}{2.54} & \ms{48.93}{3.61} & \ms{35.34}{3.42} \\
& RSR$_{\min}$
  & \ms{4.44}{1.92} & \ms{32.13}{3.68} & \ms{47.14}{3.04} & \ms{45.93}{3.51} & \ms{32.41}{2.98}
  & \ms{8.89}{3.85} & \ms{36.14}{3.61} & \ms{46.30}{2.78} & \ms{47.73}{3.23} & \ms{34.77}{3.28} \\
\rowcolor{lightblue!100}
& \textbf{LARK (Ours)}
  & \bms{7.78}{1.92} & \bms{40.56}{4.23} & \bms{63.13}{2.53} & \bms{49.33}{3.50} & \bms{41.09}{2.98}
  & \bms{10.00}{3.33} & \bms{42.57}{3.03} & \bms{66.67}{3.03} & \bms{53.33}{3.40} & \bms{43.14}{3.20} \\

\bottomrule
\end{tabular}}
\end{table*}

\noindent\textbf{Students, training, evaluation.}
We fine-tune \textbf{Qwen-2.5-7B}, \textbf{Qwen-2.5-1.5B}~\citep{yang2024qwen2},
and \textbf{Llama-3.2-3B}~\citep{grattafiori2024llama} via SFT, and
evaluate on AIME-2024~\citep{aops_aime}, AMC~\citep{maa_amc},
GPQA-Diamond~\citep{rein2024gpqa}, and
MATH-500~\citep{hendrycks2021measuring} using Acc@5 ($5$ samples per
problem, marked correct if any contains the gold answer). For each
method, we run inference under three independent decoding seeds using
the corresponding trained checkpoint and evaluation protocol for that
method; we report mean$\pm$std across these three decoding seeds, so
the reported variability reflects inference-time sampling rather than
training-seed noise. The Avg column in Table~\ref{tab:main_weighted_math_merged}
is the unweighted mean across the four benchmarks; per-benchmark
numbers are the primary basis of comparison. A train/eval
contamination audit appears in Appendix~\ref{app:contamination}.
Training and decoding details are in
Appendices~\ref{app:exp_impl}--\ref{app:exp_eval}.


\subsection{RQ1: Does LARK Outperform Existing Trajectory Selection Methods?}
\label{sec:rq1_main}

\begin{wrapfigure}{r}{0.49\textwidth}
\vspace{-0.4em}
\setlength{\intextsep}{0pt}
\setlength{\columnsep}{8pt}
\centering
\begin{subfigure}[t]{0.49\linewidth}
  \centering
  \includegraphics[width=\linewidth]{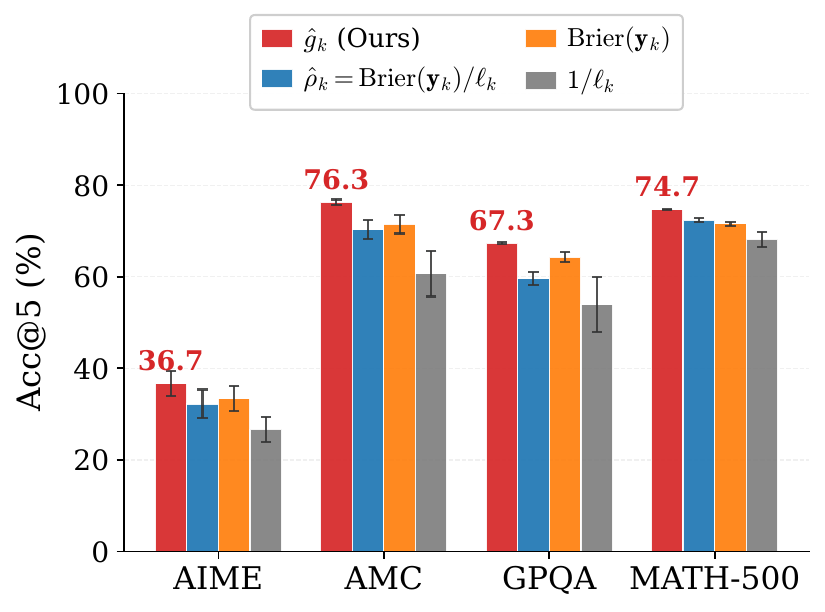}
  \caption{}
  \label{fig:ablation_q1}
\end{subfigure}
\hfill
\begin{subfigure}[t]{0.49\linewidth}
  \centering
  \includegraphics[width=\linewidth]{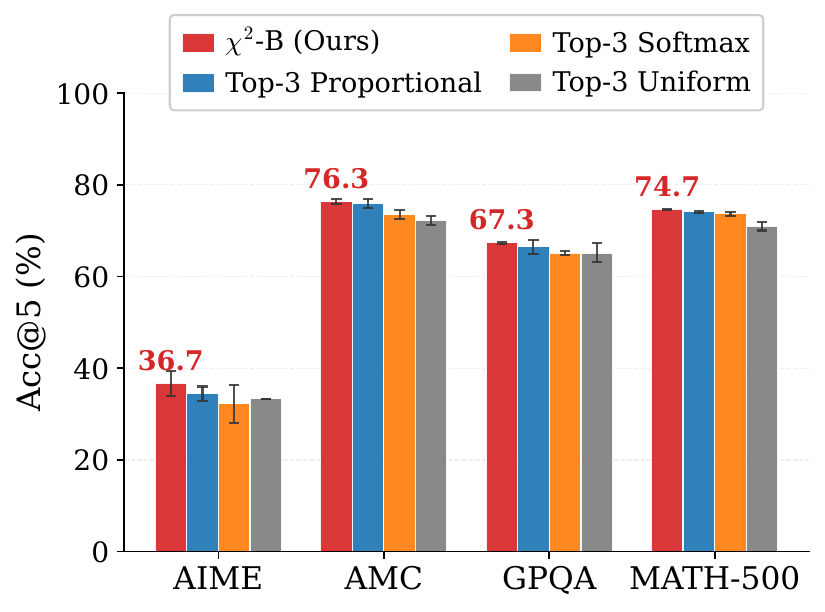}
  \caption{}
  \label{fig:ablation_q2}
\end{subfigure}
\caption{\textbf{Ablation of the two LARK components} on Qwen-2.5-7B
($B=3$). Values are Acc@5 percentages.
\textbf{(a)} Score ablation.
\textbf{(b)} Weighting ablation.}
\label{fig:ablation_q1q2}
\vspace{-0.4em}
\end{wrapfigure}

Table~\ref{tab:main_weighted_math_merged} compares LARK with all
baselines on three students. LARK achieves the best average Acc@5 in
all six model-budget settings. On Qwen-2.5-7B, it improves over the
strongest baseline RSR by $7.16$ points at $B{=}1$ and $7.07$ points
at $B{=}3$, with large gains on AMC and GPQA. The consistent gains
under both budgets suggest that LARK improves trajectory quality, not
only supervision size. The same trend holds for smaller students:
LARK reaches $41.09\%$ / $43.14\%$ on Qwen-2.5-1.5B and
$26.29\%$ / $29.13\%$ on Llama-3.2-3B under $B{=}1$ / $B{=}3$,
suggesting that the score is not tied to one student family or scale.
Appendix~\ref{app:budget_scaling} provides a budget-scaling analysis
for LARK across $B \in \{1,3,5,10,20\}$.

\subsection{RQ2: How Much Does Each Component Contribute?}
\label{sec:rq2_ablation}

We ablate the two design choices in LARK: the trajectory score and
the weighting rule.

\noindent\textbf{Trajectory score.}
Fixing the weighting to $\chi^2$-$B$, we compare four scores
(Figure~\ref{fig:ablation_q1}): the inverse loss
$1/\ell_k$, the Brier residual $\sum_t\|\boldsymbol\pi_t^k -
\boldsymbol\delta(y_t^k)\|^2$, the proxy
$\hat\rho_k$ of Lemma~\ref{lm:1}, and the induced score
$\hat g_k$. The inverse loss is the weakest, confirming that
likelihood alone does not characterize learnability; $\hat\rho_k$
already improves substantially by combining residual error with
loss; and $\hat g_k$ is the best across all four benchmarks, since
it captures the marginal contribution of each trajectory around
$\mathbf p$ rather than scoring trajectories in isolation.

\noindent\textbf{Weighting rule.}
Fixing the score to $\hat g_k$, we vary the weighting at $B{=}3$
(Figure~\ref{fig:ablation_q2}) over uniform,
score-proportional, softmax, and $\chi^2$-$B$ weights. The
$\chi^2$-$B$ rule consistently outperforms the alternatives, with
the largest gains on AIME (where score noise matters most) and on
MATH-500 (where the benefit accumulates over many problems). LARK's
improvement therefore comes not only from selecting high-$\hat g_k$
trajectories but also from assigning robust weights within the
selected support.


\noindent\textbf{Efficiency.}
LARK is forward-pass only and therefore does not require the
per-trajectory backward passes needed to compute the exact directional
derivative $g_k^*$. Its scoring cost matches the lowest-cost
student-side baselines, with full complexity analysis and wall-clock
measurements reported in Appendix~\ref{app:complexity}.

\subsection{RQ3: Why Does LARK Work?}
\label{sec:rq3_why}

We provide three pieces of evidence connecting $\hat g_k$ to the
quantities it is designed to control, then characterize what makes
the trajectories LARK selects informative.

\begin{figure}[t]
\centering
\begin{subfigure}[t]{0.32\linewidth}
  \centering
  \includegraphics[width=\linewidth]{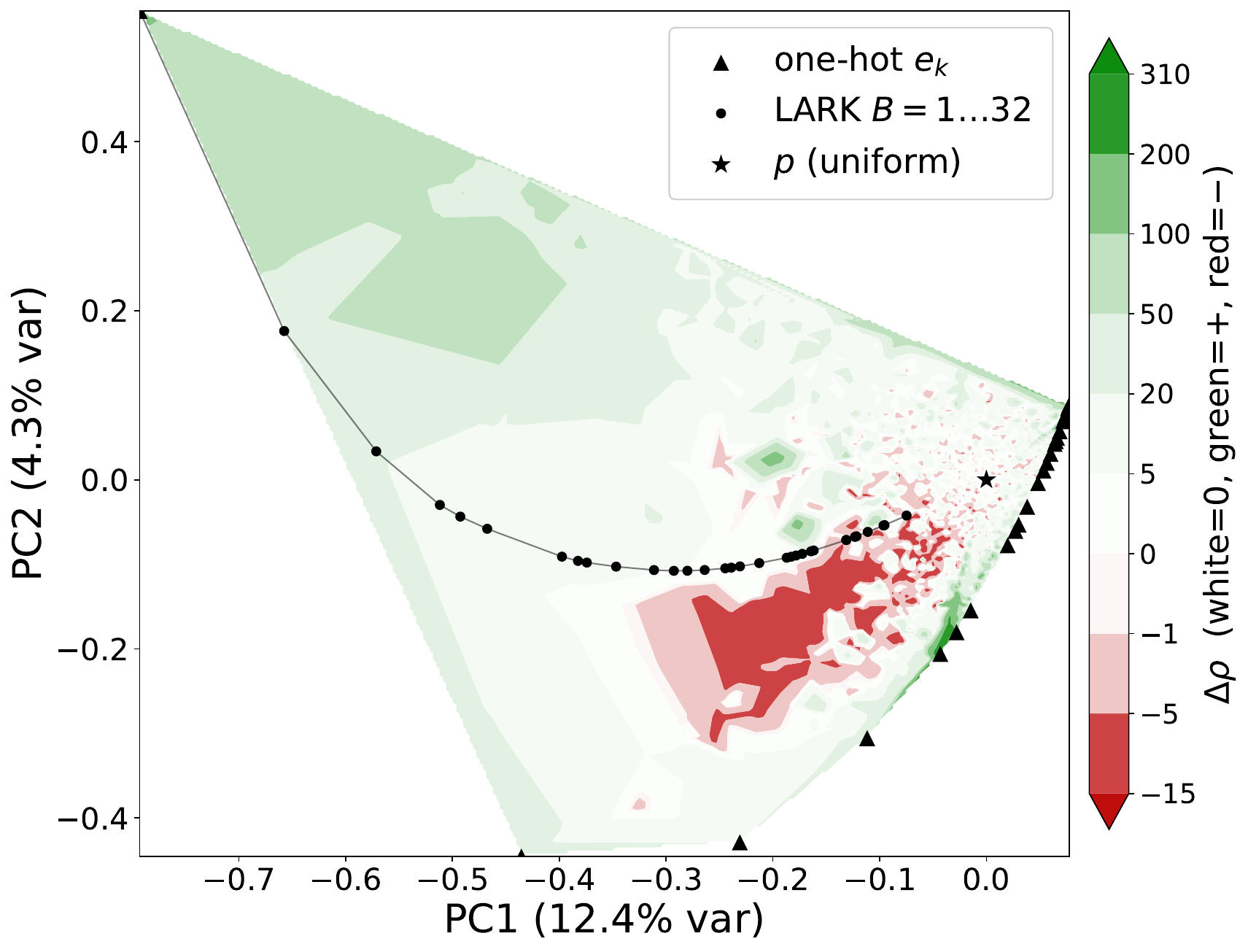}
  \caption{}
  \label{fig:d_2d}
\end{subfigure}
\hfill
\begin{subfigure}[t]{0.32\linewidth}
  \centering
  \includegraphics[width=\linewidth]{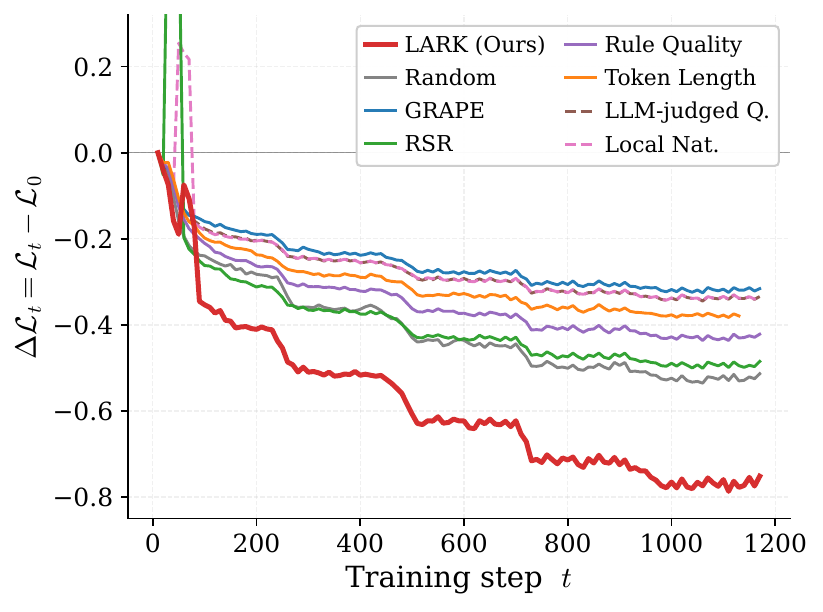}
  \caption{}
  \label{fig:loss_reduction}
\end{subfigure}
\hfill
\begin{subfigure}[t]{0.32\linewidth}
  \centering
  \includegraphics[width=\linewidth]{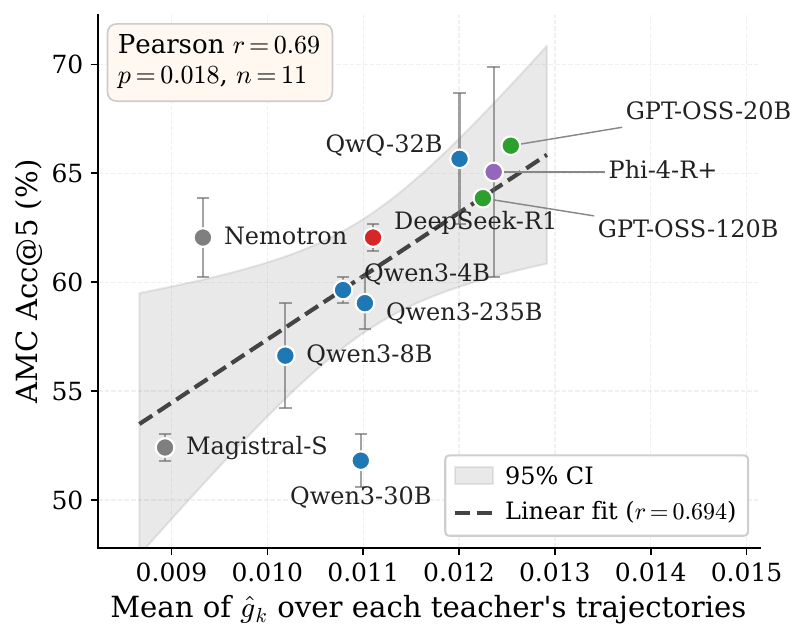}
  \caption{}
  \label{fig:score_utility}
\end{subfigure}
\caption{\textbf{Three pieces of evidence linking $\hat g_k$ to learnability.}
\textbf{(a)} $\Delta\rho(\mathbf{q})$ landscape on the simplex (PCA); black dots are $\hat{\mathbf{q}}(B)$ for $B{=}1{,}\ldots{,}32$. LARK selections trace the $\Delta\rho \geq 0$ region (green) and avoid $\Delta\rho < 0$ (red).
\textbf{(b)} SFT loss reduction $\Delta\mathcal{L}_t = \mathcal{L}_t - \mathcal{L}_0$ on Qwen-2.5-7B ($B{=}3$); LARK induces the largest SFT loss decay.
\textbf{(c)} Per-teacher mean of $\hat g_k$ vs.\ AMC Acc@5.}
\label{fig:rq3_evidence}
\end{figure}
\noindent\textbf{$\hat g_k$ controls the anchor-time learnability
rate.}
We first verify that maximizing $\hat g_k$ actually raises
$\rho(\boldsymbol\theta_{\mathrm{ref}}; \mathbf{q})$, the quantity
LARK is designed to optimize (\S\ref{sec:setup}). For a
representative problem we evaluate
$\Delta\rho(\mathbf{q}) = \rho(\boldsymbol\theta_{\mathrm{ref}};
\mathbf{q}) - \rho(\boldsymbol\theta_{\mathrm{ref}}; \mathbf{p})$
densely over $\Delta^K$ and project onto two PCA components
(Figure~\ref{fig:rq3_evidence}a). The LARK selections
$\hat{\mathbf{q}}(B)$ for $B \in \{1,\ldots,32\}$ trace a path
through the $\Delta\rho > 0$ region (green) and avoid the
$\Delta\rho < 0$ region (red): every $\hat{\mathbf{q}}(B)$ satisfies
$\rho(\boldsymbol\theta_{\mathrm{ref}}; \hat{\mathbf{q}}(B)) \geq
\rho(\boldsymbol\theta_{\mathrm{ref}}; \mathbf{p})$. The
forward-pass surrogate therefore acts as an effective control on
the underlying anchor-time learnability rate on real student models.

\noindent\textbf{High $\hat g_k$ yields faster SFT loss decay.}
By Proposition~\ref{prop:anchor-control}, a higher $\rho$ should
translate into faster SFT loss decay. We verify this by tracking
$\Delta \mathcal{L}_t \triangleq \mathcal{L}_t - \mathcal{L}_0$ on
Qwen-2.5-7B at $B{=}3$. Figure~\ref{fig:rq3_evidence}b shows that
LARK achieves the largest loss reduction throughout training, with
the gap visible already in the early stage; heuristic and
naturalness baselines reduce the loss more slowly and converge to a
shallower plateau, even when starting from a lower initial loss.

\noindent\textbf{$\hat g_k$ predicts downstream training utility.}
We finally check that $\hat g_k$ tracks downstream Acc@5 in
practice. We group the $33$ candidates by their teacher model
($11$ teachers), fine-tune Qwen-2.5-7B separately on each teacher's
trajectories, and evaluate on AMC. The per-teacher mean of
$\hat g_k$ is strongly correlated with downstream AMC accuracy
(Figure~\ref{fig:rq3_evidence}c, Pearson $r = 0.69$, $p = 0.018$):
high-$\hat g_k$ teachers (GPT-OSS-20B/120B, Phi-4-R+, QwQ-32B)
yield students above $65\%$, while low-$\hat g_k$ teachers
(Magistral-S, Nemotron) yield students below $55\%$. The score is
predictively valid at the teacher level, capturing a student-side
training signal that surface quality and likelihood do not.

\begin{wrapfigure}{r}{0.36\textwidth}
\vspace{-0.6em}
\setlength{\intextsep}{0pt}
\setlength{\columnsep}{8pt}
\centering
\includegraphics[width=\linewidth]{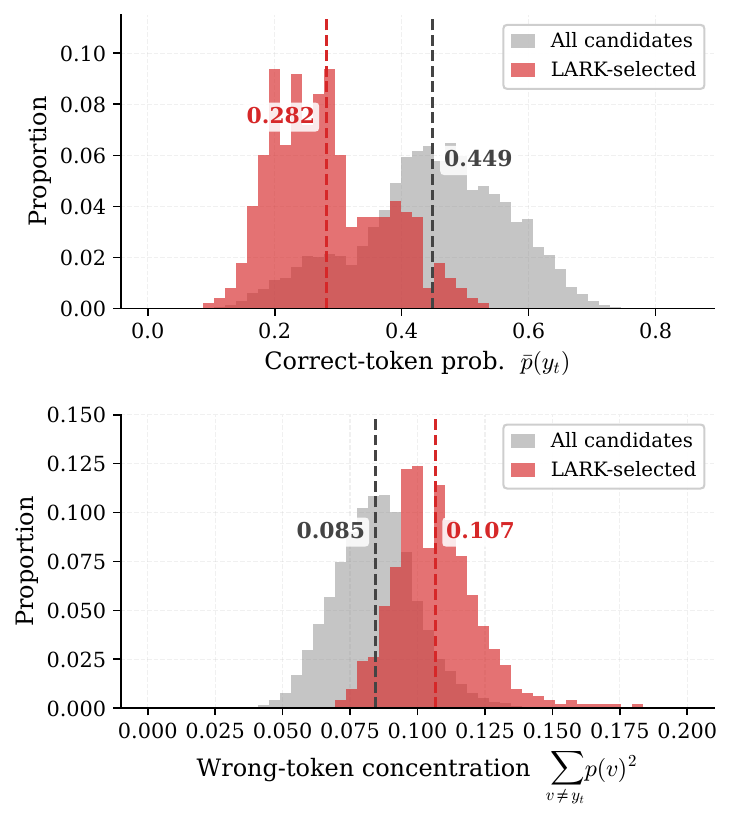}
\caption{\textbf{Token-level behavior} on Qwen-2.5-7B. LARK-selected
trajectories have lower correct-token probability $\bar p(y_t)$
(top) but higher wrong-token concentration
$\sum_{v\neq y_t} p(v)^2$ (bottom).}
\label{fig:token_level}
\vspace{-0.6em}
\end{wrapfigure}

\noindent\textbf{Mechanism: confidently wrong, in a structured way.}
\label{sec:rq4_why}
The three results above establish that maximizing $\hat g_k$ raises
$\rho$, accelerates SFT loss decay, and predicts utility, but do
not explain \emph{why} the selected trajectories carry strong
signal. We compute under $\boldsymbol\theta_{\mathrm{ref}}$ two
token-level statistics for each candidate: the mean correct-token
probability $\bar p(y_t)$, and the wrong-token concentration
$\sum_{v \neq y_t} p(v)^2$, which is large when the residual mass
falls on a few wrong tokens. Comparing all $33 \times 500$
candidates with the LARK-selected top-$1$
(Figure~\ref{fig:token_level}), the selected trajectories have
substantially \emph{lower} $\bar p(y_t)$ ($0.282$ vs.\ $0.449$) and
\emph{higher} wrong-token concentration ($0.107$ vs.\ $0.085$, a
$26\%$ increase). LARK therefore avoids trajectories the student
already predicts confidently and routes supervision toward
trajectories on which the student is wrong but in a \emph{structured}
way: residual mass concentrates on near-miss tokens close to the
gold answer, providing a clear local correction signal. This
signature echoes the rank-based view of
RSR~\citep{yang2026reasoning}, which detects token positions where
the gold answer is a near-miss, and clarifies why LARK and RSR end
up close in Table~\ref{tab:main_weighted_math_merged} despite
arising from distinct theoretical objectives. A single-problem
walkthrough is given in Appendix~\ref{app:case_study}.
\section{Conclusion}
\label{sec:conclusion}

We presented \textbf{LARK}, a learnability-grounded method for reasoning trajectory selection in distillation. LARK estimates trajectory learnability with a forward-pass surrogate $\hat{g}_k$ and uses a closed-form $\chi^2$-regularized soft top-$B$ rule to construct a weighted SFT objective, avoiding expensive per-trajectory backward passes while preserving distributional coverage. Across three student models and four reasoning benchmarks, LARK consistently outperforms strong trajectory selection baselines under both single-trajectory and multi-trajectory budgets. Ablation and diagnostic analyses show that both the learnability score and the weighting rule are important, and that LARK-selected trajectories accelerate SFT loss decay while providing structured local correction signals. Overall, our results support learnability as an effective principle for efficient reasoning distillation, with future extensions to noisier candidate pools and broader reasoning domains.

\bibliographystyle{plainnat}
\bibliography{references}

\appendix
\newpage
\addcontentsline{toc}{section}{Appendix}

\startcontents[sections]
\printcontents[sections]{l}{1}{\setcounter{tocdepth}{2}}

\section{Anchor-Time Analysis of the Learnability Objective}
\label{app:anchor-time-analysis}

This appendix supports the anchor-time reduction in Section~\ref{sec:setup}. We first prove the post-training loss bound (Appendix~\ref{app:proof-anchor-control}). We then provide a theoretical justification of the anchor-relative trajectory bound through the NTK lazy-training framework (Appendix~\ref{app:ntk-justification}), and verify it empirically on real SFT trajectories (Appendix~\ref{app:empirical-verification}).

\subsection{Proof of the Anchor-Time Loss Decay Bound}
\label{app:proof-anchor-control}

\begin{proof}[Proof of Proposition~\ref{prop:anchor-control}]
Fix any $\mathbf{q} \in \Delta^K$. By the gradient-flow decay identity~\eqref{eq:prelim-traj-decay} and Condition~\ref{ass:anchor-bound}, which gives $\rho(\phi_{\mathbf{q}}(s); \mathbf{q}) \geq \kappa \cdot \rho(\boldsymbol{\theta}_{\mathrm{ref}}; \mathbf{q})$ for every $s \in [0, T]$, we have
\begin{equation*}
\int_0^T \rho(\phi_{\mathbf{q}}(s); \mathbf{q})\, ds 
\;\geq\; \int_0^T \kappa \cdot \rho(\boldsymbol{\theta}_{\mathrm{ref}}; \mathbf{q})\, ds 
\;=\; \kappa T \cdot \rho(\boldsymbol{\theta}_{\mathrm{ref}}; \mathbf{q}).
\end{equation*}
Substituting into~\eqref{eq:prelim-traj-decay} and using the non-negativity of $\mathcal{L}(\boldsymbol{\theta}_{\mathrm{ref}}; \mathbf{q})$ together with the monotonicity of $x \mapsto \exp(-x)$ on $\mathbb{R}$,
\begin{equation}
    \mathcal{L}(\phi_{\mathbf{q}}(T); \mathbf{q})
    \;=\;
    \mathcal{L}(\boldsymbol{\theta}_{\mathrm{ref}}; \mathbf{q}) \cdot
    \exp\!\left( -\int_0^T \rho(\phi_{\mathbf{q}}(s); \mathbf{q})\, ds \right)
    \;\leq\;
    \mathcal{L}(\boldsymbol{\theta}_{\mathrm{ref}}; \mathbf{q}) \cdot
    \exp\!\bigl(-\kappa T \cdot \rho(\boldsymbol{\theta}_{\mathrm{ref}}; \mathbf{q})\bigr).
\end{equation}
Since $\mathbf{q} \in \Delta^K$ was arbitrary, the bound holds for every $\mathbf{q} \in \Delta^K$.
\end{proof}

\subsection{NTK-Based Justification of the Trajectory Bound}
\label{app:ntk-justification}

We provide a theoretical justification of Condition~\ref{ass:anchor-bound} through the Neural Tangent Kernel (NTK) framework. Under standard NTK lazy-training conditions for pretrained language model fine-tuning, combined with a local non-degeneracy condition on the ground-truth probability, we derive a strictly positive constant $\kappa$ for which Condition~\ref{ass:anchor-bound} provably holds.

\paragraph{Note on terminology.}
Throughout this subsection, we distinguish between two kinds of structural premises. \emph{Condition~\ref{ass:anchor-bound}} is the anchor-relative condition stated in the main text, which we aim to \emph{justify}. \emph{Assumptions~\ref{app:ass:lazy}--\ref{app:ass:bounded-prob}} below are standard NTK lazy-training premises, which we use to \emph{derive} Condition~\ref{ass:anchor-bound} as a theorem (Proposition~\ref{app:prop:ntk-justification}).

\paragraph{NTK setup.}
We adopt the token-level objects from Section~\ref{sec:method-strategy}. For each candidate trajectory $\mathbf{y}_k$, recall the per-token residual $\delta_t = \pi_t - \mathbf{e}_{y_t} \in \mathbb{R}^V$ and the stacked residual matrix $\Delta_k = [\delta_1, \ldots, \delta_{|a_k|}]^\top \in \mathbb{R}^{|a_k| \times V}$, with $p_t = \pi_t(y_t)$ the probability assigned to the ground-truth token. We additionally introduce the following appendix-only objects.

\textit{Logits Jacobian.} For trajectory $\mathbf{y}_k$, define
\begin{equation}
    J_k(\boldsymbol{\theta}) \triangleq \frac{\partial}{\partial \boldsymbol{\theta}}\, \mathrm{vec}(\mathbf{z}_1, \ldots, \mathbf{z}_{|a_k|}) \;\in\; \mathbb{R}^{|a_k| V \times |\boldsymbol{\theta}|}.
\end{equation}

\textit{Vectorized residual.} Let $\boldsymbol{\Delta}_k(\boldsymbol{\theta}) \triangleq \mathrm{vec}(\Delta_k(\boldsymbol{\theta})) \in \mathbb{R}^{|a_k| V}$.

\textit{Weighted NTK matrix and weighted residual.} Given $\mathbf{q} \in \Delta^K$, let
\begin{equation}
    \tilde{J}(\boldsymbol{\theta}; \mathbf{q}) \triangleq
    \begin{pmatrix}
        \sqrt{q_1 / |a_1|}\, J_1(\boldsymbol{\theta}) \\
        \vdots \\
        \sqrt{q_K / |a_K|}\, J_K(\boldsymbol{\theta})
    \end{pmatrix},
    \qquad
    \tilde{\boldsymbol{\Delta}}(\boldsymbol{\theta}; \mathbf{q}) \triangleq
    \begin{pmatrix}
        \sqrt{q_1 / |a_1|}\, \boldsymbol{\Delta}_1(\boldsymbol{\theta}) \\
        \vdots \\
        \sqrt{q_K / |a_K|}\, \boldsymbol{\Delta}_K(\boldsymbol{\theta})
    \end{pmatrix},
\end{equation}
and define the weighted empirical NTK matrix
\begin{equation}
    \Theta(\boldsymbol{\theta}; \mathbf{q}) \triangleq \tilde{J}(\boldsymbol{\theta}; \mathbf{q})\, \tilde{J}(\boldsymbol{\theta}; \mathbf{q})^\top.
\end{equation}

\textit{NTK identity.} By the standard chain rule for softmax cross-entropy ($\nabla_{\mathbf{z}_t} \ell_t = \delta_t$, see Section~\ref{sec:method-strategy}), the weighted gradient admits the factorization
\[
\nabla_{\boldsymbol{\theta}} \mathcal{L}(\boldsymbol{\theta}; \mathbf{q})
=
\tilde{J}(\boldsymbol{\theta}; \mathbf{q})^\top\,
\tilde{\boldsymbol{\Delta}}(\boldsymbol{\theta}; \mathbf{q}),
\]
which yields the NTK identity
\begin{equation}
    \|\nabla_{\boldsymbol{\theta}} \mathcal{L}(\boldsymbol{\theta}; \mathbf{q})\|^2
    \;=\;
    \tilde{\boldsymbol{\Delta}}(\boldsymbol{\theta}; \mathbf{q})^\top\, \Theta(\boldsymbol{\theta}; \mathbf{q})\, \tilde{\boldsymbol{\Delta}}(\boldsymbol{\theta}; \mathbf{q}).
    \label{eq:ntk-identity}
\end{equation}
When some $q_k=0$, the corresponding trajectory block contributes neither to the weighted loss nor to its gradient. The zero-weight blocks are therefore harmless in the identity above, but they should not be used to impose full-matrix positive definiteness. This is why Assumption~\ref{app:ass:ntk-pd} below conditions the NTK only along the residual directions that enter the weighted objective.

\paragraph{Assumptions.}
We adopt four conditions for the NTK analysis of language model fine-tuning.

\begin{appassumption}[Lazy training]
\label{app:ass:lazy}
There exists a constant $R > 0$ such that, for every $\mathbf{q} \in \Delta^K$,
\begin{equation}
    \sup_{s \in [0, T]} \|\phi_{\mathbf{q}}(s) - \boldsymbol{\theta}_{\mathrm{ref}}\| \;\leq\; R.
\end{equation}
\end{appassumption}

\begin{appassumption}[NTK stability]
\label{app:ass:ntk-stability}
There exists $\epsilon \in (0, 1)$ such that, for every $\mathbf{q} \in \Delta^K$ and every $\boldsymbol{\theta} \in B(\boldsymbol{\theta}_{\mathrm{ref}}, R)$,
\begin{equation}
    (1 - \epsilon)\, \Theta(\boldsymbol{\theta}_{\mathrm{ref}}; \mathbf{q})
    \;\preceq\;
    \Theta(\boldsymbol{\theta}; \mathbf{q})
    \;\preceq\;
    (1 + \epsilon)\, \Theta(\boldsymbol{\theta}_{\mathrm{ref}}; \mathbf{q}).
\end{equation}
\end{appassumption}

\begin{appassumption}[Residual-direction NTK conditioning at anchor]
\label{app:ass:ntk-pd}
There exist constants $0 < \lambda_- \leq \lambda_+ < \infty$ such that, for every $\mathbf{q} \in \Delta^K$ and every $\boldsymbol{\theta} \in B(\boldsymbol{\theta}_{\mathrm{ref}}, R)$,
\begin{equation}
    \tilde{\boldsymbol{\Delta}}(\boldsymbol{\theta}; \mathbf{q})^\top
    \Theta(\boldsymbol{\theta}_{\mathrm{ref}}; \mathbf{q})
    \tilde{\boldsymbol{\Delta}}(\boldsymbol{\theta}; \mathbf{q})
    \;\geq\;
    \lambda_-\, \|\tilde{\boldsymbol{\Delta}}(\boldsymbol{\theta}; \mathbf{q})\|^2,
    \label{eq:residual-ntk-lower}
\end{equation}
and
\begin{equation}
    \tilde{\boldsymbol{\Delta}}(\boldsymbol{\theta}_{\mathrm{ref}}; \mathbf{q})^\top
    \Theta(\boldsymbol{\theta}_{\mathrm{ref}}; \mathbf{q})
    \tilde{\boldsymbol{\Delta}}(\boldsymbol{\theta}_{\mathrm{ref}}; \mathbf{q})
    \;\leq\;
    \lambda_+\, \|\tilde{\boldsymbol{\Delta}}(\boldsymbol{\theta}_{\mathrm{ref}}; \mathbf{q})\|^2.
    \label{eq:residual-ntk-upper}
\end{equation}
\end{appassumption}

\begin{appassumption}[Bounded ground-truth probability]
\label{app:ass:bounded-prob}
There exist constants $0 < p_- \leq p_+ < 1$ such that, for every $\boldsymbol{\theta} \in B(\boldsymbol{\theta}_{\mathrm{ref}}, R)$, every $\mathbf{q} \in \Delta^K$, every $k \in [K]$, and every token position $t$,
\begin{equation}
    p_- \;\leq\; p_t(\boldsymbol{\theta}) \;\leq\; p_+.
\end{equation}
\end{appassumption}

Assumption~\ref{app:ass:lazy} encodes lazy fine-tuning behavior, classical in NTK theory \citep{jacot2018neural, chizat2019lazy} and supported empirically for pre-trained transformer fine-tuning by \citet{malladi2023kernel}. Assumption~\ref{app:ass:ntk-stability} requires the empirical NTK matrix to remain near-constant within the lazy ball, with rate $O(1/\sqrt{m})$ at width $m$ \citep{lee2019wide, jacot2018neural}, and is supported empirically by \citet{malladi2023kernel} for fine-tuning. Assumption~\ref{app:ass:ntk-pd} is a residual-direction conditioning condition: it does not require the full weighted NTK matrix to be positive definite in directions corresponding to zero-weight trajectories, but only requires non-degenerate curvature along the residual directions that determine the weighted loss and its gradient. Assumption~\ref{app:ass:bounded-prob} is a local non-degeneracy condition on teacher-forced token probabilities: $p_t > 0$ holds strictly under finite logits, while $p_t < 1$ excludes already-converged tokens for which the residual signal vanishes.

\paragraph{Loss-residual relation.}
Before establishing the main result, we relate the loss $\mathcal{L}(\boldsymbol{\theta}; \mathbf{q})$ to the squared residual norm $\|\tilde{\boldsymbol{\Delta}}(\boldsymbol{\theta}; \mathbf{q})\|^2$. The relation depends only on Assumption~\ref{app:ass:bounded-prob}, and the constants $\nu_-, \nu_+$ are derived explicitly from $p_-, p_+$.

\begin{lemma}[Loss--residual two-sided bound]
\label{app:lem:loss-residual}
Under Assumption~\ref{app:ass:bounded-prob}, for every $\boldsymbol{\theta} \in B(\boldsymbol{\theta}_{\mathrm{ref}}, R)$ and every $\mathbf{q} \in \Delta^K$,
\begin{equation}
    \nu_-\, \|\tilde{\boldsymbol{\Delta}}(\boldsymbol{\theta}; \mathbf{q})\|^2
    \;\leq\;
    \mathcal{L}(\boldsymbol{\theta}; \mathbf{q})
    \;\leq\;
    \nu_+\, \|\tilde{\boldsymbol{\Delta}}(\boldsymbol{\theta}; \mathbf{q})\|^2,
\end{equation}
where
\begin{equation}
    \nu_- \;=\; \tfrac{1}{2}, \qquad
    \nu_+ \;=\; \frac{\log(1/p_-)}{(1 - p_+)^2}.
    \label{eq:lem-nu-values}
\end{equation}
\end{lemma}

\begin{proof}
We first establish corresponding pointwise bounds at the token level: for every token position $t$, with $p \triangleq p_t \in [p_-, p_+]$,
\begin{equation}
    \tfrac{1}{2}\, \|\delta_t\|^2
    \;\leq\;
    \ell_t
    \;\leq\;
    \frac{\log(1/p_-)}{(1 - p_+)^2}\, \|\delta_t\|^2.
    \label{eq:token-bound}
\end{equation}

\textbf{Lower bound.}
\begin{align}
    \|\delta_t\|^2
    &\stackrel{(a)}{=} (1 - p)^2 + \sum_{v \neq y_t} \pi_t(v)^2 \notag \\
    &\stackrel{(b)}{\leq} (1 - p)^2 + \Bigl(\max_{v \neq y_t} \pi_t(v)\Bigr) \sum_{v \neq y_t} \pi_t(v) \notag \\
    &\stackrel{(c)}{\leq} (1 - p)^2 + (1 - p) \cdot (1 - p)
    \;=\; 2(1 - p)^2 \notag \\
    &\stackrel{(d)}{\leq} 2(1 - p)
    \stackrel{(e)}{\leq} 2 \ell_t,
\end{align}
where (a) is the definition $\delta_t = \pi_t - \mathbf{e}_{y_t}$ together with $\pi_t(y_t) = p$; (b) uses $\sum_v a_v^2 \leq \max_v a_v \cdot \sum_v a_v$ for non-negative $\{a_v\}$; (c) uses both $\max_{v \neq y_t} \pi_t(v) \leq \sum_{v \neq y_t} \pi_t(v) = 1 - p$ and $\sum_{v \neq y_t} \pi_t(v) = 1 - p$; (d) uses $1 - p \leq 1$; and (e) uses the elementary inequality $-\log p \geq 1 - p$ for $p \in (0, 1]$.

\textbf{Upper bound.}
\begin{align}
    \ell_t
    &\stackrel{(f)}{=} -\log p
    \stackrel{(g)}{\leq} \log(1/p_-) \notag \\
    &\stackrel{(h)}{=} \frac{\log(1/p_-)}{(1 - p_+)^2}\, (1 - p_+)^2
    \stackrel{(i)}{\leq} \frac{\log(1/p_-)}{(1 - p_+)^2}\, (1 - p)^2
    \stackrel{(j)}{\leq} \frac{\log(1/p_-)}{(1 - p_+)^2}\, \|\delta_t\|^2,
\end{align}
where (f) is the definition of $\ell_t$; (g) uses $p \geq p_-$ and the monotonicity of $-\log$; (h) is multiplication and division by $(1 - p_+)^2 > 0$; (i) uses $(1 - p_+)^2 \leq (1 - p)^2$ from $p \leq p_+ < 1$; and (j) uses $\|\delta_t\|^2 \geq (1 - p)^2$ from the $y_t$-th coordinate.

\textbf{Aggregation.}
Summing \eqref{eq:token-bound} over $t = 1, \ldots, |a_k|$ and dividing by $|a_k|$ yields
\[
\nu_-\, \|\boldsymbol{\Delta}_k\|^2 / |a_k|
\leq
\mathcal{L}(\boldsymbol{\theta}; \mathbf{y}_k)
\leq
\nu_+\, \|\boldsymbol{\Delta}_k\|^2 / |a_k|,
\]
where $\|\boldsymbol{\Delta}_k\|^2 = \sum_t \|\delta_t\|^2$. Multiplying by $q_k$ and summing over $k$ gives
\begin{equation}
    \nu_-\, \sum_{k=1}^K \frac{q_k}{|a_k|}\, \|\boldsymbol{\Delta}_k\|^2
    \;\leq\;
    \mathcal{L}(\boldsymbol{\theta}; \mathbf{q})
    \;\leq\;
    \nu_+\, \sum_{k=1}^K \frac{q_k}{|a_k|}\, \|\boldsymbol{\Delta}_k\|^2,
\end{equation}
and the result follows from
\[
\sum_{k} (q_k/|a_k|)\, \|\boldsymbol{\Delta}_k\|^2
=
\|\tilde{\boldsymbol{\Delta}}(\boldsymbol{\theta}; \mathbf{q})\|^2.
\]
\end{proof}

\paragraph{Main result.}
We now establish the main result.

\begin{proposition}[NTK-based justification of Condition~\ref{ass:anchor-bound}]
\label{app:prop:ntk-justification}
Under Assumptions~\ref{app:ass:lazy}--\ref{app:ass:bounded-prob}, Condition~\ref{ass:anchor-bound} holds with
\begin{equation}
    \kappa
    \;=\;
    \frac{(1 - \epsilon)\, \lambda_-\, \nu_-}{\lambda_+\, \nu_+}
    \;=\;
    \frac{(1 - \epsilon)\, \lambda_-\, (1 - p_+)^2}{2\, \lambda_+\, \log(1/p_-)}
    \;>\; 0.
    \label{eq:kappa-explicit}
\end{equation}
\end{proposition}

\begin{proof}
Fix any $\mathbf{q} \in \Delta^K$ and $s \in [0, T]$, and set $\boldsymbol{\theta} = \phi_{\mathbf{q}}(s)$. Assumption~\ref{app:ass:lazy} gives $\boldsymbol{\theta} \in B(\boldsymbol{\theta}_{\mathrm{ref}}, R)$.

Writing $\tilde{\boldsymbol{\Delta}} \triangleq \tilde{\boldsymbol{\Delta}}(\boldsymbol{\theta}; \mathbf{q})$,
\begin{align}
    \|\nabla_{\boldsymbol{\theta}} \mathcal{L}(\boldsymbol{\theta}; \mathbf{q})\|^2
    &\stackrel{(a)}{=}
    \tilde{\boldsymbol{\Delta}}^\top
    \Theta(\boldsymbol{\theta}; \mathbf{q})
    \tilde{\boldsymbol{\Delta}} \notag \\
    &\stackrel{(b)}{\geq}
    (1 - \epsilon)\,
    \tilde{\boldsymbol{\Delta}}^\top
    \Theta(\boldsymbol{\theta}_{\mathrm{ref}}; \mathbf{q})
    \tilde{\boldsymbol{\Delta}} \notag \\
    &\stackrel{(c)}{\geq}
    (1 - \epsilon)\, \lambda_-\, \|\tilde{\boldsymbol{\Delta}}\|^2,
\end{align}
where (a) is the NTK identity~\eqref{eq:ntk-identity}; (b) applies the lower-bound part of Assumption~\ref{app:ass:ntk-stability}; and (c) applies the residual-direction lower bound~\eqref{eq:residual-ntk-lower} from Assumption~\ref{app:ass:ntk-pd}. Combined with the upper bound $\mathcal{L}(\boldsymbol{\theta}; \mathbf{q}) \leq \nu_+\, \|\tilde{\boldsymbol{\Delta}}\|^2$ from Lemma~\ref{app:lem:loss-residual},
\begin{equation}
    \rho(\boldsymbol{\theta}; \mathbf{q})
    \;=\;
    \frac{\|\nabla_{\boldsymbol{\theta}} \mathcal{L}(\boldsymbol{\theta}; \mathbf{q})\|^2}
    {\mathcal{L}(\boldsymbol{\theta}; \mathbf{q})}
    \;\geq\;
    \frac{(1 - \epsilon)\, \lambda_-\, \|\tilde{\boldsymbol{\Delta}}\|^2}
    {\nu_+\, \|\tilde{\boldsymbol{\Delta}}\|^2}
    \;=\;
    \frac{(1 - \epsilon)\, \lambda_-}{\nu_+}.
    \label{eq:rho-lb-traj}
\end{equation}

Writing $\tilde{\boldsymbol{\Delta}}_{\mathrm{ref}} \triangleq \tilde{\boldsymbol{\Delta}}(\boldsymbol{\theta}_{\mathrm{ref}}; \mathbf{q})$,
\begin{align}
    \|\nabla_{\boldsymbol{\theta}} \mathcal{L}(\boldsymbol{\theta}_{\mathrm{ref}}; \mathbf{q})\|^2
    &\stackrel{(d)}{=}
    \tilde{\boldsymbol{\Delta}}_{\mathrm{ref}}^\top
    \Theta(\boldsymbol{\theta}_{\mathrm{ref}}; \mathbf{q})
    \tilde{\boldsymbol{\Delta}}_{\mathrm{ref}} \notag \\
    &\stackrel{(e)}{\leq}
    \lambda_+\, \|\tilde{\boldsymbol{\Delta}}_{\mathrm{ref}}\|^2,
\end{align}
where (d) is~\eqref{eq:ntk-identity} at $\boldsymbol{\theta}_{\mathrm{ref}}$ and (e) applies the residual-direction upper bound~\eqref{eq:residual-ntk-upper} from Assumption~\ref{app:ass:ntk-pd}. Combined with the lower bound $\mathcal{L}(\boldsymbol{\theta}_{\mathrm{ref}}; \mathbf{q}) \geq \nu_-\, \|\tilde{\boldsymbol{\Delta}}_{\mathrm{ref}}\|^2$ from Lemma~\ref{app:lem:loss-residual},
\begin{equation}
    \rho(\boldsymbol{\theta}_{\mathrm{ref}}; \mathbf{q})
    \;=\;
    \frac{\|\nabla_{\boldsymbol{\theta}} \mathcal{L}(\boldsymbol{\theta}_{\mathrm{ref}}; \mathbf{q})\|^2}
    {\mathcal{L}(\boldsymbol{\theta}_{\mathrm{ref}}; \mathbf{q})}
    \;\leq\;
    \frac{\lambda_+\, \|\tilde{\boldsymbol{\Delta}}_{\mathrm{ref}}\|^2}
    {\nu_-\, \|\tilde{\boldsymbol{\Delta}}_{\mathrm{ref}}\|^2}
    \;=\;
    \frac{\lambda_+}{\nu_-}.
    \label{eq:rho-ub-anchor}
\end{equation}

Combining~\eqref{eq:rho-lb-traj} and~\eqref{eq:rho-ub-anchor}, both being positive,
\begin{equation}
    \frac{\rho(\phi_{\mathbf{q}}(s); \mathbf{q})}
    {\rho(\boldsymbol{\theta}_{\mathrm{ref}}; \mathbf{q})}
    \;\stackrel{(f)}{\geq}\;
    \frac{(1 - \epsilon)\, \lambda_- / \nu_+}{\lambda_+ / \nu_-}
    \;=\;
    \frac{(1 - \epsilon)\, \lambda_-\, \nu_-}{\lambda_+\, \nu_+},
\end{equation}
where (f) uses the lower bound on the numerator from~\eqref{eq:rho-lb-traj} and the upper bound on the denominator from~\eqref{eq:rho-ub-anchor}. Since $\mathbf{q}$ and $s$ were arbitrary,
\begin{equation}
    \inf_{s \in [0, T]} \rho(\phi_{\mathbf{q}}(s); \mathbf{q})
    \;\geq\;
    \frac{(1 - \epsilon)\, \lambda_-\, \nu_-}{\lambda_+\, \nu_+}
    \cdot
    \rho(\boldsymbol{\theta}_{\mathrm{ref}}; \mathbf{q}),
\end{equation}
which is Condition~\ref{ass:anchor-bound} with
\[
\kappa = \frac{(1 - \epsilon)\lambda_-\nu_-}{\lambda_+\nu_+}.
\]
Substituting $\nu_- = 1/2$ and $\nu_+ = \log(1/p_-)/(1 - p_+)^2$ from Lemma~\ref{app:lem:loss-residual} yields~\eqref{eq:kappa-explicit}.
\end{proof}

\paragraph{Remark on $\kappa$.}
The expression
\[
\kappa
=
\frac{(1 - \epsilon)\lambda_-(1 - p_+)^2}
{2\lambda_+ \log(1/p_-)}
\]
admits a clean interpretation. The ratio $\lambda_-/\lambda_+$ is the effective condition number of the empirical NTK along the residual directions that determine the weighted loss; $1 - \epsilon$ captures NTK stability along the trajectory; and $(1 - p_+)^2 / \log(1/p_-)$ is a geometric correction accounting for the cross-entropy loss, arising from the unbounded growth of $-\log p$ near $p = 0$. Under squared loss, this last factor reduces to a constant and $\kappa$ collapses to $(1 - \epsilon)\lambda_-/\lambda_+$, recovering the classical NTK convergence-rate dependence on the kernel condition number \citep{jacot2018neural, du2019gradient}. The post-training loss bound in Proposition~\ref{prop:anchor-control} thus holds with this $\kappa$, providing theoretical justification for using $\rho(\boldsymbol{\theta}_{\mathrm{ref}}; \mathbf{q})$ as the anchor-time selection objective in Section~\ref{sec:setup}.

\subsection{Empirical Verification of the Trajectory Bound}
\label{app:empirical-verification}

\paragraph{Experimental setup.}
Condition~\ref{ass:anchor-bound} quantifies over all $\mathbf{q} \in \Delta^K$, with no restriction on the candidate pool of the weighted SFT loss in Eq.~\eqref{eq:prelim-weighted-sft}: the pool may consist of the $K$ trajectories of a single question, or of (question, trajectory) pairs aggregated across multiple questions. In actual deployment, LARK is not applied to a single question in isolation but to mini-batches drawn from a corpus of questions, where each batch induces a weight vector over the (question, trajectory) pairs it contains. We therefore verify Condition~\ref{ass:anchor-bound} directly in the corpus-level regime in which LARK is deployed.

\textit{Model and anchor.} We verify on Qwen-2.5-7B, the main student model used in Section~\ref{sec:exp_setup}. The pre-trained checkpoint serves as the anchor parameter $\boldsymbol{\theta}_{\mathrm{ref}}$.

\textit{Verification corpus.} From the 5{,}000 NuminaMath training questions used in Section~\ref{sec:exp_setup}, we randomly sample 500 questions $\{\mathbf{x}_n\}_{n=1}^{500}$ as the verification corpus. Each $\mathbf{x}_n$ retains its full pool of 33 candidate trajectories (11 teachers $\times$ 3 rollouts). For each $\mathbf{x}_n$ and each budget $B \in \{1, 3, 5, 10, 20\}$, we apply the LARK selection rule (Lemma~\ref{prop:b_param}) to obtain a weight vector $\hat{\mathbf{q}}_n^{(B)}$. Stacking and rescaling the $500$ per-question weights yields a corpus-level weight vector $\mathbf{Q}^{(B)} \in \Delta^{500 \times 33}$, with the corresponding weighted SFT loss
\begin{equation}
    \mathcal{L}_{\mathrm{total}}(\boldsymbol{\theta}; B)
    \;\triangleq\;
    \mathcal{L}\bigl(\boldsymbol{\theta}; \mathbf{Q}^{(B)}\bigr)
    \;=\;
    \frac{1}{500} \sum_{n=1}^{500} \mathcal{L}\bigl(\boldsymbol{\theta}; \hat{\mathbf{q}}_n^{(B)}\bigr).
\end{equation}
We treat $\mathbf{Q}^{(B)}$ directly as the weight vector $\mathbf{q}$ in Condition~\ref{ass:anchor-bound}, so the corpus-level quantities $\rho_t$ and $\hat{\kappa}^{(B)}$ defined below are instantiations of Condition~\ref{ass:anchor-bound} at $\mathbf{Q}^{(B)}$.

\textit{SFT training.} For each budget $B \in \{1, 3, 5, 10, 20\}$, we run mini-batch SFT on $\mathcal{L}_{\mathrm{total}}(\cdot; B)$ starting from $\boldsymbol{\theta}_{\mathrm{ref}}$, using the same hyperparameters as the main experiments. Full training details are deferred to Appendix~\ref{app:exp_impl}. 

\textit{Recorded quantities.} At each training step $t \in \{0, 1, \ldots, T_{\max}\}$, we record the corpus-level decay rate
\begin{equation}
    \rho_t
    \;\triangleq\;
    \frac{\|\nabla_{\boldsymbol{\theta}} \mathcal{L}_{\mathrm{total}}(\boldsymbol{\theta}_t; B)\|^2}{\mathcal{L}_{\mathrm{total}}(\boldsymbol{\theta}_t; B)},
\end{equation}
and the anchor-relative ratio $\hat{\kappa}_t \triangleq \rho_t / \rho_0$, which is the discrete-time analog of $\rho(\phi_{\mathbf{Q}^{(B)}}(s); \mathbf{Q}^{(B)}) / \rho(\boldsymbol{\theta}_{\mathrm{ref}}; \mathbf{Q}^{(B)})$ in Condition~\ref{ass:anchor-bound}. After training, for each $B$ we report the trajectory-wise minimum
\begin{equation}
    \hat{\kappa}^{(B)}
    \;\triangleq\;
    \min_{t \in [0, T_{\max}]} \rho_t / \rho_0,
\end{equation}
as the empirical estimate of $\kappa$ in Condition~\ref{ass:anchor-bound} under budget $B$.

\paragraph{Findings.}
Figure~\ref{fig:kappa-verification-qwen25-7b} reports the verification results for $B \in \{1, 3, 5, 10, 20\}$. Each panel plots the corpus-level decay rate $\rho_t$ (blue solid, left axis) and the anchor-relative ratio $\hat{\kappa}_t = \rho_t / \rho_0$ (red dashed, right axis) against the training step $t$. The trajectory-wise minimum $\hat{\kappa}^{(B)}$ and the step at which it is attained are annotated in each panel title.

\begin{figure}[t]
    \centering
    \includegraphics[width=\textwidth]{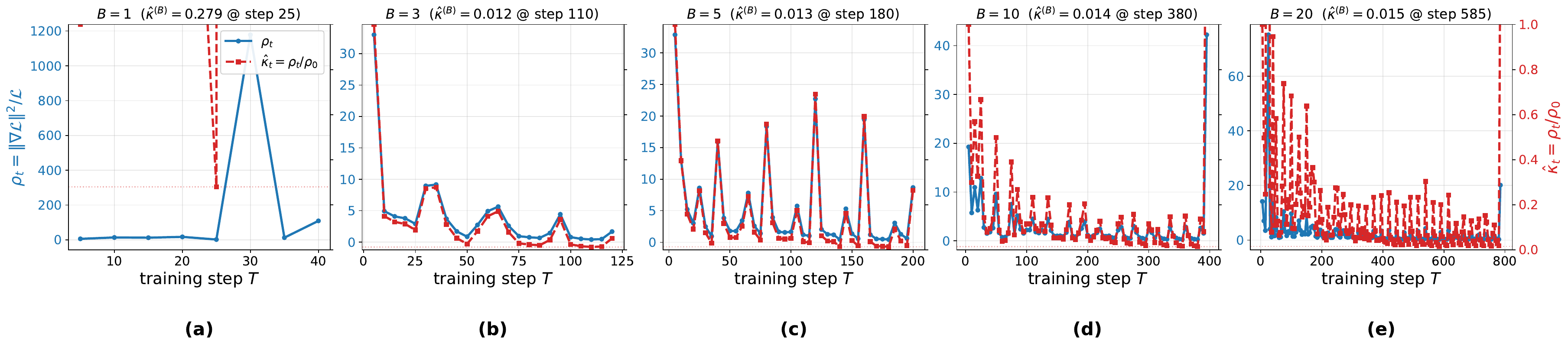}
    \caption{Empirical verification of Condition~\ref{ass:anchor-bound} on Qwen-2.5-7B for $B \in \{1, 3, 5, 10, 20\}$. Blue: $\rho_t$ (left axis); red dashed: $\hat{\kappa}_t = \rho_t/\rho_0$ (right axis). Panel titles report $\hat{\kappa}^{(B)} = \min_t \hat{\kappa}_t$ and the step at which it is attained.}
    \label{fig:kappa-verification-qwen25-7b}
\end{figure}

We highlight three observations from the results.
\begin{itemize}[leftmargin=*]
    \item \emph{Condition~\ref{ass:anchor-bound} holds across all configurations.} For every $B$, the trajectory-wise minimum $\hat{\kappa}^{(B)}$ is strictly positive ($\hat{\kappa}^{(B)} \in [0.012, 0.279]$), confirming that $\rho$ along the SFT trajectory does not collapse to zero relative to its anchor-time value. This empirically validates the existence of a positive $\kappa$ in Condition~\ref{ass:anchor-bound}.
    \item \emph{Smallest budget yields the largest $\hat{\kappa}$.} The minimum ratio at $B = 1$ ($\hat{\kappa}^{(1)} = 0.279$) substantially exceeds the values for $B \geq 3$ (which fall in the range $0.012$--$0.015$). Single-trajectory training produces a more stable $\rho$ trajectory and a tighter anchor-relative bound, while multi-trajectory training induces more oscillatory $\rho_t$ but still maintains a positive lower bound.
    \item \emph{$\hat{\kappa}^{(B)}$ stabilizes for moderate $B$.} For $B \in \{3, 5, 10, 20\}$, the minimum ratio remains within a narrow range ($0.012$--$0.015$), indicating that beyond a small budget, the corpus-level anchor-relative bound depends only weakly on $B$.
\end{itemize}

\paragraph{On the early-step transient in $\rho_t$.}
A closer look at Figure~\ref{fig:kappa-verification-qwen25-7b} reveals that, for $B \geq 3$, $\rho_t$ is non-monotone: it exhibits a sharp transient spike during the first $\sim$10\% of training steps (reaching values an order of magnitude above $\rho_0$) before settling to a positive plateau. Two factors are consistent with this behavior. First, the empirical NTK $\Theta(\boldsymbol{\theta}_t; \mathbf{q})$ adapts most rapidly during the warmup phase of SFT, before the lazy-training regime of Assumption~\ref{app:ass:lazy} fully takes effect; the looser NTK stability in this early window admits larger fluctuations of $\rho_t$ around its anchor-time value. Second, weighted multi-trajectory gradients $\nabla_{\boldsymbol{\theta}} \mathcal{L}(\boldsymbol{\theta}_t; \mathbf{q})$ admit cancellation across trajectory blocks with different residual directions, which can produce non-monotone $\|\nabla_{\boldsymbol{\theta}} \mathcal{L}\|^2$ in early steps. Crucially, the empirical $\hat{\kappa}^{(B)}$ reported in the panel titles is the trajectory-wise \emph{minimum} of $\rho_t / \rho_0$ over the entire training run, which already absorbs this transient: even at the lowest point of the trajectory, $\rho_t / \rho_0$ remains bounded away from zero, which is exactly what Condition~\ref{ass:anchor-bound} requires. The transient therefore tightens (rather than relaxes) the empirical estimate of $\kappa$.

\section{Theoretical Analysis of LARK}
\label{app:lark-theory}

This appendix collects the formal statements, proofs, and supporting derivations for the theoretical claims in Section~\ref{sec:method-strategy} (the local linearization analysis of LARK) and Section~\ref{sec:algorithm} (the closed-form selection rule).

We begin in Appendix~\ref{app:quasiconvex} by establishing that $\rho(\boldsymbol{\theta}_{\mathrm{ref}};\, \mathbf{q})$ is quasiconvex on $\Delta^K$ and that its (constrained or unconstrained) maximum collapses to a one-hot solution; this is the formal justification for the local-linearization strategy adopted in Section~\ref{sec:method-strategy}. Appendix~\ref{app:notation-assumptions} fixes the additional notation and structural assumptions used by the rest of the appendix. Appendix~\ref{app:aux-lemmas} collects auxiliary lemmas (gradient identities and bounds at the LM-head and backbone, together with the Taylor decomposition of $\rho$ around $\mathbf p$) that are reused by the subsequent proofs. Appendices~\ref{app:proof-lemma1}, \ref{app:proof-lemma2}, \ref{app:proof-thm1}, and \ref{app:proof-prop2} respectively give the formal versions and full proofs of Lemma~\ref{lm:1} (forward-pass $\rho_k^*$ estimation), Lemma~\ref{lm:chi2_bound} ($\chi^2$ controls both $R_1$ and $R_2$), Theorem~\ref{thm:lark-objective} (the LARK objective lower bound), and Lemma~\ref{prop:b_param} (the $B$-parameterized closed form). Finally, Appendix~\ref{app:ghat-empirical} provides a quantitative error bound on the surrogate $\hat g_k$ relative to the exact directional derivative $g_k^*$, together with empirical evidence that the two are tightly correlated on real student models.

\subsection{Quasiconvexity of \texorpdfstring{$\rho(\boldsymbol{\theta}_{\mathrm{ref}};\, \mathbf{q})$}{rho(theta\_ref; q)} and Collapse to One-Hot}
\label{app:quasiconvex}

Throughout this subsection, fix $\boldsymbol{\theta} = \boldsymbol{\theta}_{\mathrm{ref}}$ and recall from Section~\ref{sec:method-strategy} that $\mathbf g_k = \nabla_{\boldsymbol\theta} \ell_k$ with $\ell_k = \ell(\boldsymbol\theta_{\mathrm{ref}}, \mathbf y_k) > 0$. Define
\begin{equation}
    N(\mathbf{q}) \;\triangleq\; \Bigl\|\textstyle\sum_k q_k \mathbf{g}_k\Bigr\|^2
    \;=\; \mathbf{q}^\top G\, \mathbf{q},
    \qquad
    D(\mathbf{q}) \;\triangleq\; \textstyle\sum_k q_k \ell_k,
    \label{eq:ND-def}
\end{equation}
where $G \in \mathbb{R}^{K\times K}$ with $G_{ij} = \langle \mathbf{g}_i,\, \mathbf{g}_{j}\rangle$ is the gradient Gram matrix. Since $G \succeq 0$, $N(\mathbf{q})$ is a convex quadratic in $\mathbf{q}$, and $D(\mathbf{q}) > 0$ is affine, so $\rho(\boldsymbol{\theta}_{\mathrm{ref}};\, \mathbf{q}) = N(\mathbf{q}) / D(\mathbf{q})$.

\begin{proposition}[Quasiconvexity of $\rho(\boldsymbol{\theta}_{\mathrm{ref}};\, \mathbf{q})$]
\label{prop:quasiconvex}
$\rho(\boldsymbol{\theta}_{\mathrm{ref}};\, \mathbf{q})$ is quasiconvex in $\mathbf{q} \in \Delta^K$.
\end{proposition}

\begin{proof}
It suffices to show that for every $c \geq 0$, the sublevel set $\mathcal{S}_c \triangleq \{\mathbf{q}\in\Delta^K : \rho(\boldsymbol{\theta}_{\mathrm{ref}};\, \mathbf{q}) \leq c\}$ is convex. We have
\begin{align}
    \mathbf{q}\in\mathcal{S}_c
    &\;\stackrel{(a)}{\iff}\;
    \rho(\boldsymbol{\theta}_{\mathrm{ref}};\, \mathbf{q}) \leq c \notag\\
    &\;\stackrel{(b)}{\iff}\;
    \frac{N(\mathbf{q})}{D(\mathbf{q})} \leq c \notag\\
    &\;\stackrel{(c)}{\iff}\;
    N(\mathbf{q}) \leq c\cdot D(\mathbf{q}) \notag\\
    &\;\stackrel{(d)}{\iff}\;
    \mathbf{q}^\top G\,\mathbf{q} - c\textstyle\sum_k q_k\ell_k \leq 0,
    \label{eq:sublevel_condition}
\end{align}
where (a) is the definition of $\mathcal{S}_c$; (b) substitutes $\rho(\boldsymbol{\theta}_{\mathrm{ref}};\, \mathbf{q}) = N(\mathbf{q})/D(\mathbf{q})$; (c) multiplies both sides by $D(\mathbf{q}) > 0$; and (d) substitutes the definitions of $N$ and $D$ from~\eqref{eq:ND-def}. The left-hand side of the final inequality in~\eqref{eq:sublevel_condition} is the sum of the convex quadratic $\mathbf{q}^\top G\,\mathbf{q}$ (since $G\succeq 0$) and the affine term $-c\sum_k q_k\ell_k$, hence convex in $\mathbf{q}$. Therefore $\mathcal{S}_c$ is the intersection of a convex sublevel set with the convex set $\Delta^K$, which is convex.
\end{proof}

\begin{corollary}[Maximum at a vertex]
\label{cor:vertex}
$\max_{\mathbf{q}\in\Delta^K} \rho(\boldsymbol{\theta}_{\mathrm{ref}};\, \mathbf{q})$ is attained at an extreme point of $\Delta^K$.
\end{corollary}

\begin{proof}
We have
\begin{align}
    \max_{\mathbf{q}\in\Delta^K} \rho(\boldsymbol{\theta}_{\mathrm{ref}};\, \mathbf{q})
    &\;\stackrel{(a)}{=}\;
    \max_{\mathbf{q}\in\Delta^K} \rho(\boldsymbol{\theta}_{\mathrm{ref}};\, \mathbf{q})\big|_{\text{extreme point}}
    \notag\\
    &\;\stackrel{(b)}{=}\;
    \max_{k\in[K]} \rho(\boldsymbol{\theta}_{\mathrm{ref}};\, \mathbf{e}_k),
\end{align}
where (a) follows because every $\mathbf q\in\Delta^K$ can be written as a convex combination of the extreme points $\{\mathbf e_1,\ldots,\mathbf e_K\}$, namely $\mathbf q=\sum_{k=1}^K q_k\mathbf e_k$, and quasiconvexity implies
\(
\rho(\boldsymbol{\theta}_{\mathrm{ref}};\, \mathbf q)
\leq
\max_{k\in[K]:\,q_k>0}
\rho(\boldsymbol{\theta}_{\mathrm{ref}};\, \mathbf e_k)
\leq
\max_{k\in[K]}
\rho(\boldsymbol{\theta}_{\mathrm{ref}};\, \mathbf e_k).
\)
The reverse inequality is immediate since each $\mathbf e_k\in\Delta^K$. Hence the maximum over $\Delta^K$ is attained at an extreme point. Step (b) identifies the extreme points of $\Delta^K$ as the standard basis vectors $\{\mathbf{e}_1,\ldots,\mathbf{e}_K\}$.
\end{proof}

\begin{corollary}[One-hot optimum]
\label{cor:onehot}
\begin{equation}
    \max_{\mathbf{q}\in\Delta^K} \rho(\boldsymbol{\theta}_{\mathrm{ref}};\, \mathbf{q})
    \;=\; \max_{k\in[K]} \rho(\boldsymbol{\theta}_{\mathrm{ref}};\, \mathbf{e}_k)
    \;=\; \max_{k\in[K]} \frac{\|\mathbf{g}_k\|^2}{\ell_k}
    \;=\; \max_{k\in[K]} \rho_k^*.
\end{equation}
\end{corollary}

\begin{proof}
We have
\begin{align}
    \max_{\mathbf{q}\in\Delta^K} \rho(\boldsymbol{\theta}_{\mathrm{ref}};\, \mathbf{q})
    &\;\stackrel{(a)}{=}\;
    \max_{k\in[K]} \rho(\boldsymbol{\theta}_{\mathrm{ref}};\, \mathbf{e}_k) \notag\\
    &\;\stackrel{(b)}{=}\;
    \max_{k\in[K]} \frac{N(\mathbf{e}_k)}{D(\mathbf{e}_k)}
    \;=\;
    \max_{k\in[K]} \frac{\|\mathbf{g}_k\|^2}{\ell_k}
    \;\stackrel{(c)}{=}\;
    \max_{k\in[K]} \rho_k^*,
\end{align}
where (a) applies Corollary~\ref{cor:vertex}; (b) substitutes $\mathbf{q} = \mathbf{e}_k$ into the definitions in~\eqref{eq:ND-def}, giving $N(\mathbf{e}_k) = \|\mathbf{g}_k\|^2$ and $D(\mathbf{e}_k) = \ell_k$; and (c) uses the definition $\rho_k^* = \|\mathbf{g}_k\|^2/\ell_k$ from Section~\ref{sec:method-strategy}.
\end{proof}

\subsection{Notation and Structural Assumptions}
\label{app:notation-assumptions}

This subsection fixes the additional notation and structural assumptions used throughout Appendices~\ref{app:aux-lemmas}--\ref{app:ghat-empirical}.

\subsubsection{Notation}
\label{app:notation}

\paragraph{Symbols inherited from the main text.}
We use the same symbols as Section~\ref{sec:method-strategy}: $\mathbf x$ is the question; $\mathbf y_k = (y_1^k, \ldots, y_{|a_k|}^k)$ is the $k$-th candidate trajectory, of length $|a_k|$; $\boldsymbol\pi_t^k \in \Delta^{|\mathcal V|}$ is the predictive distribution of the reference student $\pi_{\mathrm{ref}}$ at position $t$ under teacher forcing; $\boldsymbol\delta(y_t^k) \in \Delta^{|\mathcal V|}$ is the one-hot vector for the ground-truth token; $\ell_k = \ell(\boldsymbol\theta_{\mathrm{ref}}, \mathbf y_k)$ is the length-normalized cross-entropy of trajectory $\mathbf y_k$; $\mathbf g_k = \nabla_{\boldsymbol\theta} \ell_k$ is its gradient; $\rho_k^* = \|\mathbf g_k\|^2 / \ell_k$ is the per-trajectory rate; $\hat\rho_k$ is its forward-pass proxy (see Lemma~\ref{lm:1}); $\hat g_k$ and $g_k^*$ are the surrogate and exact directional derivatives at $\mathbf p$; $R_1, R_2$ are the residual terms in~\eqref{eq:1}; $\mathbf p = (1/K, \ldots, 1/K) \in \Delta^K$ is the uniform prior.

\paragraph{Parameter decomposition.}
We split the student parameters as $\boldsymbol\theta = (\mathbf W_{\mathrm{out}}, \boldsymbol\theta_{\mathrm{rest}})$, where $\mathbf W_{\mathrm{out}} \in \mathbb R^{|\mathcal V| \times d}$ is the LM head matrix and $\boldsymbol\theta_{\mathrm{rest}}$ collects all backbone parameters. For each trajectory $\mathbf y_k$ and position $t$, the backbone produces a hidden state $\mathbf h_t \in \mathbb R^d$, the LM head produces logits $\mathbf z_t = \mathbf W_{\mathrm{out}} \mathbf h_t \in \mathbb R^{|\mathcal V|}$, and $\boldsymbol\pi_t^k = \mathrm{softmax}(\mathbf z_t)$. The token-level loss is $\ell_t = -\log \pi_t^k(y_t^k)$, so that $\ell_k = \tfrac{1}{|a_k|} \sum_t \ell_t$.

\paragraph{Token-level residual.}
We write the per-token prediction residual as $\boldsymbol\delta_t \triangleq \boldsymbol\pi_t^k - \boldsymbol\delta(y_t^k) \in \mathbb R^{|\mathcal V|}$, which coincides with $\nabla_{\mathbf z_t}\ell_t$ by the standard softmax--cross-entropy identity (see Lemma~\ref{lem:lmhead_grad} below).

\paragraph{Stacked matrices.}
For each trajectory $\mathbf y_k$ we collect the per-token residuals and hidden states into matrices
\begin{equation}
    \boldsymbol\Delta
    \;\triangleq\;
    [\boldsymbol\delta_1, \ldots, \boldsymbol\delta_{|a_k|}]^\top
    \;\in\; \mathbb R^{|a_k| \times |\mathcal V|},
    \qquad
    \mathbf H
    \;\triangleq\;
    [\mathbf h_1, \ldots, \mathbf h_{|a_k|}]^\top
    \;\in\; \mathbb R^{|a_k| \times d}.
    \label{eq:Delta-H-def}
\end{equation}
We also write $\mathbf J_H(\mathbf y_k) \triangleq [\partial \mathbf h_t / \partial \boldsymbol\theta_{\mathrm{rest}}]_{t=1}^{|a_k|} \in \mathbb R^{|a_k|\,d \times |\boldsymbol\theta_{\mathrm{rest}}|}$ for the sequence-level Jacobian of the hidden states with respect to the backbone parameters. The dependence on $\mathbf y_k$ is suppressed when the trajectory is clear from context.

\paragraph{Aggregated quantities.}
Two aggregations of the per-token residual norms appear repeatedly. Define the (length-normalized) Brier score
\begin{equation}
    \mathrm{Brier}(\mathbf y_k)
    \;\triangleq\;
    \tfrac{1}{|a_k|}\sum_{t=1}^{|a_k|} \|\boldsymbol\delta_t\|^2
    \;=\;
    \tfrac{\|\boldsymbol\Delta\|_F^2}{|a_k|}.
    \label{eq:brier-def}
\end{equation}
With this notation, the forward-pass proxy of Lemma~\ref{lm:1} can be written as
\begin{equation}
    \hat\rho_k
    \;=\;
    \frac{\sum_t \|\boldsymbol\delta_t\|^2}{\sum_t \ell_t}
    \;=\;
    \frac{\mathrm{Brier}(\mathbf y_k)}{\ell_k},
    \label{eq:rhohat-equiv}
\end{equation}
so that the $1/|a_k|$ factors cancel exactly in numerator and denominator.

\paragraph{Multi-trajectory quantities.}
Across the candidate pool, write $G \in \mathbb R^{K \times K}$ for the gradient Gram matrix with $G_{ij} = \langle \mathbf g_i, \mathbf g_j \rangle$, and let $\boldsymbol\ell = (\ell_1, \ldots, \ell_K)^\top$. We use the shorthands
\begin{equation}
    \bar{\mathbf g} \;\triangleq\; \tfrac{1}{K}\textstyle\sum_k \mathbf g_k,
    \qquad
    \bar\ell \;\triangleq\; \tfrac{1}{K}\textstyle\sum_k \ell_k,
\end{equation}
together with $\hat\rho_{\max} \triangleq \max_i \hat\rho_i$, $\ell_{\min} \triangleq \min_i \ell_i$, and $\ell_{\max} \triangleq \max_i \ell_i$. We also write $\|\mathbf v\|_{\mathbf p}^2 \triangleq \sum_k p_k v_k^2$ for the $\mathbf p$-weighted Euclidean norm; under uniform $\mathbf p$, $\|\mathbf v\|_{\mathbf p}^2 = \tfrac{1}{K} \sum_k v_k^2$.

\subsubsection{Structural assumptions}
\label{app:assumptions}

We make four structural assumptions on the reference student $\pi_{\mathrm{ref}}$ and its candidate-pool statistics. All four are mild conditions that are either standard in the convergence analysis of deep networks or directly verifiable from a single forward pass.

\begin{appassumption}[Hidden-state norm stability]
\label{ass:norm}
There exist constants $0 < C_- \leq C_+$ such that for every candidate trajectory $\mathbf y_k$ and every token position $t$,
\begin{equation}
    C_-^2 \;\leq\; \|\mathbf h_t\|^2 \;\leq\; C_+^2.
\end{equation}
\end{appassumption}
\begin{remark}
This holds naturally in modern LLMs employing RMSNorm~\citep{zhang2019root}, which constrains hidden-state norms to a tight interval; we have $C_+/C_- \approx 1$ in practice.
\end{remark}

\begin{appassumption}[Local Jacobian conditioning]
\label{ass:jacobian}
There exists $\Lambda > 0$ such that for every candidate trajectory $\mathbf y_k$,
\begin{equation}
    \lambda_{\max}\bigl(\mathbf J_H(\mathbf y_k)\, \mathbf J_H(\mathbf y_k)^\top\bigr) \;\leq\; \Lambda.
\end{equation}
\end{appassumption}
\begin{remark}
This is a standard local Lipschitz continuity assumption on the backbone, widely adopted in convergence analyses of SFT~\citep{liu2022loss}.
\end{remark}

\begin{appassumption}[Matrix cosine alignment]
\label{ass:align}
There exists $\zeta_- > 0$ such that for every candidate trajectory $\mathbf y_k$,
\begin{equation}
    \zeta(\mathbf y_k)
    \;\triangleq\;
    \frac{\|\boldsymbol\Delta^\top \mathbf H\|_F^2}
         {\|\boldsymbol\Delta\|_F^2 \cdot \|\mathbf H\|_F^2}
    \;\geq\; \zeta_- \;>\; 0.
    \label{eq:zeta-align}
\end{equation}
\end{appassumption}
\begin{remark}
$\zeta(\mathbf y_k)$ is the squared Frobenius cosine similarity between the residual matrix $\boldsymbol\Delta$ and the hidden-state matrix $\mathbf H$. Assumption~\ref{ass:align} requires that the joint signal between prediction residuals and hidden states does not degenerate, and is verifiable from a single forward pass.
\end{remark}

\begin{appassumption}[Strictly positive trajectory loss]
\label{ass:lmin}
$\ell_{\min} = \min_{i \in [K]} \ell_i > 0$.
\end{appassumption}
\begin{remark}
Under finite logits, $\ell_t > 0$ at every token, so $\ell_k > 0$ for every trajectory and Assumption~\ref{ass:lmin} holds automatically. We state it explicitly because it is needed in Appendix~\ref{app:proof-lemma2} to ensure the Hessian of $\rho(\boldsymbol\theta_{\mathrm{ref}};\, \mathbf q)$ is uniformly bounded over $\Delta^K$ via $D(\mathbf q) = \boldsymbol\ell^\top \mathbf q \geq \ell_{\min}$.
\end{remark}
\subsection{Auxiliary Lemmas}
\label{app:aux-lemmas}

This subsection collects building-block lemmas that are reused by the proofs in Appendices~\ref{app:proof-lemma1}--\ref{app:proof-thm1}. Lemmas~\ref{lem:lmhead_grad}--\ref{lem:backbone_bound} are the gradient identity and gradient norm bounds for the LM head and the backbone; they are the main tools for the proof of Lemma~\ref{lm:1} in Appendix~\ref{app:proof-lemma1}. Lemma~\ref{lem:taylor-decomp} gives the Taylor decomposition of $\rho(\boldsymbol\theta_{\mathrm{ref}};\, \mathbf q)$ around the uniform prior $\mathbf p$, together with the closed form of $g_k^*$ and the definition of $\hat g_k$ in terms of forward-pass quantities; this is the starting identity for the proof of Lemma~\ref{lm:chi2_bound} in Appendix~\ref{app:proof-lemma2}. Lemma~\ref{lem:gram-bound} provides a forward-computable upper bound on the entries of the gradient Gram matrix $G$ via the proxy $\hat\rho_i$, which controls every $G$-dependent quantity that appears in the proof of Lemma~\ref{lm:chi2_bound}. Finally, Lemma~\ref{lem:hessian-form} gives the closed-form Hessian of $\rho$ on $\Delta^K$, used to bound the second-order remainder $R_2$.

\subsubsection{LM head and backbone gradient identities and bounds}

\begin{lemma}[LM head gradient]
\label{lem:lmhead_grad}
Let $\ell_t = -\log \pi_t^k(y_t^k)$, $\boldsymbol\pi_t^k = \mathrm{softmax}(\mathbf z_t)$, and $\mathbf z_t = \mathbf W_{\mathrm{out}} \mathbf h_t$. Then
\begin{equation}
    \frac{\partial \ell_t}{\partial \mathbf W_{\mathrm{out}}}
    \;=\; \boldsymbol\delta_t\, \mathbf h_t^\top \;\in\; \mathbb R^{|\mathcal V|\times d},
    \qquad
    \frac{\partial \ell_t}{\partial \mathbf h_t}
    \;=\; \mathbf W_{\mathrm{out}}^\top\, \boldsymbol\delta_t.
    \label{eq:lmhead-grad}
\end{equation}
\end{lemma}

\begin{proof}
We prove the first identity; the second is immediate from the chain rule applied to $\mathbf z_t = \mathbf W_{\mathrm{out}} \mathbf h_t$. For any direction $\mathbf U \in \mathbb R^{|\mathcal V|\times d}$,
\begin{align}
    \left\langle\frac{\partial\ell_t}{\partial\mathbf W_{\mathrm{out}}},\,\mathbf U\right\rangle_F
    &\;\stackrel{(a)}{=}\;
    \left\langle\frac{\partial\ell_t}{\partial\mathbf z_t},\,
    \frac{\partial\mathbf z_t}{\partial\mathbf W_{\mathrm{out}}}[\mathbf U]\right\rangle
    \notag\\
    &\;\stackrel{(b)}{=}\;
    \left\langle\frac{\partial\ell_t}{\partial\mathbf z_t},\,\mathbf U \mathbf h_t\right\rangle
    \notag\\
    &\;\stackrel{(c)}{=}\;
    \sum_{v=1}^{|\mathcal V|}
    \frac{\partial\ell_t}{\partial z_{t,v}}\,[\mathbf U \mathbf h_t]_v
    \notag\\
    &\;\stackrel{(d)}{=}\;
    \sum_{v=1}^{|\mathcal V|}
    \bigl(\pi_t^k(v) - \mathbf{1}[v = y_t^k]\bigr)\,[\mathbf U \mathbf h_t]_v
    \notag\\
    &\;\stackrel{(e)}{=}\;
    \langle\boldsymbol\pi_t^k - \boldsymbol\delta(y_t^k),\,\mathbf U \mathbf h_t\rangle
    \notag\\
    &\;\stackrel{(f)}{=}\;
    \langle\boldsymbol\delta_t\, \mathbf h_t^\top,\,\mathbf U\rangle_F,
\end{align}
where (a) is the chain rule for the Fr\'echet derivative; (b) follows from the linearity of $\mathbf z_t = \mathbf W_{\mathrm{out}} \mathbf h_t$ in $\mathbf W_{\mathrm{out}}$, giving $\frac{\partial \mathbf z_t}{\partial \mathbf W_{\mathrm{out}}}[\mathbf U] = \mathbf U \mathbf h_t$; (c) expands the inner product componentwise; (d) is the standard softmax--cross-entropy gradient identity, $\partial \ell_t / \partial z_{t,v} = \pi_t^k(v) - \mathbf{1}[v = y_t^k]$, obtained by applying the quotient rule to $\pi_t^k(y_t^k) = e^{z_{t,y_t^k}}/\sum_{v'} e^{z_{t,v'}}$; (e) recognizes the sum as $\langle \boldsymbol\delta_t, \mathbf U \mathbf h_t\rangle$ with $\boldsymbol\delta_t = \boldsymbol\pi_t^k - \boldsymbol\delta(y_t^k)$; and (f) uses the trace identity $\langle \mathbf a, \mathbf B \mathbf c\rangle = \langle \mathbf a \mathbf c^\top, \mathbf B\rangle_F$. Since $\mathbf U$ was arbitrary, $\partial \ell_t / \partial \mathbf W_{\mathrm{out}} = \boldsymbol\delta_t \mathbf h_t^\top$.
\end{proof}

\begin{lemma}[LM head gradient bounds]
\label{lem:lmhead_bounds}
Under Assumptions~\ref{ass:norm} and~\ref{ass:align}, for every candidate trajectory $\mathbf y_k$,
\begin{equation}
    \zeta_-\, C_-^2\, \mathrm{Brier}(\mathbf y_k)
    \;\leq\;
    \|\nabla_{\mathbf W_{\mathrm{out}}}\ell_k\|_F^2
    \;\leq\;
    C_+^2\, \mathrm{Brier}(\mathbf y_k).
    \label{eq:lmhead-bounds}
\end{equation}
\end{lemma}

\begin{proof}
By Lemma~\ref{lem:lmhead_grad} and the definition $\ell_k = \tfrac{1}{|a_k|}\sum_t \ell_t$,
\begin{equation}
    \nabla_{\mathbf W_{\mathrm{out}}}\ell_k
    \;=\; \tfrac{1}{|a_k|}\textstyle\sum_t \boldsymbol\delta_t \mathbf h_t^\top
    \;=\; \tfrac{1}{|a_k|}\boldsymbol\Delta^\top \mathbf H,
    \label{eq:lmhead_exact}
\end{equation}
where the second equality follows from the definitions of $\boldsymbol\Delta, \mathbf H$ in~\eqref{eq:Delta-H-def}. Hence $\|\nabla_{\mathbf W_{\mathrm{out}}}\ell_k\|_F^2 = \|\boldsymbol\Delta^\top \mathbf H\|_F^2 / |a_k|^2$.

\textit{Lower bound.}
\begin{align}
    \|\nabla_{\mathbf W_{\mathrm{out}}}\ell_k\|_F^2
    &\;\stackrel{(a)}{=}\;
    \frac{\|\boldsymbol\Delta^\top \mathbf H\|_F^2}{|a_k|^2}
    \notag\\
    &\;\stackrel{(b)}{\geq}\;
    \frac{\zeta_-\, \|\boldsymbol\Delta\|_F^2\, \|\mathbf H\|_F^2}{|a_k|^2}
    \notag\\
    &\;\stackrel{(c)}{\geq}\;
    \frac{\zeta_-\, |a_k|\, \mathrm{Brier}(\mathbf y_k) \cdot |a_k|\, C_-^2}{|a_k|^2}
    \notag\\
    &\;=\;
    \zeta_-\, C_-^2\, \mathrm{Brier}(\mathbf y_k),
\end{align}
where (a) is~\eqref{eq:lmhead_exact}; (b) applies Assumption~\ref{ass:align}, which gives $\|\boldsymbol\Delta^\top \mathbf H\|_F^2 \geq \zeta_-\, \|\boldsymbol\Delta\|_F^2\, \|\mathbf H\|_F^2$; and (c) substitutes $\|\boldsymbol\Delta\|_F^2 = |a_k|\, \mathrm{Brier}(\mathbf y_k)$ from~\eqref{eq:brier-def} together with $\|\mathbf H\|_F^2 = \sum_t \|\mathbf h_t\|^2 \geq |a_k|\, C_-^2$ from Assumption~\ref{ass:norm}.

\textit{Upper bound.}
\begin{align}
    \|\nabla_{\mathbf W_{\mathrm{out}}}\ell_k\|_F^2
    &\;\stackrel{(d)}{=}\;
    \frac{\|\boldsymbol\Delta^\top \mathbf H\|_F^2}{|a_k|^2}
    \notag\\
    &\;\stackrel{(e)}{\leq}\;
    \frac{\|\boldsymbol\Delta\|_F^2\, \|\mathbf H\|_F^2}{|a_k|^2}
    \notag\\
    &\;\stackrel{(f)}{\leq}\;
    \frac{|a_k|\, \mathrm{Brier}(\mathbf y_k) \cdot |a_k|\, C_+^2}{|a_k|^2}
    \notag\\
    &\;=\;
    C_+^2\, \mathrm{Brier}(\mathbf y_k),
\end{align}
where (d) is~\eqref{eq:lmhead_exact}; (e) is the matrix Cauchy--Schwarz inequality $\|\boldsymbol\Delta^\top \mathbf H\|_F^2 \leq \|\boldsymbol\Delta\|_F^2\, \|\mathbf H\|_F^2$; and (f) substitutes $\|\boldsymbol\Delta\|_F^2 = |a_k|\, \mathrm{Brier}(\mathbf y_k)$ from~\eqref{eq:brier-def} together with $\|\mathbf H\|_F^2 = \sum_t \|\mathbf h_t\|^2 \leq |a_k|\, C_+^2$ from Assumption~\ref{ass:norm}.
\end{proof}

\begin{lemma}[Backbone gradient bound]
\label{lem:backbone_bound}
Under Assumptions~\ref{ass:norm} and~\ref{ass:jacobian}, for every candidate trajectory $\mathbf y_k$,
\begin{equation}
    \|\nabla_{\boldsymbol\theta_{\mathrm{rest}}}\ell_k\|^2
    \;\leq\;
    \frac{\Lambda\, \|\mathbf W_{\mathrm{out}}\|_{\mathrm{op}}^2}{|a_k|}\, \mathrm{Brier}(\mathbf y_k).
    \label{eq:backbone-bound}
\end{equation}
\end{lemma}

\begin{proof}
By the chain rule together with the second identity of Lemma~\ref{lem:lmhead_grad}, $\partial \ell_t / \partial \mathbf h_t = \mathbf W_{\mathrm{out}}^\top \boldsymbol\delta_t$, hence
\begin{equation}
    \nabla_{\boldsymbol\theta_{\mathrm{rest}}}\ell_k
    \;=\; \mathbf J_H(\mathbf y_k)^\top\,
    \mathrm{vec}\!\left(\tfrac{1}{|a_k|}\boldsymbol\Delta\, \mathbf W_{\mathrm{out}}\right).
    \label{eq:backbone_seq}
\end{equation}
Then
\begin{align}
    \|\nabla_{\boldsymbol\theta_{\mathrm{rest}}}\ell_k\|^2
    &\;\stackrel{(a)}{\leq}\;
    \lambda_{\max}\bigl(\mathbf J_H(\mathbf y_k)\, \mathbf J_H(\mathbf y_k)^\top\bigr)
    \left\|\mathrm{vec}\!\left(\tfrac{1}{|a_k|}\boldsymbol\Delta\, \mathbf W_{\mathrm{out}}\right)\right\|^2
    \notag\\
    &\;\stackrel{(b)}{\leq}\;
    \frac{\Lambda}{|a_k|^2}\, \|\boldsymbol\Delta\, \mathbf W_{\mathrm{out}}\|_F^2
    \notag\\
    &\;\stackrel{(c)}{\leq}\;
    \frac{\Lambda\, \|\mathbf W_{\mathrm{out}}\|_{\mathrm{op}}^2}{|a_k|^2}\, \|\boldsymbol\Delta\|_F^2
    \notag\\
    &\;\stackrel{(d)}{=}\;
    \frac{\Lambda\, \|\mathbf W_{\mathrm{out}}\|_{\mathrm{op}}^2}{|a_k|}\, \mathrm{Brier}(\mathbf y_k),
\end{align}
where (a) uses $\|\mathbf A^\top \mathbf v\|^2 \leq \lambda_{\max}(\mathbf A \mathbf A^\top) \|\mathbf v\|^2$ with $\mathbf A = \mathbf J_H(\mathbf y_k)$; (b) applies Assumption~\ref{ass:jacobian} together with the identity $\|\mathrm{vec}(\mathbf M)\|^2 = \|\mathbf M\|_F^2$; (c) uses the submultiplicativity $\|\boldsymbol\Delta\, \mathbf W_{\mathrm{out}}\|_F \leq \|\mathbf W_{\mathrm{out}}\|_{\mathrm{op}} \|\boldsymbol\Delta\|_F$; and (d) substitutes $\|\boldsymbol\Delta\|_F^2 = |a_k|\, \mathrm{Brier}(\mathbf y_k)$ from~\eqref{eq:brier-def}.
\end{proof}

\subsubsection{Taylor decomposition of \texorpdfstring{$\rho(\boldsymbol\theta_{\mathrm{ref}};\, \mathbf q)$}{rho(theta\_ref; q)}}

The next lemma packages the algebraic preliminaries for the linearization argument: the closed form of the exact directional derivative $g_k^*$, the definition of the forward-pass surrogate $\hat g_k$, and the resulting decomposition of the change in $\rho$ into a linear term plus two residuals.

\begin{lemma}[Taylor decomposition]
\label{lem:taylor-decomp}
Let $N(\mathbf q) \triangleq \|\sum_k q_k \mathbf g_k\|^2 = \mathbf q^\top G \mathbf q$ and $D(\mathbf q) \triangleq \sum_k q_k \ell_k = \boldsymbol\ell^\top \mathbf q$, so that $\rho(\boldsymbol\theta_{\mathrm{ref}};\, \mathbf q) = N(\mathbf q)/D(\mathbf q)$. Then:
\begin{itemize}[leftmargin=1.4em,topsep=2pt,itemsep=1pt]
\item[(i)] (\emph{Closed form of $g_k^*$.}) The exact directional derivative at $\mathbf p$ satisfies
\begin{equation}
    g_k^*
    \;=\;
    \left.\frac{\partial}{\partial q_k}\frac{N(\mathbf q)}{D(\mathbf q)}\right|_{\mathbf q = \mathbf p}
    \;=\;
    \frac{2\, \mathbf g_k^\top \bar{\mathbf g}}{\bar\ell}
    \;-\;
    \frac{\rho(\boldsymbol\theta_{\mathrm{ref}};\, \mathbf p)\, \ell_k}{\bar\ell}.
    \label{eq:gstar-closed}
\end{equation}

\item[(ii)] (\emph{Forward-pass surrogate.}) Applying the diagonal approximations $\mathbf g_k^\top \bar{\mathbf g} \approx \|\mathbf g_k\|^2 / K$ and $\|\bar{\mathbf g}\|^2 \approx \tfrac{1}{K^2}\sum_i \|\mathbf g_i\|^2$ in~\eqref{eq:gstar-closed}, and substituting the proxy $\hat\rho_i$ for $\rho_i^* = \|\mathbf g_i\|^2/\ell_i$, gives
\begin{equation}
    \hat g_k
    \;=\;
    \frac{\ell_k}{\sum_i \ell_i}
    \!\left(2\hat\rho_k - \frac{\sum_i \hat\rho_i\, \ell_i}{\sum_i \ell_i}\right).
    \label{eq:ghat-closed}
\end{equation}

\item[(iii)] (\emph{Decomposition.}) Setting $\boldsymbol\eta \triangleq \hat{\mathbf g} - \mathbf g^*$, the change in $\rho$ admits the decomposition
\begin{equation}
    \rho(\boldsymbol\theta_{\mathrm{ref}};\, \mathbf q) - \rho(\boldsymbol\theta_{\mathrm{ref}};\, \mathbf p)
    \;=\;
    \langle \mathbf q - \mathbf p,\, \hat{\mathbf g}\rangle
    \;+\; R_1
    \;+\; R_2,
    \label{eq:rho-decomp}
\end{equation}
where
\begin{equation}
    R_1 \;\triangleq\; -\langle \mathbf q - \mathbf p,\, \boldsymbol\eta\rangle,
    \qquad
    R_2 \;\triangleq\;
    \rho(\boldsymbol\theta_{\mathrm{ref}};\, \mathbf q) - \rho(\boldsymbol\theta_{\mathrm{ref}};\, \mathbf p) - \langle \mathbf q - \mathbf p,\, \mathbf g^*\rangle
    \label{eq:R1R2-def}
\end{equation}
are the surrogate substitution error and the second-order Taylor remainder, respectively.
\end{itemize}
\end{lemma}

\begin{proof}
For (i), apply the quotient rule to $N(\mathbf q)/D(\mathbf q)$:
\begin{equation}
    \frac{\partial}{\partial q_k}\frac{N(\mathbf q)}{D(\mathbf q)}
    \;=\;
    \frac{D(\mathbf q) \cdot 2\, \mathbf g_k^\top \!\sum_i q_i \mathbf g_i - N(\mathbf q)\, \ell_k}{D(\mathbf q)^2}.
\end{equation}
Evaluating at $\mathbf q = \mathbf p$ with $p_k = 1/K$ gives $\sum_i p_i \mathbf g_i = \bar{\mathbf g}$, $D(\mathbf p) = \bar\ell$, and $N(\mathbf p)/D(\mathbf p) = \rho(\boldsymbol\theta_{\mathrm{ref}};\, \mathbf p)$, yielding~\eqref{eq:gstar-closed}.

For (ii), substitute $\mathbf g_k^\top \bar{\mathbf g} \approx \|\mathbf g_k\|^2/K$ and $\rho(\boldsymbol\theta_{\mathrm{ref}};\, \mathbf p) = \|\bar{\mathbf g}\|^2 / \bar\ell \approx \tfrac{1}{K^2 \bar\ell}\sum_i \|\mathbf g_i\|^2$ into~\eqref{eq:gstar-closed}, and replace each $\|\mathbf g_i\|^2 = \rho_i^*\, \ell_i$ by $\hat\rho_i\, \ell_i$. Using $K\bar\ell = \sum_i \ell_i$ and simplifying gives~\eqref{eq:ghat-closed}.

For (iii), expand $\rho(\boldsymbol\theta_{\mathrm{ref}};\, \mathbf q) - \rho(\boldsymbol\theta_{\mathrm{ref}};\, \mathbf p) = \langle \mathbf q - \mathbf p, \mathbf g^*\rangle + R_2$ by definition of $R_2$, and add and subtract $\langle \mathbf q - \mathbf p, \hat{\mathbf g}\rangle$ to obtain~\eqref{eq:rho-decomp} with $R_1$ as in~\eqref{eq:R1R2-def}.
\end{proof}

\subsubsection{Forward-computable bound on the gradient Gram matrix}

The next lemma upgrades the per-trajectory bound of Lemma~\ref{lm:1} to a uniform, forward-computable bound on every entry of the gradient Gram matrix $G$. It is the main tool that lets us replace gradient-dependent quantities (such as $\mathbf g_k^\top \bar{\mathbf g}$, $\|\bar{\mathbf g}\|^2$, and $\|G\|_{\mathrm{op}}$) by forward-only quantities throughout the proofs of Lemma~\ref{lm:chi2_bound} and Theorem~\ref{thm:lark-objective}.

\begin{lemma}[Gram-matrix entry bound]
\label{lem:gram-bound}
Assume Assumptions~\ref{ass:norm}--\ref{ass:align}, and let
\begin{equation}
    E_+
    \;\triangleq\;
    C_+^2 \max_{i \in [K]}\bigl(1 + \varepsilon_{\mathrm{head}}(|a_i|)\bigr),
    \qquad
    \varepsilon_{\mathrm{head}}(|a_i|)
    \;=\;
    \frac{\Lambda\, \|\mathbf W_{\mathrm{out}}\|_{\mathrm{op}}^2}{C_+^2\, |a_i|},
    \label{eq:Eplus-def}
\end{equation}
so that Lemma~\ref{lm:1} yields the uniform inequality $\rho_i^* \leq E_+\, \hat\rho_i$ for every $i \in [K]$. Then for every $i, j \in [K]$,
\begin{equation}
    |G_{ij}|
    \;\leq\;
    \sqrt{G_{ii}\, G_{jj}}
    \;=\;
    \sqrt{\rho_i^*\, \rho_j^*\, \ell_i\, \ell_j}
    \;\leq\;
    E_+ \sqrt{\hat\rho_i\, \hat\rho_j\, \ell_i\, \ell_j}.
    \label{eq:Gij-bound}
\end{equation}
\end{lemma}

\begin{proof}
We have
\begin{align}
    |G_{ij}|
    &\;\stackrel{(a)}{\leq}\;
    \sqrt{G_{ii}\, G_{jj}}
    \notag\\
    &\;\stackrel{(b)}{=}\;
    \sqrt{\rho_i^*\, \rho_j^*\, \ell_i\, \ell_j}
    \notag\\
    &\;\stackrel{(c)}{\leq}\;
    E_+ \sqrt{\hat\rho_i\, \hat\rho_j\, \ell_i\, \ell_j},
\end{align}
where (a) is the entry-wise Cauchy--Schwarz bound for inner products, $|\langle \mathbf g_i, \mathbf g_j\rangle| \leq \|\mathbf g_i\| \|\mathbf g_j\|$; (b) uses $G_{ii} = \|\mathbf g_i\|^2 = \rho_i^*\, \ell_i$ from the definition $\rho_i^* = \|\mathbf g_i\|^2 / \ell_i$; and (c) applies the upper part of Lemma~\ref{lm:1} entry-wise to each of $\rho_i^*$ and $\rho_j^*$.
\end{proof}

\begin{remark}
All bounds derived in Appendix~\ref{app:proof-lemma2} (the proof of Lemma~\ref{lm:chi2_bound}) trace back to~\eqref{eq:Gij-bound}. In particular, both the off-diagonal residuals appearing in the surrogate substitution error and the operator-norm bound on $G$ used in the Hessian analysis are obtained by summing~\eqref{eq:Gij-bound} over the appropriate index sets.
\end{remark}

\subsubsection{Hessian of \texorpdfstring{$\rho$}{rho} on the simplex}

\begin{lemma}[Hessian formula]
\label{lem:hessian-form}
On the relative interior of $\Delta^K$, the Hessian of $\rho(\boldsymbol\theta_{\mathrm{ref}};\, \mathbf q) = N(\mathbf q)/D(\mathbf q)$ in $\mathbf q$ is
\begin{equation}
    \nabla^2_{\mathbf q}\, \rho(\boldsymbol\theta_{\mathrm{ref}};\, \mathbf q)
    \;=\;
    \frac{1}{D(\mathbf q)}
    \Bigl[\,
        2\, G
        \;-\; \boldsymbol\ell\, (\nabla \rho)^\top
        \;-\; (\nabla \rho)\, \boldsymbol\ell^\top
    \,\Bigr],
    \label{eq:hessian-form}
\end{equation}
where $\nabla \rho = (2\, G \mathbf q - \rho\, \boldsymbol\ell)/D(\mathbf q)$.
\end{lemma}

\begin{proof}
The quotient rule applied to $\rho = N/D$ gives $\nabla \rho = (\nabla N - \rho\, \nabla D)/D = (2\, G \mathbf q - \rho\, \boldsymbol\ell)/D$, or equivalently $D\, \nabla \rho = 2\, G \mathbf q - \rho\, \boldsymbol\ell$. Differentiating both sides with respect to $\mathbf q$, using $\nabla D = \boldsymbol\ell$ and the product rule $\nabla(D \nabla \rho) = \boldsymbol\ell\, (\nabla \rho)^\top + D\, \nabla^2 \rho$, gives
\begin{equation}
    \boldsymbol\ell\, (\nabla \rho)^\top + D\, \nabla^2 \rho
    \;=\;
    2\, G \;-\; (\nabla \rho)\, \boldsymbol\ell^\top.
\end{equation}
Solving for $\nabla^2 \rho$ and noting that the left-hand side is symmetric (which forces the symmetrized form of the rank-2 correction) yields~\eqref{eq:hessian-form}.
\end{proof}
\subsection{Proof of Lemma~\ref{lm:1} (Forward-pass \texorpdfstring{$\rho_k^*$}{rho\_k\textasciicircum*} Estimation)}
\label{app:proof-lemma1}

This subsection states and proves the formal version of Lemma~\ref{lm:1}. The informal version in Section~\ref{sec:method-strategy} asserts the existence of absolute positive constants $C_1, C_2$ with $\hat\rho_k \in [C_1\, \rho_k^*,\, C_2\, \rho_k^*]$. The formal statement below makes the constants $C_1, C_2$ explicit in terms of the structural quantities of Assumptions~\ref{ass:norm}--\ref{ass:align}.

\begin{lemma}[Formal version of Lemma~\ref{lm:1}]
\label{lm:1-formal}
Under Assumptions~\ref{ass:norm}--\ref{ass:align}, for every candidate trajectory $\mathbf y_k$,
\begin{equation}
    C_1\,\rho_k^*
    \;\leq\;
    \hat\rho_k
    \;\leq\;
    C_2\,\rho_k^*,
    \label{eq:rho-bound}
\end{equation}
where
\begin{equation}
    C_1
    \;\triangleq\;
    \frac{1}{E_+},
    \qquad
    C_2
    \;\triangleq\;
    \frac{1}{\zeta_-\, C_-^2},
    \qquad
    E_+
    \;\triangleq\;
    C_+^2
    \max_{i\in[K]}
    \bigl(1+\varepsilon_{\mathrm{head}}(|a_i|)\bigr),
    \label{eq:c1c2-def}
\end{equation}
and the LM-head correction term is
\begin{equation}
    \varepsilon_{\mathrm{head}}(|a_i|)
    \;\triangleq\;
    \frac{\Lambda\, \|\mathbf W_{\mathrm{out}}\|_{\mathrm{op}}^2}{C_+^2\, |a_i|}.
    \label{eq:eps-head-def}
\end{equation}
\end{lemma}

\begin{proof}
Since $\boldsymbol\theta = (\mathbf W_{\mathrm{out}},\, \boldsymbol\theta_{\mathrm{rest}})$ are disjoint parameter blocks,
\begin{equation}
    \|\nabla_{\boldsymbol\theta}\, \ell_k\|^2
    \;=\;
    \|\nabla_{\mathbf W_{\mathrm{out}}}\, \ell_k\|_F^2
    \;+\;
    \|\nabla_{\boldsymbol\theta_{\mathrm{rest}}}\, \ell_k\|^2.
    \label{eq:grad_split}
\end{equation}
Using the equivalent form $\hat\rho_k = \mathrm{Brier}(\mathbf y_k)/\ell_k$ from~\eqref{eq:rhohat-equiv} together with $\rho_k^* = \|\nabla_{\boldsymbol\theta}\, \ell_k\|^2 / \ell_k$ from Section~\ref{sec:method-strategy}, it suffices to relate $\|\nabla_{\boldsymbol\theta}\ell_k\|^2$ to $\mathrm{Brier}(\mathbf y_k)$.

\textit{Upper bound on $\hat\rho_k$.}
\begin{align}
    \rho_k^*
    &\;\stackrel{(a)}{=}\;
    \frac{\|\nabla_{\boldsymbol\theta}\, \ell_k\|^2}{\ell_k}
    \notag\\
    &\;\stackrel{(b)}{\geq}\;
    \frac{\|\nabla_{\mathbf W_{\mathrm{out}}}\, \ell_k\|_F^2}{\ell_k}
    \notag\\
    &\;\stackrel{(c)}{\geq}\;
    \frac{\zeta_-\, C_-^2\, \mathrm{Brier}(\mathbf y_k)}{\ell_k}
    \notag\\
    &\;\stackrel{(d)}{=}\;
    \zeta_-\, C_-^2\, \hat\rho_k,
\end{align}
where (a) is the definition $\rho_k^* = \|\nabla_{\boldsymbol\theta}\, \ell_k\|^2 / \ell_k$; (b) applies the decomposition~\eqref{eq:grad_split} and discards the non-negative backbone term $\|\nabla_{\boldsymbol\theta_{\mathrm{rest}}}\, \ell_k\|^2 \geq 0$; (c) applies the lower bound of Lemma~\ref{lem:lmhead_bounds}; and (d) substitutes $\hat\rho_k = \mathrm{Brier}(\mathbf y_k)/\ell_k$ from~\eqref{eq:rhohat-equiv}. Therefore,
\begin{equation}
    \hat\rho_k
    \;\leq\;
    \frac{1}{\zeta_-\, C_-^2}\,\rho_k^*
    \;=\;
    C_2\,\rho_k^*.
    \label{eq:rhohat-upper-final}
\end{equation}

\textit{Lower bound on $\hat\rho_k$.}
\begin{align}
    \rho_k^*
    &\;\stackrel{(e)}{=}\;
    \frac{\|\nabla_{\boldsymbol\theta}\, \ell_k\|^2}{\ell_k}
    \notag\\
    &\;\stackrel{(f)}{=}\;
    \frac{\|\nabla_{\mathbf W_{\mathrm{out}}}\, \ell_k\|_F^2 + \|\nabla_{\boldsymbol\theta_{\mathrm{rest}}}\, \ell_k\|^2}{\ell_k}
    \notag\\
    &\;\stackrel{(g)}{\leq}\;
    \frac{1}{\ell_k}
    \!\left(
    C_+^2\, \mathrm{Brier}(\mathbf y_k)
    \;+\;
    \frac{\Lambda\, \|\mathbf W_{\mathrm{out}}\|_{\mathrm{op}}^2}{|a_k|}\, \mathrm{Brier}(\mathbf y_k)
    \right)
    \notag\\
    &\;\stackrel{(h)}{=}\;
    C_+^2\,\bigl(1+\varepsilon_{\mathrm{head}}(|a_k|)\bigr)\,
    \hat\rho_k
    \notag\\
    &\;\stackrel{(i)}{\leq}\;
    E_+\,\hat\rho_k,
\end{align}
where (e) is the definition $\rho_k^* = \|\nabla_{\boldsymbol\theta}\, \ell_k\|^2 / \ell_k$; (f) applies the decomposition~\eqref{eq:grad_split}; (g) applies the upper bound of Lemma~\ref{lem:lmhead_bounds} to the LM-head term and Lemma~\ref{lem:backbone_bound} to the backbone term; (h) substitutes $\hat\rho_k = \mathrm{Brier}(\mathbf y_k)/\ell_k$ from~\eqref{eq:rhohat-equiv} together with the definition~\eqref{eq:eps-head-def} of $\varepsilon_{\mathrm{head}}(|a_k|)$; and (i) follows from the definition of $E_+$ in~\eqref{eq:c1c2-def}. Therefore,
\begin{equation}
    \hat\rho_k
    \;\geq\;
    \frac{1}{E_+}\,\rho_k^*
    \;=\;
    C_1\,\rho_k^*.
    \label{eq:rhohat-lower-final}
\end{equation}

Combining~\eqref{eq:rhohat-upper-final} and~\eqref{eq:rhohat-lower-final} yields~\eqref{eq:rho-bound}.
\end{proof}
\subsection{Proof of Lemma~\ref{lm:chi2_bound} (\texorpdfstring{$\chi^2$}{chi-squared} Controls Both Error Terms)}
\label{app:proof-lemma2}

This subsection states and proves the formal version of Lemma~\ref{lm:chi2_bound}. We use the Taylor decomposition in Lemma~\ref{lem:taylor-decomp}. Let
$\boldsymbol\eta \triangleq \hat{\mathbf g}-\mathbf g^*$. The two residual terms in~\eqref{eq:rho-decomp} are
\begin{equation}
    R_1
    \;=\;
    -\langle \mathbf q-\mathbf p,\,\boldsymbol\eta\rangle,
    \qquad
    R_2
    \;=\;
    \rho(\boldsymbol\theta_{\mathrm{ref}};\mathbf q)
    -
    \rho(\boldsymbol\theta_{\mathrm{ref}};\mathbf p)
    -
    \langle \mathbf q-\mathbf p,\,\mathbf g^*\rangle .
    \label{eq:R1R2-recall}
\end{equation}
Since $p_k=1/K$, we also have
\begin{equation}
    \chi^2(\mathbf q\,\|\,\mathbf p)
    =
    \sum_k \frac{(q_k-p_k)^2}{p_k}
    =
    K\|\mathbf q-\mathbf p\|^2 .
    \label{eq:chi2-as-norm}
\end{equation}

Let
\begin{align}
    E_+
    &\triangleq
    C_+^2\max_{i\in[K]}\bigl(1+\varepsilon_{\mathrm{head}}(|a_i|)\bigr),
    \label{eq:Eplus}
    \\
    \Delta_\rho
    &\triangleq
    \max\!\Bigl(
        \bigl|1-\zeta_-C_-^2\bigr|,\;
        \max_{i\in[K]}
        \bigl|C_+^2(1+\varepsilon_{\mathrm{head}}(|a_i|))-1\bigr|
    \Bigr).
    \label{eq:Drho}
\end{align}
By Lemma~\ref{lm:1-formal}, for all $i\in[K]$,
$\rho_i^* \leq E_+ \hat\rho_i$ and
$|\hat\rho_i-\rho_i^*|\leq \Delta_\rho \hat\rho_i$.
We also define
$\hat\rho_{\max}\triangleq \max_i\hat\rho_i$,
$\ell_{\min}\triangleq \min_i\ell_i$,
$\ell_{\max}\triangleq \max_i\ell_i$, and
$\bar\ell\triangleq K^{-1}\sum_i\ell_i$.

\begin{lemma}[Formal version of Lemma~\ref{lm:chi2_bound}]
\label{lm:chi2_bound-formal}
Under Assumptions~\ref{ass:norm}--\ref{ass:lmin}, for every $\mathbf q\in\Delta^K$,
\begin{equation}
    |R_1|
    \;\leq\;
    \alpha_1 \chi^2(\mathbf q\,\|\,\mathbf p)+1,
    \qquad
    |R_2|
    \;\leq\;
    \alpha_2 \chi^2(\mathbf q\,\|\,\mathbf p),
    \label{eq:formal-chi2-residual-bound}
\end{equation}
where
\begin{align}
    \alpha_1
    &=
    \frac{1}{4}
    \left(
    \frac{3\Delta_\rho \hat\rho_{\max}}{\sqrt K}
    +
    \frac{2E_+
    \sqrt{\hat\rho_{\max}\ell_{\max}\sum_i \hat\rho_i\ell_i}}
    {\sqrt K\,\bar\ell}
    +
    \frac{E_+\ell_{\max}\sum_i\hat\rho_i\ell_i}
    {K\bar\ell^2}
    \right)^2,
    \label{eq:alpha1-explicit}
    \\
    \alpha_2
    &=
    \frac{E_+\sum_i\hat\rho_i\ell_i}{K\ell_{\min}}
    \left(
    1+\frac{\sqrt K\,\ell_{\max}}{\ell_{\min}}
    \right)^2 .
    \label{eq:alpha2-explicit}
\end{align}
\end{lemma}

\begin{proof}
We first control $R_1$. By Cauchy--Schwarz and~\eqref{eq:chi2-as-norm},
\begin{align}
    |R_1|
    &=
    \left|
    \sum_k (q_k-p_k)\eta_k
    \right|
    \notag\\
    &\leq
    \sqrt{K\|\mathbf q-\mathbf p\|^2}
    \sqrt{\frac{1}{K}\sum_k\eta_k^2}
    \notag\\
    &=
    \sqrt{\chi^2(\mathbf q\,\|\,\mathbf p)}
    \,\|\boldsymbol\eta\|_{\mathbf p},
    \label{eq:R1-CS}
\end{align}
where $\|\mathbf v\|_{\mathbf p}^2\triangleq K^{-1}\sum_k v_k^2$ under the uniform prior. It remains to bound $\|\boldsymbol\eta\|_{\mathbf p}$.

Introduce the intermediate vector
\begin{equation}
    \tilde g_k
    \triangleq
    \frac{\ell_k}{\sum_i\ell_i}
    \left(
        2\rho_k^*
        -
        \frac{\sum_i\rho_i^*\ell_i}{\sum_i\ell_i}
    \right).
    \label{eq:gtilde-def}
\end{equation}
This vector keeps the exact per-trajectory rates $\rho_i^*$ but uses the same diagonal approximation that leads to $\hat g_k$. Therefore,
\begin{equation}
    \boldsymbol\eta
    =
    (\hat{\mathbf g}-\tilde{\mathbf g})
    +
    (\tilde{\mathbf g}-\mathbf g^*)
    \triangleq
    \boldsymbol\eta_{\mathrm{prox}}
    +
    \boldsymbol\eta_{\mathrm{diag}} .
    \label{eq:eta-split}
\end{equation}

For the proxy-substitution term, subtracting~\eqref{eq:gtilde-def} from~\eqref{eq:ghat-closed} gives
\begin{equation}
    \eta_{\mathrm{prox},k}
    =
    \frac{\ell_k}{\sum_i\ell_i}
    \left(
        2(\hat\rho_k-\rho_k^*)
        -
        \frac{\sum_i(\hat\rho_i-\rho_i^*)\ell_i}{\sum_i\ell_i}
    \right).
    \label{eq:eta-prox-k}
\end{equation}
Using $|\hat\rho_i-\rho_i^*|\leq \Delta_\rho\hat\rho_{\max}$ for all $i$, we obtain
\begin{equation}
    |\eta_{\mathrm{prox},k}|
    \leq
    \frac{3\Delta_\rho\hat\rho_{\max}\ell_k}{\sum_i\ell_i}.
\end{equation}
Hence
\begin{align}
    \|\boldsymbol\eta_{\mathrm{prox}}\|_{\mathbf p}^2
    &=
    \frac{1}{K}\sum_k|\eta_{\mathrm{prox},k}|^2
    \notag\\
    &\leq
    \frac{(3\Delta_\rho\hat\rho_{\max})^2}{K}
    \frac{\sum_k\ell_k^2}{(\sum_i\ell_i)^2}
    \notag\\
    &\leq
    \frac{(3\Delta_\rho\hat\rho_{\max})^2}{K},
\end{align}
where the last inequality uses $\sum_k\ell_k^2\leq(\sum_i\ell_i)^2$. Thus
\begin{equation}
    \|\boldsymbol\eta_{\mathrm{prox}}\|_{\mathbf p}
    \leq
    \frac{3\Delta_\rho\hat\rho_{\max}}{\sqrt K}.
    \label{eq:prox-bound}
\end{equation}

We next bound the diagonal-approximation error. Define
\begin{equation}
    o_k
    \triangleq
    \mathbf g_k^\top\bar{\mathbf g}
    -
    \frac{\|\mathbf g_k\|^2}{K}
    =
    \frac{1}{K}\sum_{i\neq k}G_{ki},
    \qquad
    O
    \triangleq
    \|\bar{\mathbf g}\|^2
    -
    \frac{1}{K^2}\sum_i\|\mathbf g_i\|^2
    =
    \frac{1}{K^2}\sum_{i\neq j}G_{ij}.
    \label{eq:ok-O-def}
\end{equation}
Combining the closed form of $g_k^*$ in Lemma~\ref{lem:taylor-decomp}(i) with~\eqref{eq:gtilde-def} gives
\begin{equation}
    \eta_{\mathrm{diag},k}
    =
    \tilde g_k-g_k^*
    =
    -\frac{2o_k}{\bar\ell}
    +
    \frac{O\ell_k}{\bar\ell^2}.
    \label{eq:eta-diag-k}
\end{equation}
By Lemma~\ref{lem:gram-bound},
$|G_{ij}|\leq E_+\sqrt{\hat\rho_i\hat\rho_j\ell_i\ell_j}$. Therefore,
\begin{align}
    K|o_k|
    &\leq
    \sum_{i\neq k}|G_{ki}|
    \notag\\
    &\leq
    E_+\sqrt{\hat\rho_k\ell_k}
    \sum_{i\neq k}\sqrt{\hat\rho_i\ell_i}
    \notag\\
    &\leq
    E_+\sqrt{\hat\rho_k\ell_k}
    \sqrt{K\sum_i\hat\rho_i\ell_i}.
\end{align}
Taking the maximum over $k$ yields
\begin{equation}
    \max_k|o_k|
    \leq
    \frac{
    E_+\sqrt{\hat\rho_{\max}\ell_{\max}\sum_i\hat\rho_i\ell_i}
    }{\sqrt K}.
    \label{eq:ok-bound}
\end{equation}
Similarly,
\begin{align}
    K^2|O|
    &\leq
    \sum_{i\neq j}|G_{ij}|
    \notag\\
    &\leq
    E_+
    \left(\sum_i\sqrt{\hat\rho_i\ell_i}\right)^2
    \notag\\
    &\leq
    E_+K\sum_i\hat\rho_i\ell_i,
\end{align}
and therefore
\begin{equation}
    |O|
    \leq
    \frac{E_+\sum_i\hat\rho_i\ell_i}{K}.
    \label{eq:O-bound}
\end{equation}
Using~\eqref{eq:eta-diag-k},~\eqref{eq:ok-bound},~\eqref{eq:O-bound}, and $\|\mathbf v\|_{\mathbf p}\leq\max_k|v_k|$, we get
\begin{align}
    \|\boldsymbol\eta_{\mathrm{diag}}\|_{\mathbf p}
    &\leq
    \frac{2\max_k|o_k|}{\bar\ell}
    +
    \frac{|O|\ell_{\max}}{\bar\ell^2}
    \notag\\
    &\leq
    \frac{
    2E_+\sqrt{\hat\rho_{\max}\ell_{\max}\sum_i\hat\rho_i\ell_i}
    }{\sqrt K\,\bar\ell}
    +
    \frac{
    E_+\ell_{\max}\sum_i\hat\rho_i\ell_i
    }{K\bar\ell^2}.
    \label{eq:diag-bound}
\end{align}
Combining~\eqref{eq:prox-bound} and~\eqref{eq:diag-bound} gives
\begin{equation}
    \|\boldsymbol\eta\|_{\mathbf p}
    \leq
    \alpha_1^{\mathrm{rms}},
    \label{eq:eta-rms-bound}
\end{equation}
where
\begin{equation}
    \alpha_1^{\mathrm{rms}}
    \triangleq
    \frac{3\Delta_\rho\hat\rho_{\max}}{\sqrt K}
    +
    \frac{
    2E_+\sqrt{\hat\rho_{\max}\ell_{\max}\sum_i\hat\rho_i\ell_i}
    }{\sqrt K\,\bar\ell}
    +
    \frac{
    E_+\ell_{\max}\sum_i\hat\rho_i\ell_i
    }{K\bar\ell^2}.
    \label{eq:alpha1-rms-def}
\end{equation}

If $\alpha_1^{\mathrm{rms}}=0$, then $\boldsymbol\eta=\mathbf 0$ and hence $R_1=0$, so the desired bound is trivial. Otherwise, applying Young's inequality with
$a=\sqrt{\chi^2(\mathbf q\,\|\,\mathbf p)}$,
$b=\|\boldsymbol\eta\|_{\mathbf p}$, and
$\gamma=(\alpha_1^{\mathrm{rms}})^2/2$ gives
\begin{align}
    |R_1|
    &\leq
    \sqrt{\chi^2(\mathbf q\,\|\,\mathbf p)}
    \,\|\boldsymbol\eta\|_{\mathbf p}
    \notag\\
    &\leq
    \frac{\gamma}{2}\chi^2(\mathbf q\,\|\,\mathbf p)
    +
    \frac{1}{2\gamma}\|\boldsymbol\eta\|_{\mathbf p}^2
    \notag\\
    &\leq
    \frac{(\alpha_1^{\mathrm{rms}})^2}{4}
    \chi^2(\mathbf q\,\|\,\mathbf p)
    +
    1.
\end{align}
This proves the first bound in~\eqref{eq:formal-chi2-residual-bound}, with
$\alpha_1=(\alpha_1^{\mathrm{rms}})^2/4$, which is exactly~\eqref{eq:alpha1-explicit}.

It remains to control $R_2$. By Assumption~\ref{ass:lmin}, $D(\mathbf q)=\boldsymbol\ell^\top\mathbf q\geq \ell_{\min}>0$ on $\Delta^K$, so $\rho(\boldsymbol\theta_{\mathrm{ref}};\mathbf q)$ is twice continuously differentiable on the simplex. The second-order Taylor remainder gives
\begin{equation}
    |R_2|
    \leq
    \frac{M}{2}\|\mathbf q-\mathbf p\|^2
    =
    \frac{M}{2K}\chi^2(\mathbf q\,\|\,\mathbf p),
    \qquad
    M
    \triangleq
    \sup_{\mathbf q\in\Delta^K}
    \bigl\|
    \nabla_{\mathbf q}^2\rho(\boldsymbol\theta_{\mathrm{ref}};\mathbf q)
    \bigr\|_{\mathrm{op}}.
    \label{eq:R2-Hessian}
\end{equation}
We now bound $M$. By Lemma~\ref{lem:hessian-form},
\begin{equation}
    \nabla_{\mathbf q}^2\rho
    =
    \frac{1}{D(\mathbf q)}
    \left[
        2G
        -
        \boldsymbol\ell(\nabla\rho)^\top
        -
        (\nabla\rho)\boldsymbol\ell^\top
    \right],
\end{equation}
which implies
\begin{equation}
    \|\nabla_{\mathbf q}^2\rho\|_{\mathrm{op}}
    \leq
    \frac{
        2\|G\|_{\mathrm{op}}
        +
        2\|\boldsymbol\ell\|\|\nabla\rho\|
    }{D(\mathbf q)}.
    \label{eq:hessian-pre}
\end{equation}
For any $\mathbf q\in\Delta^K$,
\begin{equation}
    \rho(\boldsymbol\theta_{\mathrm{ref}};\mathbf q)
    =
    \frac{\mathbf q^\top G\mathbf q}{D(\mathbf q)}
    \leq
    \frac{\|G\|_{\mathrm{op}}\|\mathbf q\|^2}{\ell_{\min}}
    \leq
    \frac{\|G\|_{\mathrm{op}}}{\ell_{\min}},
    \label{eq:rho-upper}
\end{equation}
where we used $\|\mathbf q\|^2\leq \sum_k q_k=1$. Since
$\nabla\rho=(2G\mathbf q-\rho\boldsymbol\ell)/D(\mathbf q)$,
\begin{equation}
    \|\nabla\rho\|
    \leq
    \frac{2\|G\|_{\mathrm{op}}\|\mathbf q\|+\rho\|\boldsymbol\ell\|}{D(\mathbf q)}
    \leq
    \frac{\|G\|_{\mathrm{op}}}{\ell_{\min}}
    \left(
        2+\frac{\|\boldsymbol\ell\|}{\ell_{\min}}
    \right).
    \label{eq:grad-rho-bound}
\end{equation}
Substituting~\eqref{eq:grad-rho-bound} into~\eqref{eq:hessian-pre} and using $D(\mathbf q)\geq \ell_{\min}$ gives
\begin{equation}
    \|\nabla_{\mathbf q}^2\rho\|_{\mathrm{op}}
    \leq
    \frac{2\|G\|_{\mathrm{op}}}{\ell_{\min}}
    \left(
        1+\frac{\|\boldsymbol\ell\|}{\ell_{\min}}
    \right)^2 .
    \label{eq:hessian-bound-G}
\end{equation}
By Lemma~\ref{lem:gram-bound},
\begin{equation}
    \|G\|_{\mathrm{op}}
    \leq
    \|G\|_F
    \leq
    E_+\sqrt{\sum_{i,j}\hat\rho_i\hat\rho_j\ell_i\ell_j}
    =
    E_+\sum_i\hat\rho_i\ell_i .
    \label{eq:Gop-forward-bound}
\end{equation}
Together with $\|\boldsymbol\ell\|\leq \sqrt K\,\ell_{\max}$, this yields
\begin{equation}
    M
    \leq
    \frac{2E_+\sum_i\hat\rho_i\ell_i}{\ell_{\min}}
    \left(
        1+\frac{\sqrt K\,\ell_{\max}}{\ell_{\min}}
    \right)^2.
    \label{eq:M-final}
\end{equation}
Substituting~\eqref{eq:M-final} into~\eqref{eq:R2-Hessian} gives
\[
    |R_2|
    \leq
    \alpha_2\chi^2(\mathbf q\,\|\,\mathbf p),
\]
with $\alpha_2$ as in~\eqref{eq:alpha2-explicit}. This proves the second bound in~\eqref{eq:formal-chi2-residual-bound} and completes the proof.
\end{proof}

\begin{remark}[Operational meaning of $\alpha_1^{\mathrm{rms}}$]
\label{rem:alpha1-rms}
The quantity $\alpha_1^{\mathrm{rms}}$ defined in~\eqref{eq:alpha1-rms-def} is a $\mathbf p$-weighted RMS bound on the surrogate error:
\begin{equation}
    \sqrt{\frac{1}{K}\sum_{k=1}^K|\hat g_k-g_k^*|^2}
    =
    \|\hat{\mathbf g}-\mathbf g^*\|_{\mathbf p}
    \leq
    \alpha_1^{\mathrm{rms}} .
    \label{eq:alpha1-rms-interp}
\end{equation}
The constant $\alpha_1=(\alpha_1^{\mathrm{rms}})^2/4$ in~\eqref{eq:alpha1-explicit} is the coefficient obtained after absorbing the square-root residual into the $\chi^2$ term through Young's inequality. This is an aggregate bound rather than a per-coordinate one: in the worst case, an individual error $|\hat g_k-g_k^*|$ can be as large as $\sqrt K\,\alpha_1^{\mathrm{rms}}$. The use of the RMS norm is natural here because the Cauchy--Schwarz step in~\eqref{eq:R1-CS} pairs $\sqrt{\chi^2(\mathbf q\,\|\,\mathbf p)}$ with $\|\boldsymbol\eta\|_{\mathbf p}$.
\end{remark}
\subsection{Proof of Theorem~\ref{thm:lark-objective} (LARK Objective)}
\label{app:proof-thm1}

This subsection states and proves the formal version of Theorem~\ref{thm:lark-objective}. The informal statement in the main text asserts the existence of finite constants $\alpha > 0$ and $C \in \mathbb R$, independent of the optimization variable $\mathbf q$, realizing a $\chi^2$-regularized lower bound on the learnability improvement; the formal version below makes both constants explicit in terms of the constants $\alpha_1, \alpha_2$ from Lemma~\ref{lm:chi2_bound-formal}.

\begin{theorem}[Formal version of Theorem~\ref{thm:lark-objective}]
\label{thm:lark-formal}
Under Assumptions~\ref{ass:norm}--\ref{ass:lmin}, for every $\mathbf q \in \Delta^K$,
\begin{equation}
    \rho(\boldsymbol\theta_{\mathrm{ref}};\, \mathbf q) - \rho(\boldsymbol\theta_{\mathrm{ref}};\, \mathbf p)
    \;\geq\;
    \langle \mathbf q - \mathbf p,\, \hat{\mathbf g}\rangle
    \;-\; \alpha\, \chi^2(\mathbf q\,\|\,\mathbf p)
    \;+\; C,
    \label{eq:lark-formal}
\end{equation}
where
\begin{equation}
    \alpha \;\triangleq\; \alpha_1 + \alpha_2 \;>\; 0,
    \qquad
    C \;\triangleq\; -1,
    \label{eq:alpha-C-explicit}
\end{equation}
and $\alpha_1, \alpha_2 > 0$ are the explicit constants given by Lemma~\ref{lm:chi2_bound-formal}. In particular, $\alpha$ and $C$ are independent of $\mathbf q$.
\end{theorem}

\begin{proof}
By Lemma~\ref{lem:taylor-decomp}(iii), the change in $\rho$ admits the exact decomposition
\begin{equation}
    \rho(\boldsymbol\theta_{\mathrm{ref}};\, \mathbf q) - \rho(\boldsymbol\theta_{\mathrm{ref}};\, \mathbf p)
    \;=\;
    \langle \mathbf q - \mathbf p,\, \hat{\mathbf g}\rangle
    \;+\; R_1
    \;+\; R_2.
    \label{eq:thm1-decomp}
\end{equation}
By Lemma~\ref{lm:chi2_bound-formal}, the two residuals satisfy
\begin{equation}
    |R_1| \;\leq\; \alpha_1\, \chi^2(\mathbf q\,\|\,\mathbf p) + 1,
    \qquad
    |R_2| \;\leq\; \alpha_2\, \chi^2(\mathbf q\,\|\,\mathbf p),
    \label{eq:R1R2-bounds}
\end{equation}
where $\alpha_1, \alpha_2 > 0$ are given by~\eqref{eq:alpha1-explicit} and~\eqref{eq:alpha2-explicit}, respectively. Combining~\eqref{eq:thm1-decomp} and~\eqref{eq:R1R2-bounds} gives
\begin{align}
    \rho(\boldsymbol\theta_{\mathrm{ref}};\, \mathbf q) - \rho(\boldsymbol\theta_{\mathrm{ref}};\, \mathbf p)
    &\;\stackrel{(a)}{=}\;
    \langle \mathbf q - \mathbf p,\, \hat{\mathbf g}\rangle
    \;+\; R_1 \;+\; R_2
    \notag\\
    &\;\stackrel{(b)}{\geq}\;
    \langle \mathbf q - \mathbf p,\, \hat{\mathbf g}\rangle
    \;-\; |R_1| \;-\; |R_2|
    \notag\\
    &\;\stackrel{(c)}{\geq}\;
    \langle \mathbf q - \mathbf p,\, \hat{\mathbf g}\rangle
    \;-\; \bigl(\alpha_1\, \chi^2(\mathbf q\,\|\,\mathbf p) + 1\bigr)
    \;-\; \alpha_2\, \chi^2(\mathbf q\,\|\,\mathbf p)
    \notag\\
    &\;\stackrel{(d)}{=}\;
    \langle \mathbf q - \mathbf p,\, \hat{\mathbf g}\rangle
    \;-\; (\alpha_1 + \alpha_2)\, \chi^2(\mathbf q\,\|\,\mathbf p)
    \;-\; 1,
\end{align}
where (a) is~\eqref{eq:thm1-decomp}; (b) uses $R_1, R_2 \geq -|R_1|, -|R_2|$; (c) substitutes the bounds~\eqref{eq:R1R2-bounds} from Lemma~\ref{lm:chi2_bound-formal}; and (d) collects the two $\chi^2$ terms. Identifying $\alpha = \alpha_1 + \alpha_2$ and $C = -1$ as in~\eqref{eq:alpha-C-explicit} yields~\eqref{eq:lark-formal}.
\end{proof}

\begin{remark}[Correspondence with the informal constants $\alpha$ and $C$]
\label{rem:alpha-C-correspondence}
Theorem~\ref{thm:lark-formal} instantiates the informal constants $\alpha, C$ of Theorem~\ref{thm:lark-objective} as
$\alpha = \alpha_1 + \alpha_2$ and $C = -1$, where $\alpha_1, \alpha_2$ are the forward-computable constants from Lemma~\ref{lm:chi2_bound-formal}. The resulting bound depends on $\mathbf q$ only through the linear term $\langle \mathbf q - \mathbf p, \hat{\mathbf g}\rangle$ and the $\chi^2$ regularization. Since $C = -1$ is independent of $\mathbf q$, it does not affect the maximizer of the right-hand side over $\Delta^K$. The coefficient $\alpha$ motivates the use of a $\chi^2$ penalty; in the fixed-budget selection rule of Section~\ref{sec:algorithm}, the penalty temperature is calibrated as $\tau(B)$ so that the optimizer has exactly $B$ active trajectories.
\end{remark}
\subsection{Proof of Lemma~\ref{prop:b_param} (\texorpdfstring{$B$}{B}-Parameterized Closed Form)}
\label{app:proof-prop2}

Theorem~\ref{thm:lark-formal} motivates a $\chi^2$-regularized linear surrogate. In the fixed-budget setting, we calibrate the regularization temperature so that the optimizer has exactly $B$ active coordinates. By the equivalence~\eqref{eq:chi2-as-norm} between the $\chi^2$-divergence and the squared Euclidean distance under the uniform prior, the corresponding objective can be written as
\begin{equation}
    \hat{\mathbf q}
    \;=\;
    \arg\max_{\mathbf q \in \Delta^K}
    \bigg\{
        \langle \mathbf q - \mathbf p,\, \hat{\mathbf g}\rangle
        \;-\; \frac{\tau\, K}{2}\, \|\mathbf q - \mathbf p\|^2
    \bigg\},
    \label{eq:lark-quadratic}
\end{equation}
where we write $\tau$ for the temperature to be chosen as a function of the budget $B$.

\begin{proof}[Proof of Lemma~\ref{prop:b_param}]
The objective~\eqref{eq:lark-quadratic} is strictly concave in $\mathbf q$ whenever $\tau K>0$, so the KKT conditions characterize its unique global maximizer over $\Delta^K$. Introduce a Lagrange multiplier $\lambda\in\mathbb R$ for the equality constraint $\sum_k q_k=1$ and KKT multipliers $\mu_k\geq 0$ for the non-negativity constraints $q_k\geq 0$. The Lagrangian is
\begin{equation}
    \mathcal L(\mathbf q,\lambda,\boldsymbol\mu)
    =
    \langle \mathbf q-\mathbf p,\,\hat{\mathbf g}\rangle
    -
    \frac{\tau K}{2}\|\mathbf q-\mathbf p\|^2
    -
    \lambda\Bigl(\sum_k q_k-1\Bigr)
    +
    \sum_k \mu_k q_k .
\end{equation}
Stationarity and complementary slackness give
\begin{equation}
    \hat g_k-\tau K(q_k-p_k)-\lambda+\mu_k=0,
    \qquad
    \mu_k q_k=0,
    \qquad
    \mu_k\geq 0.
    \label{eq:kkt-stationarity}
\end{equation}
Since $p_k=1/K$, these conditions imply the projection form
\begin{equation}
    \hat q_k
    =
    \left[
        \frac{1}{K}
        +
        \frac{\hat g_k-\lambda}{\tau K}
    \right]_+ ,
    \label{eq:kkt-projection}
\end{equation}
where $[x]_+=\max\{x,0\}$. The threshold $\lambda$ is determined by the normalization condition $\sum_k \hat q_k=1$.

Sort the scores as $\hat g_{(1)}\geq \hat g_{(2)}\geq\cdots\geq \hat g_{(K)}$, where $(i)$ denotes the sorted index. From~\eqref{eq:kkt-projection}, $\hat q_{(i)}>0$ if and only if $\hat g_{(i)}>\lambda-\tau$. Thus an active set of size $B$ is obtained whenever
\begin{equation}
    \hat g_{(B)}
    >
    \lambda-\tau
    \geq
    \hat g_{(B+1)}.
    \label{eq:active-set-condition}
\end{equation}
Under the assumption $\hat g_{(B)}>\hat g_{(B+1)}$, the top-$B$ active set is well-defined. We choose the boundary value
\begin{equation}
    \lambda
    =
    \hat g_{(B+1)}+\tau,
    \label{eq:lambda-threshold}
\end{equation}
which makes the $(B+1)$-st coordinate exactly inactive while keeping the first $B$ coordinates positive.

Substituting~\eqref{eq:lambda-threshold} into~\eqref{eq:kkt-projection}, for every $i\leq B$ we obtain
\begin{align}
    \hat q_{(i)}
    &=
    \frac{1}{K}
    +
    \frac{\hat g_{(i)}-\hat g_{(B+1)}-\tau}{\tau K}
    \notag\\
    &=
    \frac{\hat g_{(i)}-\hat g_{(B+1)}}{\tau K}.
    \label{eq:active-weights-tau}
\end{align}
For $i>B$, the projection in~\eqref{eq:kkt-projection} gives $\hat q_{(i)}=0$. Imposing the simplex normalization on the active coordinates gives
\begin{equation}
    1
    =
    \sum_{i=1}^B \hat q_{(i)}
    =
    \frac{1}{\tau K}
    \sum_{i=1}^B
    \bigl(\hat g_{(i)}-\hat g_{(B+1)}\bigr).
\end{equation}
Therefore the budget-calibrated temperature is
\begin{equation}
    \tau^*(B)
    =
    \frac{1}{K}
    \sum_{i=1}^{B}
    \bigl(\hat g_{(i)}-\hat g_{(B+1)}\bigr),
    \label{eq:tau-star}
\end{equation}
which is positive because $\hat g_{(B)}>\hat g_{(B+1)}$. Substituting~\eqref{eq:tau-star} into~\eqref{eq:active-weights-tau} yields
\begin{equation}
    \hat q_{(i)}
    =
    \frac{\hat g_{(i)}-\hat g_{(B+1)}}
    {\sum_{j=1}^{B}\bigl(\hat g_{(j)}-\hat g_{(B+1)}\bigr)}
    \quad \text{for } i\leq B,
    \qquad
    \hat q_{(i)}=0
    \quad \text{for } i>B.
\end{equation}
This is the closed-form $B$-parameterized weighting rule stated in Lemma~\ref{prop:b_param}.
\end{proof}
\subsection{Theoretical and Empirical Validation of the Surrogate \texorpdfstring{$\hat{\mathbf g}$}{g\textasciicircum}}
\label{app:ghat-empirical}

The forward-pass surrogate $\hat{\mathbf g}$ is the central computational shortcut underlying the LARK algorithm: it replaces the exact directional derivative
$\mathbf g^* = \nabla_{\mathbf q}\rho(\boldsymbol\theta_{\mathrm{ref}};\mathbf q)|_{\mathbf q=\mathbf p}$,
whose evaluation requires one backward pass per candidate trajectory, with a quantity computable from forward passes alone. This subsection studies the relationship between $\hat{\mathbf g}$ and $\mathbf g^*$ from two complementary perspectives: a theoretical RMS error bound and an empirical comparison on real student models.

\subsubsection{Theoretical Error Bound}
\label{app:ghat-theoretical}

The bounds derived in Appendix~\ref{app:proof-lemma2} immediately yield a quantitative bound on $\|\hat{\mathbf g}-\mathbf g^*\|$ in the $\mathbf p$-weighted norm
$\|\mathbf v\|_{\mathbf p}^2 \triangleq K^{-1}\sum_k v_k^2$.

\begin{corollary}[$\mathbf p$-weighted RMS error of $\hat{\mathbf g}$]
\label{cor:ghat-rms}
Under Assumptions~\ref{ass:norm}--\ref{ass:lmin},
\begin{equation}
    \|\hat{\mathbf g} - \mathbf g^*\|_{\mathbf p}
    \;=\;
    \sqrt{\frac{1}{K}\sum_{k=1}^K \bigl|\hat g_k - g_k^*\bigr|^2}
    \;\leq\;
    \alpha_1^{\mathrm{rms}},
    \label{eq:ghat-rms-bound}
\end{equation}
where $\alpha_1^{\mathrm{rms}}$ is the RMS upper bound from~\eqref{eq:alpha1-rms-def}, given by
\begin{equation}
    \alpha_1^{\mathrm{rms}}
    \;=\;
    \underbrace{\frac{3\,\Delta_\rho\, \hat\rho_{\max}}{\sqrt{K}}}_{\text{proxy term}}
    \;+\;
    \underbrace{\frac{2\, E_+\, \sqrt{\hat\rho_{\max}\, \ell_{\max}\, \sum_i \hat\rho_i\, \ell_i}}{\sqrt{K}\, \bar\ell}}_{\text{off-diagonal Gram term}}
    \;+\;
    \underbrace{\frac{E_+\, \ell_{\max}\, \sum_i \hat\rho_i\, \ell_i}{K\, \bar\ell^{\,2}}}_{\text{aggregate off-diagonal term}}.
    \label{eq:alpha1-rms-three-terms}
\end{equation}
The constant $\alpha_1$ in Lemma~\ref{lm:chi2_bound-formal} is recovered as
$\alpha_1=(\alpha_1^{\mathrm{rms}})^2/4$.
\end{corollary}

\begin{proof}
By the triangle inequality,
$\|\hat{\mathbf g}-\mathbf g^*\|_{\mathbf p}
\leq
\|\boldsymbol\eta_{\mathrm{prox}}\|_{\mathbf p}
+
\|\boldsymbol\eta_{\mathrm{diag}}\|_{\mathbf p}$,
where $\boldsymbol\eta_{\mathrm{prox}}$ and $\boldsymbol\eta_{\mathrm{diag}}$ are defined in the proof of Lemma~\ref{lm:chi2_bound-formal}. Combining the bounds in~\eqref{eq:prox-bound} and~\eqref{eq:diag-bound} gives~\eqref{eq:ghat-rms-bound}, with $\alpha_1^{\mathrm{rms}}$ given by~\eqref{eq:alpha1-rms-three-terms}.
\end{proof}

The three terms in~\eqref{eq:alpha1-rms-three-terms} correspond to the three sources of surrogate error. The first term,
$3\Delta_\rho \hat\rho_{\max}/\sqrt K$, is the proxy-substitution error: it comes from replacing the exact per-trajectory rates $\rho_i^*$ in the closed form of $g_k^*$ by the forward-pass proxies $\hat\rho_i$. Its size is controlled by $\Delta_\rho$, which measures the multiplicative gap between $\rho_i^*$ and $\hat\rho_i$ in Lemma~\ref{lm:1-formal}. The second term,
$2E_+\sqrt{\hat\rho_{\max}\ell_{\max}\sum_i\hat\rho_i\ell_i}/(\sqrt K\,\bar\ell)$, controls the per-coordinate off-diagonal Gram residual
$o_k=\mathbf g_k^\top\bar{\mathbf g}-\|\mathbf g_k\|^2/K$, which is ignored by the diagonal approximation
$\mathbf g_k^\top\bar{\mathbf g}\approx \|\mathbf g_k\|^2/K$. The third term,
$E_+\ell_{\max}\sum_i\hat\rho_i\ell_i/(K\bar\ell^2)$, controls the aggregate off-diagonal mass
$O=\|\bar{\mathbf g}\|^2-K^{-2}\sum_i\|\mathbf g_i\|^2$ left over by the approximation
$\|\bar{\mathbf g}\|^2\approx K^{-2}\sum_i\|\mathbf g_i\|^2$. The latter two terms are small when the per-trajectory gradients are close to orthogonal in parameter space.

The bound in~\eqref{eq:ghat-rms-bound} is an aggregate RMS bound across the $K$ candidates, not a per-coordinate guarantee. By the equivalence of finite-dimensional norms,
\begin{equation}
    \max_{k\in[K]} |\hat g_k-g_k^*|
    \;\leq\;
    \sqrt{\sum_{k=1}^K |\hat g_k-g_k^*|^2}
    \;=\;
    \sqrt K\,\|\hat{\mathbf g}-\mathbf g^*\|_{\mathbf p}
    \;\leq\;
    \sqrt K\,\alpha_1^{\mathrm{rms}}.
    \label{eq:ghat-linfty-bound}
\end{equation}
Thus a single coordinate can have a larger error in the worst case, while the RMS error remains controlled. This RMS norm is the relevant quantity for the residual analysis because the Cauchy--Schwarz step in~\eqref{eq:R1-CS} naturally pairs $\sqrt{\chi^2(\mathbf q\|\mathbf p)}$ with $\|\boldsymbol\eta\|_{\mathbf p}$.

\subsubsection{Empirical Correlation between $\hat{\mathbf g}$ and $\mathbf g^*$}
\label{app:ghat-empirical-corr}

The bound in Corollary~\ref{cor:ghat-rms} is a worst-case guarantee whose constants depend on the structural quantities of Lemma~\ref{lm:1-formal}. We complement this analysis by directly measuring how well $\hat{\mathbf g}$ tracks $\mathbf g^*$ on a real student model.

We sample $50$ questions uniformly from the candidate pool described in Appendix~\ref{app:exp_pool}. For each question, we compute both the LARK surrogate $\hat g_k$ and the exact directional derivative
$g_k^*=\partial\rho(\boldsymbol\theta_{\mathrm{ref}};\mathbf q)/\partial q_k|_{\mathbf q=\mathbf p}$.
The exact derivative is obtained by one full-parameter backward pass per candidate trajectory at the anchor model $\boldsymbol\theta_{\mathrm{ref}}$ (Qwen-2.5-7B), rather than by a low-rank or kernel approximation. Across $50$ questions and $K=33$ candidates per question, this produces $1{,}650$ paired values $(\hat g_k,g_k^*)$. To compare values across questions, we apply a per-question affine rescaling that aligns the mean and scale of $\hat g_k$ with those of $g_k^*$. This rescaling preserves within-question rankings and therefore does not affect $\mathrm{Recall@}B$.

We report three metrics. First, we compute pooled Pearson $r$ and Spearman $\rho$ between the rescaled $\hat g_k$ and $g_k^*$ to measure scalar agreement. Second, because LARK uses $\hat{\mathbf g}$ primarily through its top-$B$ ranking, we report
$\mathrm{Recall@}B \triangleq |\mathrm{top}_B(\hat{\mathbf g})\cap \mathrm{top}_B(\mathbf g^*)|/B$
averaged over questions for $B\in\{1,\ldots,10\}$, with random baseline $B/K$. Third, we report the per-question relative $\mathbf p$-weighted RMS error
$\|\hat{\mathbf g}-\mathbf g^*\|_{\mathbf p}/\|\mathbf g^*\|_{\mathbf p}$, matching the norm used in Corollary~\ref{cor:ghat-rms}.

\begin{figure}[t]
\centering
\includegraphics[width=\linewidth]{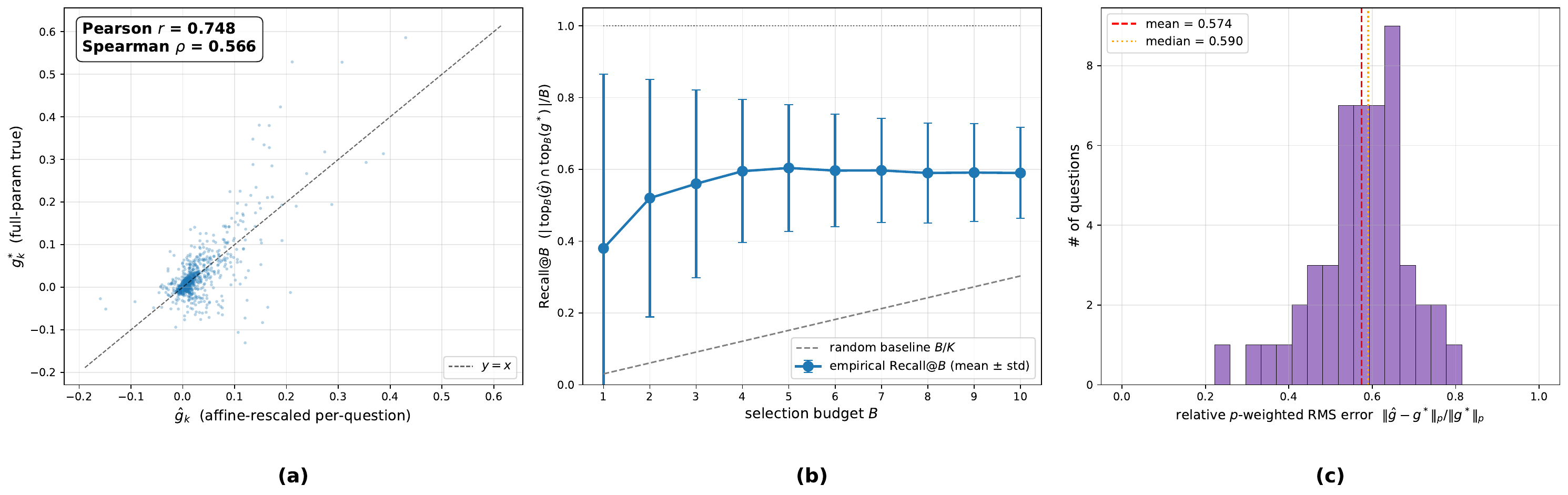}
\caption{\textbf{Empirical correlation between $\hat{\mathbf g}$ and $\mathbf g^*$ on Qwen-2.5-7B.}
\textbf{(a)} Per-question affine-rescaled $\hat g_k$ versus exact $g_k^*$ across $50 \times 33$ candidates; pooled Pearson $r = 0.748$ and Spearman $\rho = 0.566$.
\textbf{(b)} $\mathrm{Recall@}B$ between the $\mathrm{top}_B$ rankings of $\hat{\mathbf g}$ and $\mathbf g^*$ as a function of selection budget $B$; the dashed gray line is the random baseline $B/K$.
\textbf{(c)} Histogram of per-question relative $\mathbf p$-weighted RMS error $\|\hat{\mathbf g}-\mathbf g^*\|_{\mathbf p}/\|\mathbf g^*\|_{\mathbf p}$ across the $50$ verification questions; mean $0.574$, median $0.590$.}
\label{fig:ghat-empirical}
\end{figure}

Figure~\ref{fig:ghat-empirical}(a) shows a clear positive relationship between $\hat g_k$ and $g_k^*$, with pooled Pearson $r=0.748$ and Spearman $\rho=0.566$. This indicates that the surrogate captures a substantial portion of the scalar variation in $g_k^*$ after per-question affine rescaling, while the remaining gap is consistent with the off-diagonal Gram interactions omitted by the diagonal approximation in~\eqref{eq:alpha1-rms-three-terms}.

The ranking agreement is also substantially above chance. As shown in Figure~\ref{fig:ghat-empirical}(b), $\mathrm{Recall@}1\approx 0.38$, far above the random baseline $1/K\approx 0.03$. At the main experimental budget $B=3$, $\mathrm{Recall@}3\approx 0.56$, compared with the random baseline $3/K\approx 0.09$. The recall remains around $0.6$ for larger budgets, indicating that $\hat{\mathbf g}$ preserves meaningful top-$B$ ranking information for the selection rule used by LARK. The error bars also show nontrivial question-level variation, so this result should be interpreted as strong aggregate alignment rather than a perfect per-question match.

Figure~\ref{fig:ghat-empirical}(c) shows that the relative $\mathbf p$-weighted RMS error concentrates around $0.6$ (mean $0.574$, median $0.590$), which is far below the loose worst-case per-coordinate factor $\sqrt K\approx 5.7$ allowed by~\eqref{eq:ghat-linfty-bound}. The empirical error is therefore much smaller than the theoretical worst case, as expected in non-adversarial candidate pools.

Overall, these diagnostics support the use of $\hat{\mathbf g}$ as a forward-pass surrogate for top-$B$ trajectory selection. The surrogate is not a high-precision coordinate-wise estimator of $\mathbf g^*$, but it is positively aligned with $\mathbf g^*$ at the scalar level and preserves useful top-$B$ ranking information. This is the operative notion of agreement for LARK, since the closed-form selection rule in Lemma~\ref{prop:b_param} depends on the top-$B$ ordering and the score margins above the threshold, rather than on the exact magnitude of every coordinate. The empirical results therefore support the computational tradeoff made by LARK: avoiding $K$ full-parameter backward passes per question while retaining the ranking signal needed for downstream selection.
\section{Experimental Details}
\label{app:exp_details}

\subsection{Candidate Pool}
\label{app:exp_pool}

We adopt the trajectory candidate pool released by \citet{yang2026reasoning}
without modification. The pool consists of 5{,}000 mathematical
reasoning problems sampled from NuminaMath~\citep{li2024numinamath},
each paired with 33 reasoning trajectories produced by 11 teacher
models with 3 independent rollouts per teacher. All selection methods
in our experiments---including LARK and every baseline in
Appendix~\ref{app:exp_baselines}---operate on this identical pool, so
any difference in downstream performance is attributable to the
selection rule rather than to differences in the underlying candidate
set.

\paragraph{Teacher pool.}
The 11 teacher models span a range of scales and model families,
including both proprietary and open-weight reasoning models:
GPT-OSS-20B, GPT-OSS-120B, Phi-4-R+, QwQ-32B, DeepSeek-R1,
Qwen3-235B, Qwen3-30B, Qwen3-8B, Qwen3-4B, Nemotron, and Magistral-S.
For each (problem, teacher) pair, three independent rollouts are
generated, yielding $11 \times 3 = 33$ trajectories per problem.
We refer the reader to \citet{yang2026reasoning} for the full inference
configuration of each teacher (decoding temperature, top-$p$, maximum
generation length).

\paragraph{Correctness filtering.}
Every trajectory retained in the pool has been verified to reach the
correct final answer under the extraction rule of \citet{yang2026reasoning},
which matches the rule described in Appendix~\ref{app:exp_eval}.
Trajectories whose final answer disagrees with the ground truth are
not included. As a result, all 33 candidates for each problem are
\emph{correct} reasoning trajectories that nonetheless differ in
reasoning style, length, and student-side learnability---a setting
that isolates the selection problem from the orthogonal question of
correctness verification.

\subsection{Trajectory Selection Baselines}
\label{app:exp_baselines}

We compare LARK against the seven baselines reported in
Table~\ref{tab:main_weighted_math_merged}. Throughout this subsection, let
$\mathbf{x}$ denote the prompt (system and user turns), and let
$\{\mathbf{y}_k\}_{k=1}^{K}$ with $K=33$ denote the candidate
trajectories for $\mathbf{x}$ in the pool $\mathcal{C}(\mathbf{x})$
described in Appendix~\ref{app:exp_pool}. Each trajectory has length
$|a_k|$, with $t$-th token $y_t^k$, predictive distribution
$\boldsymbol{\pi}_t^k = \pi_{\boldsymbol{\theta}}(\cdot \mid \mathbf{x},
\mathbf{y}_{<t}^k) \in \Delta^{|\mathcal{V}|}$ under the student
$\pi_{\boldsymbol{\theta}}$, ground-truth one-hot
$\boldsymbol{\delta}(y_t^k) \in \Delta^{|\mathcal{V}|}$, and
length-normalized cross-entropy
$\ell_k = \tfrac{1}{|a_k|}\sum_{t=1}^{|a_k|} -\log
\boldsymbol{\pi}_t^k(y_t^k)$. Each baseline assigns a scalar score
to every candidate $\mathbf{y}_k$ and selects the top-$B$ trajectories
under that score with uniform weights $1/B$; LARK instead applies its
closed-form $\chi^2$-$B$ rule to obtain non-uniform weights.

\paragraph{Random.}
We sample one trajectory uniformly at random from
$\{\mathbf{y}_k\}_{k=1}^{K}$:
\[
\mathbf{y}^{\star} \sim \mathrm{Uniform}\bigl(\{\mathbf{y}_k\}_{k=1}^{K}\bigr).
\]
For $B>1$, we sample $B$ trajectories without replacement.

\paragraph{Token Length\textsubscript{max}.}
We score each candidate by its token length and select the longest:
\[
\mathrm{score}_{\text{len}}(\mathbf{y}_k) \triangleq |a_k|,
\qquad
\mathbf{y}^{\star} = \arg\max_{k\in[K]} \mathrm{score}_{\text{len}}(\mathbf{y}_k).
\]

\paragraph{Rule-based Quality\textsubscript{max} (LIMO-style).}
We implement a lightweight heuristic inspired by LIMO-style
filtering~\citep{ye2025limo}, combining four indicators computed from
the assistant text of $\mathbf{y}_k$:
(i) elaborated reasoning (total word length),
(ii) self-verification (frequency of ``check'' and ``verify''),
(iii) exploratory approach (frequency of ``perhaps'' and ``might''),
(iv) adaptive granularity (frequency of ``therefore'' and ``since'').
For each candidate, the frequency of each keyword group is normalized
by the total word count. We $z$-score each indicator across the $K$
candidates within the same pool and take a weighted sum:
\[
\mathrm{score}_{\text{rule}}(\mathbf{y}_k)
= 0.30 \cdot z(\mathrm{len})
+ 0.20 \cdot z(\mathrm{selfver})
+ 0.25 \cdot z(\mathrm{explore})
+ 0.25 \cdot z(\mathrm{granularity}),
\]
where $z(\cdot)$ denotes within-pool $z$-scoring. We select
$\mathbf{y}^{\star} = \arg\max_{k\in[K]}
\mathrm{score}_{\text{rule}}(\mathbf{y}_k)$.

\paragraph{LLM-judged Quality\textsubscript{max}.}
We use \texttt{Qwen3-32B-Instruct}~\citep{yang2025qwen3} as a judge
model in non-thinking mode. Given $(\mathbf{x}, \mathbf{y}_k)$, the
judge outputs a JSON object with an \texttt{overall\_score}
$s_{\text{judge}}(\mathbf{y}_k) \in [0,1]$ and a textual
\texttt{overall\_reason}. We select
$\mathbf{y}^{\star} = \arg\max_{k\in[K]}
s_{\text{judge}}(\mathbf{y}_k)$. The verbatim judge prompt is given
below.

\begin{prompt}[title={Judge Prompt: LLM-as-a-Judge for Trajectory Quality}]
You are a meticulous and highly critical evaluator of AI reasoning. Your primary goal is to identify and quantify
subtle flaws, logical gaps, inefficiencies, and hidden assumptions. Do not default to a high score. Your
starting assumption should be critical, and you must rigorously justify every point awarded.

First, please carefully read the following problem statement:
<Problem>
\texttt{\{question\}}
</Problem>

Now, please carefully read the following candidate's chain-of-thought reasoning:
<Reasoning>
\texttt{\{reasoning\_to\_evaluate\}}
</Reasoning>

When evaluating this reasoning, you must adhere to the following five key evaluation criteria and the scoring
rubric below.

Scoring Guidelines and Calibration:
You must use the full 0.0 to 1.0 scale. Scores should not be clustered at the top. Use this rubric to anchor your
scores:

1.0 (Exceptional/Flawless): Reserved for reasoning that is not only correct but also elegant, insightful, and
comprehensive. It is perfectly structured and leaves no room for doubt. This score should be exceedingly rare.

0.8 - 0.9 (Excellent but Imperfect): The core reasoning is valid and well-supported, but there may be very minor,
superficial issues (e.g., a trivial typo in a formula that doesn't affect the outcome, a slightly awkward
phrasing). The conclusion is unaffected.

0.5 - 0.7 (Competent but Flawed): The reasoning is generally on the right track but contains noticeable and non-trivial flaws.

0.2 - 0.4 (Poor): The reasoning contains fundamental flaws that largely invalidate the process or conclusion.

0.0 - 0.1 (Unacceptable): The reasoning is completely incorrect, irrelevant, nonsensical, or makes no meaningful
attempt to solve the problem.

Evaluation Criteria:
Factual Accuracy, Logical Rigor, Solution Completeness, Reasoning Efficiency, Presentation Quality.

For each criterion, give a score from 0.0 to 1.0 (in 0.1 increments) and a brief justification in JSON.

Your output must be a single, valid JSON object with:
\begin{verbatim}
{
  "dimensional_evaluation": {...},
  "overall_score": <float between 0.0 and 1.0>,
  "overall_reason": "<concise summary>"
}
\end{verbatim}
\end{prompt}

\paragraph{Global Naturalness (GRAPE)\textsubscript{max}~\citep{zhang2025grape}.}
GRAPE scores each candidate by the average log-likelihood assigned to
its tokens by the student under teacher forcing, equivalently the
negative length-normalized cross-entropy:
\[
\mathrm{score}_{\text{GRAPE}}(\mathbf{y}_k)
\;\triangleq\;
-\,\ell_k
\;=\;
\frac{1}{|a_k|}\sum_{t=1}^{|a_k|}
\log \boldsymbol{\pi}_t^{k}(y_t^k).
\]
We select
$\mathbf{y}^{\star} = \arg\max_{k\in[K]}\mathrm{score}_{\text{GRAPE}}(\mathbf{y}_k)$,
which is equivalent to selecting the trajectory of \emph{minimum}
length-normalized cross-entropy under the student.

\paragraph{Local Naturalness\textsubscript{max}~\citep{just2025local}.}
We follow the Local Naturalness metric. Each trajectory is split into
$J_k$ sentence-level steps $\mathbf{y}_k = (\mathbf{s}_1, \ldots,
\mathbf{s}_{J_k})$, where each $\mathbf{s}_j$ is a contiguous
sub-sequence of tokens. With local context size $m$ (we use $m=4$),
we compute
\[
\mathrm{score}_{\text{local}}(\mathbf{y}_k)
\;\triangleq\;
\frac{1}{J_k}\sum_{j=1}^{J_k}
\!\left(
\frac{1}{|\mathbf{s}_j|}
\sum_{u=1}^{|\mathbf{s}_j|}
\log \pi_{\boldsymbol{\theta}}\!\left(
s_{j,u}\,\big|\,\mathbf{x},\,
\mathbf{s}_{j-m:j-1},\,
\mathbf{s}_{j,<u}
\right)
\right),
\]
where $\mathbf{s}_{j-m:j-1}$ denotes up to $m$ immediately preceding
steps and $s_{j,u}$ is the $u$-th token of step $\mathbf{s}_j$. We
select $\mathbf{y}^{\star} = \arg\max_{k\in[K]}
\mathrm{score}_{\text{local}}(\mathbf{y}_k)$. Unlike GRAPE, this score
requires $m+1$ forward passes per trajectory (one per masked-context
rollout); see Appendix~\ref{app:complexity}.

\paragraph{Rank-Surprisal Ratio (RSR)\textsubscript{min}~\citep{yang2026reasoning}.}
RSR computes token-level surprisal and rank under the student. For
each token $y_t^k$ with context
$\mathbf{c}_t^k = (\mathbf{x}, \mathbf{y}_{<t}^k)$, define the
surprisal and rank as
\[
s_t^k \;\triangleq\;
-\log \boldsymbol{\pi}_t^{k}(y_t^k),
\qquad
r_t^k \;\triangleq\;
1 + \!\!\sum_{v\in\mathcal{V}}\!\!
\mathbb{I}\!\left[
\pi_{\boldsymbol{\theta}}(v \mid \mathbf{c}_t^k)
> \boldsymbol{\pi}_t^{k}(y_t^k)
\right].
\]
Ranks are clipped at $r_{\max}$ (we use $r_{\max}=100$):
$\bar r_t^k \triangleq \min(r_t^k, r_{\max})$. The trajectory-level
score is
\[
\mathrm{score}_{\text{RSR}}(\mathbf{y}_k)
\;\triangleq\;
\frac{\sum_{t=1}^{|a_k|}\bar r_t^k}{\sum_{t=1}^{|a_k|} s_t^k},
\qquad
\mathbf{y}^{\star} = \arg\min_{k\in[K]}
\mathrm{score}_{\text{RSR}}(\mathbf{y}_k).
\]
Surprisal and clipped rank are computed from a single forward pass;
clipping at $r_{\max}$ avoids the $O(|\mathcal{V}|)$ full sort while
remaining exact (Appendix~\ref{app:exp_impl}).

\subsection{Implementation Details}
\label{app:exp_impl}

\paragraph{SFT training.}
We use a single hyperparameter recipe across all three students
(Qwen-2.5-7B, Qwen-2.5-1.5B, Llama-3.2-3B) and across all
selection methods, so that any difference in downstream performance
can be attributed to the selection rule rather than to per-method
training tuning. Each student is fine-tuned with full-parameter SFT
under the following configuration:
\begin{itemize}
\item \textbf{Maximum sequence length:} $32{,}768$ tokens, kept
identical across selection scoring, SFT, and evaluation; trajectories
exceeding this length (less than $1\%$ of the pool) are truncated
from the right.
\item \textbf{Learning rate:} $1.0\times 10^{-5}$ for the main
$5{,}000$-trajectory experiments and $5.0\times 10^{-6}$ for the
$500$-trajectory verification of
Condition~\ref{ass:anchor-bound} in
Appendix~\ref{app:empirical-verification}, linearly scaled to the smaller
dataset to avoid early-step instability.
\item \textbf{Effective global batch size:} $64$ trajectories,
realized as $8$ GPUs $\times$ per-device batch $1$ $\times$ gradient
accumulation $8$.
\item \textbf{Epochs:} $5$.
\item \textbf{Optimizer:} AdamW (PyTorch implementation) with default
$(\beta_1,\beta_2,\varepsilon)=(0.9, 0.999, 10^{-8})$ and weight
decay $0$.
\item \textbf{Learning-rate schedule:} cosine decay with linear
warm-up over the first $5\%$ of total steps.
\item \textbf{Gradient clipping:} max-norm $1.0$, delegated to
DeepSpeed via \texttt{gradient\_clipping: auto}.
\item \textbf{Precision:} bf16 mixed precision throughout, with
fp32 communication for ZeRO collectives. We disable bf16
reduced-precision matmul accumulation
(\texttt{torch.backends.cuda.matmul.allow\_bf16\_reduced\_precision\_reduction}
$=$ \texttt{False}) to suppress occasional NaN gradients in the
LM-head backward on A100; fp16 is never used.
\item \textbf{Memory and throughput optimizations:} non-reentrant
gradient checkpointing, FlashAttention-2, sequence packing with
greedy bin-packing and reset \texttt{position\_ids}, and Liger fused
kernels (RoPE, RMSNorm, SwiGLU). The fused linear-CE kernel is
disabled because it is incompatible with sequence packing.
\end{itemize}

\paragraph{Software stack.}
Training and inference run inside a single conda environment with
PyTorch~$2.9.1$ (CUDA~$12.8$), Transformers~$4.57.1$,
PEFT~$0.18.0$, FlashAttention~$2.8.3$, Liger-Kernel, and DeepSpeed.
Distributed training is launched via
\texttt{torchrun --nproc\_per\_node=8} on top of the HuggingFace
\texttt{Trainer}, using DeepSpeed ZeRO Stage~$2$
(\texttt{zero2\_bf16.json}: \texttt{stage 2},
\texttt{communication\_data\_type fp32},
\texttt{reduce/allgather\_bucket\_size} $=5\times 10^{8}$). Under
ZeRO-$2$, optimizer states and gradients are sharded across the $8$
GPUs while parameters are replicated, which is sufficient for the
$\leq 8$B-parameter students considered. Evaluation uses
vLLM~$0.15.0$ for batched multi-sample decoding, with
SGLang~$0.5.9$ available as a fallback engine.

\paragraph{Hardware and wall-clock.}
All experiments run on a single node with $8\times$ NVIDIA A100~80GB
(SXM4) GPUs. Each $5{,}000$-trajectory SFT run completes in
roughly $3$--$5$ hours depending on the student size, and each
$500$-trajectory verification run for
Appendix~\ref{app:budget_scaling} completes in about $25$--$45$
minutes. 
Aggregating across $3$ students, $7$ selection baselines
(Appendix~\ref{app:exp_baselines}), and the LARK ablation grid (top-$B
\in \{1,3,5,10\}$ together with $\chi^2$- and KL-tempered weighting
variants), the total training budget is approximately
$1{,}600$ A100-hours. Evaluation is sharded across the same $8$
GPUs with \texttt{tensor\_parallel\_size}$=1$ per shard, with
generations merged across shards afterwards, adding on average
$\sim 15$ minutes per (student, benchmark) pair.

\paragraph{Prompt consistency and truncation.}
A single chat template is applied throughout the pipeline. The same
\texttt{tokenizer.apply\_chat\_template} call is used for
(i)~selection-time scoring (when computing the forward-pass quantities
$\hat\rho_k$ and $\hat g_k$ on candidate trajectories under
$\boldsymbol{\theta}_{\mathrm{ref}}$), (ii)~SFT example construction
(masking out the system and user turns so that only the assistant
turn contributes to the loss), and (iii)~evaluation prompting
(generation prompt appended via
\texttt{add\_generation\_prompt=True}). All three stages truncate to
the same $32{,}768$-token limit. The reasoning instruction
\texttt{"Please reason step by step, and put your final answer within
\textbackslash boxed\{\}"} is held fixed across selection, training,
and evaluation, ensuring that selection scores measured on
$\boldsymbol{\theta}_{\mathrm{ref}}$ remain meaningful for the
fine-tuned model under matched prompting.

\paragraph{Reproducibility.}
We fix the global seed to $42$ for
\texttt{random}, \texttt{numpy}, and \texttt{torch} in all
selection-scoring and Condition~\ref{ass:anchor-bound} verification
experiments, and rely on the HuggingFace \texttt{Trainer} default
seed of $42$ for SFT. All headline numbers in
Table~\ref{tab:main_weighted_math_merged} are reported as the mean $\pm$ standard deviation over $3$ independent
evaluation runs that differ only in the sampling seed of the inference
engine, with decoding hyperparameters \texttt{temperature}$=0.6$,
\texttt{top\_p}$=0.95$, \texttt{repetition\_penalty}$=1.1$, and
maximum generation length $32{,}768$, evaluated under ACC@$5$ with
$5$ samples per problem (Appendix~\ref{app:exp_eval}). To facilitate
reproduction we will release, under an MIT licence, the full training
code, the YAML configuration files
(\texttt{train/configs/sft\_full.yaml},
\texttt{train/configs/sft\_full\_500.yaml},
\texttt{train/configs/zero2\_bf16.json}), the precomputed
per-trajectory scores
(\texttt{data/Q1/<student>/grape\_g/scores.json}), and the selection
rules used in each ablation.

\subsection{Evaluation Protocol}
\label{app:exp_eval}

\paragraph{Benchmarks.}
We evaluate on the same four reasoning benchmarks reported in
Section~\ref{sec:exp_setup}, which span competition mathematics,
grade-level mathematics, and graduate-level science:
AIME-2024~\citep{aops_aime} ($30$ problems from the $2024$ American
Invitational Mathematics Examination),
AMC~\citep{maa_amc} ($83$ problems from the American Mathematics
Competitions),
MATH-500~\citep{hendrycks2021measuring} ($500$ problems sampled from the
MATH test set, which covers algebra, geometry, number theory,
probability, and precalculus), and
GPQA-Diamond~\citep{rein2024gpqa} ($198$ multiple-choice questions
written by domain experts in physics, chemistry, and biology). The
four benchmarks together cover both numeric-answer and
multiple-choice formats, which lets us assess whether the
improvements brought by LARK transfer beyond the mathematical
reasoning domain on which the candidate pool is constructed.
The Avg column in Table~\ref{tab:main_weighted_math_merged} is the
unweighted mean of these four benchmarks; because the four splits
differ in the number of problems, the per-benchmark numbers remain
the primary basis of comparison.

\paragraph{Decoding configuration.}
For every benchmark and every fine-tuned student, we generate $5$
independent samples per problem with the following decoding
parameters:
\begin{itemize}
  \item temperature $= 0.6$
  \item top-$p = 0.95$
  \item top-$k = -1$ (no top-$k$ truncation)
  \item maximum generation length $=$ $32{,}768$ tokens (matching the
        SFT context length)
\end{itemize}
The same decoding configuration is used across all baselines and
LARK to ensure that any difference in downstream accuracy is
attributable to the selection rule rather than to inference-time
differences.

\paragraph{Metric (ACC@5).}
We report \textbf{ACC@5} as our primary evaluation metric. For each
test problem, we sample $5$ independent generations under the
decoding configuration above; the problem is marked \emph{correct}
if any of the $5$ generations contains the correct final answer
under the extraction rule described below. ACC@5 captures whether
the student model is capable of producing a correct reasoning
trajectory for the problem within a small sampling budget, and is
less sensitive to single-sample variance than greedy ACC@1. Each
experiment is repeated under three independent decoding seeds, with
the selected training data and the SFT seed held fixed; we report
the mean and standard deviation of ACC@5 across these three
decoding seeds, so the reported variability captures inference-time
sampling noise rather than training-seed noise.

\paragraph{System prompts.}
We use a fixed system prompt per benchmark family, distinguishing
numeric-answer tasks (AIME, AMC), multiple-choice tasks (GPQA), and
free-form-answer tasks (MATH-500). All three prompts share the same
overall structure --- enforcing English output, separating the
reasoning chain from the final answer, and requiring the answer in a
\verb|\boxed{...}| environment for unambiguous extraction --- and
differ only in the answer-format specification.

\begin{prompt}[title={System Prompt: AIME / AMC (Numeric Answer)}]
You are an AI mathematician. All content you output MUST be in English. Use the question to deduce the correct numeric answer.

\textbf{Finish all your reasoning, then on a NEW line output only \texttt{boxed\{\textless number\textgreater\}}.}

The content inside \texttt{boxed\{\textless number\textgreater\}} must be the final numeric answer only, with no expressions, variables, or additional text.
\end{prompt}

\begin{prompt}[title={System Prompt: GPQA (Multiple Choice)}]
You are an AI mathematician. All content you output MUST be in English. Use the question to deduce the correct choice.

\textbf{Finish all your reasoning, then on a NEW line output only \texttt{boxed\{\textless letter\textgreater\}}.}

The content inside \texttt{boxed\{\textless letter\textgreater\}} must be exactly one of A, B, C, D.
\end{prompt}

\begin{prompt}[title={System Prompt: MATH (Free-form Answer)}]
You are an AI mathematician. All content you output MUST be in English. Use the question to deduce the correct answer.

\textbf{Finish all your reasoning, then on a NEW line output only \texttt{boxed\{\textless answer\textgreater\}}.}
\end{prompt}

\paragraph{Answer extraction rule.}
We extract the predicted final answer from each generation using the
following deterministic rule:
\begin{itemize}
  \item If the generation contains at least one \verb|\boxed{...}|
        token, we take the content inside the \emph{last}
        \verb|\boxed{...}| as the predicted answer. The last (rather
        than first) occurrence is used because models occasionally
        produce intermediate boxed expressions during reasoning;
        only the final boxed value is treated as the answer.
  \item If the generation contains no \verb|\boxed{...}| token, the
        generation is treated as \emph{incorrect}, regardless of
        whether the surrounding text contains a correct value.
\end{itemize}
For numeric-answer tasks (AIME, AMC), the extracted string is
canonicalized by stripping surrounding whitespace and converting
fractions $a/b$ to their decimal form before comparison with the
ground truth. For multiple-choice tasks (GPQA), the extracted string
is uppercased and compared against the gold letter in
$\{\mathrm{A}, \mathrm{B}, \mathrm{C}, \mathrm{D}\}$. For free-form
answers (MATH-500), we follow the canonicalization rules of
\citet{hendrycks2021measuring}, which normalize fraction, square-root,
and exponent notation prior to string matching.

\paragraph{Efficient scoring for rank-based metrics.}
The RSR baseline of Appendix~\ref{app:exp_baselines} requires the
clipped rank $\bar r_t^k = \min(r_t^k, r_{\max})$ at every token
position. A naive implementation would sort the full
$|\mathcal{V}|$-dimensional logit vector at each token, costing
$O(|\mathcal{V}| \log |\mathcal{V}|)$ per token. We avoid this by
observing that under rank clipping at $r_{\max}$, only the relative
order of the top-$r_{\max}$ logits matters: any token whose true
rank exceeds $r_{\max}$ is mapped to $r_{\max}$ regardless of its
exact rank. We therefore extract the top-$r_{\max}$ logits via
\texttt{torch.topk} (cost $O(|\mathcal{V}| \log r_{\max})$ per token),
check whether the gold token $y_t^k$ is among them, and assign
$\bar r_t^k$ accordingly. This procedure is exact under rank clipping
and reduces the per-token cost by a factor of
$\log |\mathcal{V}| / \log r_{\max}$, which is roughly
$5\times$--$6\times$ for $|\mathcal{V}| \approx 1.5 \times 10^5$ and
$r_{\max} = 100$. The same forward pass that produces the surprisal
$s_t^k = -\log \boldsymbol{\pi}_t^k(y_t^k)$ is reused to produce the
top-$r_{\max}$ logits, so RSR scoring requires only one forward pass
per candidate trajectory.

\paragraph{Prompt consistency.}
We use a consistent chat template across (i) selection-time scoring
(when the student model evaluates each candidate trajectory under
teacher forcing), (ii) SFT loss formatting (when the model is
fine-tuned on the selected trajectories), and (iii) evaluation-time
prompting (when the fine-tuned model is queried on benchmark
problems). Using the same chat template at all three stages reduces
train--test mismatch and ensures that selection scores measured on
$\boldsymbol{\theta}_{\mathrm{ref}}$ remain meaningful for the
fine-tuned model. All inputs are truncated to the same maximum
sequence length of $32{,}768$ tokens used for SFT.

\subsection{Train/Eval Contamination Audit}
\label{app:contamination}

To ensure that the downstream evaluation results in Table~\ref{tab:main_weighted_math_merged} reflect genuine generalization rather than memorization of training-time supervision, we audit the overlap between the $5{,}000$-problem NuminaMath candidate pool of Appendix~\ref{app:exp_pool} and every problem in the evaluation benchmarks of Appendix~\ref{app:exp_eval}.

\paragraph{Method.}
For each evaluation problem we compute the character $8$-gram Jaccard similarity against every problem in the candidate pool, after a normalization step that lowercases the text, collapses whitespace, and strips redundant LaTeX spacing macros. We report two statistics: (i) \emph{exact matches} (Jaccard $=1$ after normalization), which would indicate verbatim duplication; and (ii) \emph{near-matches} at three progressively stricter thresholds, $J \geq 0.4$, $J \geq 0.7$, and $J \geq 0.85$, to capture different degrees of surface overlap. The metric is sensitive to both surface and structural duplication: an exact-template variant with all numerals changed typically scores $J \in [0.7, 0.95]$, while a fully independent problem on the same topic typically scores $J < 0.2$.

\paragraph{Results.}
Table~\ref{tab:contamination} summarizes the audit across all four evaluation benchmarks (AIME-2024, AMC, GPQA-Diamond, and MATH-500, broken down by difficulty levels L1--L5). The audit detects $\mathbf{0}$ exact matches across all $811$ evaluation problems. AIME-2024, AMC, GPQA-Diamond, and MATH-L1 contain $0$ near-matches at every threshold; GPQA-Diamond is graduate-level physics, chemistry, and biology with no domain overlap with the math-only NuminaMath training pool, while AIME, AMC, and MATH-L1 simply do not appear in the pool under any threshold. MATH-L2 through L5 contain $31$ near-matches at $J \geq 0.4$ ($\sim 3.7\%$ of the MATH-500 split), of which only $1$ case at $J \geq 0.85$.

\begin{table}[h]
\centering
\small
\caption{Train/eval contamination audit. For each evaluation benchmark, we report the number of problems whose character $8$-gram Jaccard similarity against the NuminaMath candidate pool of Appendix~\ref{app:exp_pool} exceeds the indicated threshold. \emph{Exact} denotes Jaccard $= 1$ after text normalization. MATH-L1 through MATH-L5 are the five difficulty levels of MATH-500.}
\label{tab:contamination}
\setlength{\tabcolsep}{8pt}
\begin{tabular}{lrrrrr}
\toprule
\textbf{Benchmark} & $N$ & \textbf{Exact} & $J \geq 0.4$ & $J \geq 0.7$ & $J \geq 0.85$ \\
\midrule
AIME-2024     & $30$  & $0$ & $0$  & $0$ & $0$ \\
AMC           & $83$  & $0$ & $0$  & $0$ & $0$ \\
GPQA-Diamond  & $198$ & $0$ & $0$  & $0$ & $0$ \\
MATH-L1       & $43$  & $0$ & $0$  & $0$ & $0$ \\
MATH-L2       & $90$  & $0$ & $2$  & $1$ & $1$ \\
MATH-L3       & $105$ & $0$ & $7$  & $1$ & $0$ \\
MATH-L4       & $128$ & $0$ & $16$ & $5$ & $0$ \\
MATH-L5       & $134$ & $0$ & $6$  & $1$ & $0$ \\
\midrule
\textbf{Total} & $\mathbf{811}$ & $\mathbf{0}$ & $\mathbf{31}$ & $\mathbf{8}$ & $\mathbf{1}$ \\
\bottomrule
\end{tabular}
\end{table}

\paragraph{Categorization of MATH-L2--L5 near-matches.}
Manual inspection of all $31$ flagged pairs places them in three categories, none of which constitute contamination:

\begin{enumerate}[leftmargin=*]
    \item \textbf{Template-shared, numerically perturbed.} Same algebraic template, different constants or signs, leading to a different polynomial and a different answer. The highest-similarity case (\texttt{MATH-L2 \#11}, $J = 0.869$) asks for the smallest $n$ such that all roots of $z^4 + z^2 + 1 = 0$ are $n$-th roots of unity, while the closest training problem asks the same question for $z^4 - z^2 + 1 = 0$.
    
    \item \textbf{Template-shared, different question on the same setup.} Same physical or geometric setup, but a different quantity is asked, with no derivable relation between the two answers. For example, \texttt{MATH-L5 \#46} ($J = 0.693$) describes a hot-air balloon held by four ropes anchored at points $A,B,C,D$; the training problem on the same setup asks for the length of $OH$, whereas the evaluation problem asks for the greatest rope length saved by replacing $HC + HD$ with a single rope at a chosen point. These are companion problems from the parent MATH dataset.
    
    \item \textbf{Surface-level template variation.} Same problem stem with different numerical inputs (e.g., coefficients in a quadratic completion task). \texttt{MATH-L4 \#38} ($J = 0.664$) asks to rewrite $x^2 + 2.6x + 3.6$ in the form $(x+b)^2 + c$, while the training problem performs the same operation on $x^2 - 20x + 36$. The procedure is identical but the inputs and outputs differ.
\end{enumerate}

In none of the $31$ flagged cases does the answer to a training problem transfer to the corresponding evaluation problem, and no exact duplicates of any AIME-2024, AMC, GPQA-Diamond, or MATH-500 problem appear in the training pool. The downstream Acc@5 results in Table~\ref{tab:main_weighted_math_merged} therefore reflect generalization to held-out problems rather than retrieval of memorized supervision.

\section{Diagnostic Analyses of LARK}
\label{app:diagnostics}

This appendix collects additional evaluation results and diagnostic
analyses that validate the design choices of LARK beyond the main
downstream comparison.
Appendix~\ref{app:complexity} verifies the efficiency
claim through both an asymptotic complexity derivation and a wall-clock
measurement on $8$ A100 GPUs, showing that LARK matches the lowest
cost achieved by any per-trajectory scoring baseline.
Appendix~\ref{app:budget_scaling} sweeps the selection budget $B$ from 1 to 20 on a 500-problem diagnostic subset and studies how performance changes with the budget. Appendix~\ref{app:case_study} zooms into a
single training problem to illustrate, end to end, the score landscape
over candidate trajectories, the token-level structure of the chosen
trajectory, the $\chi^2$-$B$ weighting it induces, and the resulting
weighted SFT objective for that problem.

\subsection{Computational Cost}
\label{app:complexity}

We separate every per-trajectory score into two parts: a single
forward pass over the trajectory through the student model, with cost
$C_{\mathrm{fwd}}(|a_k|)$, and a per-token post-processing step on
the resulting logits. Decomposing the total cost in this way exposes
the differences across methods, which would otherwise be hidden by
the dominant forward-pass term.

\paragraph{Per-trajectory cost decomposition.}
A single forward pass through a transformer with $L$ layers, hidden
dimension $d$, and vocabulary size $|\mathcal{V}|$ on a trajectory of
length $|a_k|$ has cost
\begin{align*}
C_{\mathrm{fwd}}(|a_k|)
= \underbrace{\mathcal{O}(L|a_k| d^2)}_{\text{MLP}}
+ \underbrace{\mathcal{O}(L|a_k|^2 d)}_{\text{attention}}
+ \underbrace{\mathcal{O}(|a_k| \cdot |\mathcal{V}| \cdot d)}_{\text{LM head}}.
\end{align*}
After this forward pass, every per-trajectory selection method applies
a post-processing operation to the logits at each token. GRAPE computes
the cross-entropy of the gold token from a $|\mathcal{V}|$-dimensional
log-sum-exp, with per-token cost $\mathcal{O}(|\mathcal{V}|)$. LARK
computes both the cross-entropy and the Brier component
$\sum_v \boldsymbol{\pi}_t^k(v)^2$, again with per-token cost
$\mathcal{O}(|\mathcal{V}|)$. RSR additionally selects the
top-$r_{\max}$ entries of each token distribution to form a partial
rank statistic, with per-token cost
$\mathcal{O}(|\mathcal{V}| \log r_{\max})$. Local Naturalness instead
runs $m+1$ independent forward passes per trajectory (one for each
masked rollout), each followed by an $\mathcal{O}(|\mathcal{V}|)$
per-token operation.

\paragraph{Total complexity.}
Let $N$ denote the number of training problems and $K$ the number of
candidate trajectories per problem. Aggregating the forward and the
post-processing terms over all $NK$ candidates yields the total cost
of each method:
\begin{align*}
\text{GRAPE}\!: \quad &
\mathcal{O}\!\big(NK \cdot [C_{\mathrm{fwd}}(|a_k|) + |a_k| \cdot |\mathcal{V}|]\big), \\[2pt]
\text{LARK}\!: \quad &
\mathcal{O}\!\big(NK \cdot [C_{\mathrm{fwd}}(|a_k|) + |a_k| \cdot |\mathcal{V}|]\big), \\[2pt]
\text{RSR}\!: \quad &
\mathcal{O}\!\big(NK \cdot [C_{\mathrm{fwd}}(|a_k|) + |a_k| \cdot |\mathcal{V}| \log r_{\max}]\big), \\[2pt]
\text{Local Naturalness}\!: \quad &
\mathcal{O}\!\big(NK \cdot [(m+1)\,C_{\mathrm{fwd}}(|a_k|) + |a_k| \cdot |\mathcal{V}|]\big).
\end{align*}
LARK matches the lowest cost achieved by GRAPE. RSR adds a
$\log r_{\max}$ factor on the per-token term because of its rank
selection step, and Local Naturalness multiplies the dominant forward
pass by $m+1$ because of its repeated rollouts. The strict ordering
$\text{GRAPE} = \text{LARK} < \text{RSR}$ and
$\text{GRAPE} = \text{LARK} < \text{Local Naturalness}$ follows from
$|a_k| \cdot |\mathcal{V}| < |a_k| \cdot |\mathcal{V}| \log r_{\max}$
and from $C_{\mathrm{fwd}} < (m+1)\,C_{\mathrm{fwd}}$ for $m \geq 1$.
We do not assign a model-independent asymptotic formula to
LLM-judged Quality because its cost depends on the judge model and
serving setup; we therefore report it in the wall-clock comparison.

\paragraph{Empirical wall-clock cost.}
We benchmark each scoring method by measuring its per-sample wall-clock
cost on $8$ A100 GPUs.\footnote{Random, Token Length, and Rule-based
Quality require no forward pass through the student and run in
negligible time, so we omit them from the figure.}
Figure~\ref{fig:computational_cost} reports the resulting cost across
all five scoring methods. LARK ($1.32$\,s/sample) matches GRAPE
($1.28$\,s/sample) to within $3\%$, which reflects the fact that they
share the same forward pass and differ only by an additional
$\mathcal{O}(|\mathcal{V}|)$ Brier sum per token. LLM-judged Quality
is $1.6\times$ slower at $2.11$\,s/sample because of the auxiliary
judge call. RSR is $2.7\times$ slower at $3.44$\,s/sample, which is
consistent with the extra $\log r_{\max}$ rank-selection factor in its
per-token operation. Local Naturalness is the most expensive at
$5.06$\,s/sample, which is $4.0\times$ slower than LARK and matches
the $(m+1)$-fold forward-pass cost predicted by the asymptotic
analysis. The empirical ordering matches the constant-factor ranking
implied by the complexity formulas.

\begin{figure}[t]
\centering
\includegraphics[width=0.7\linewidth]{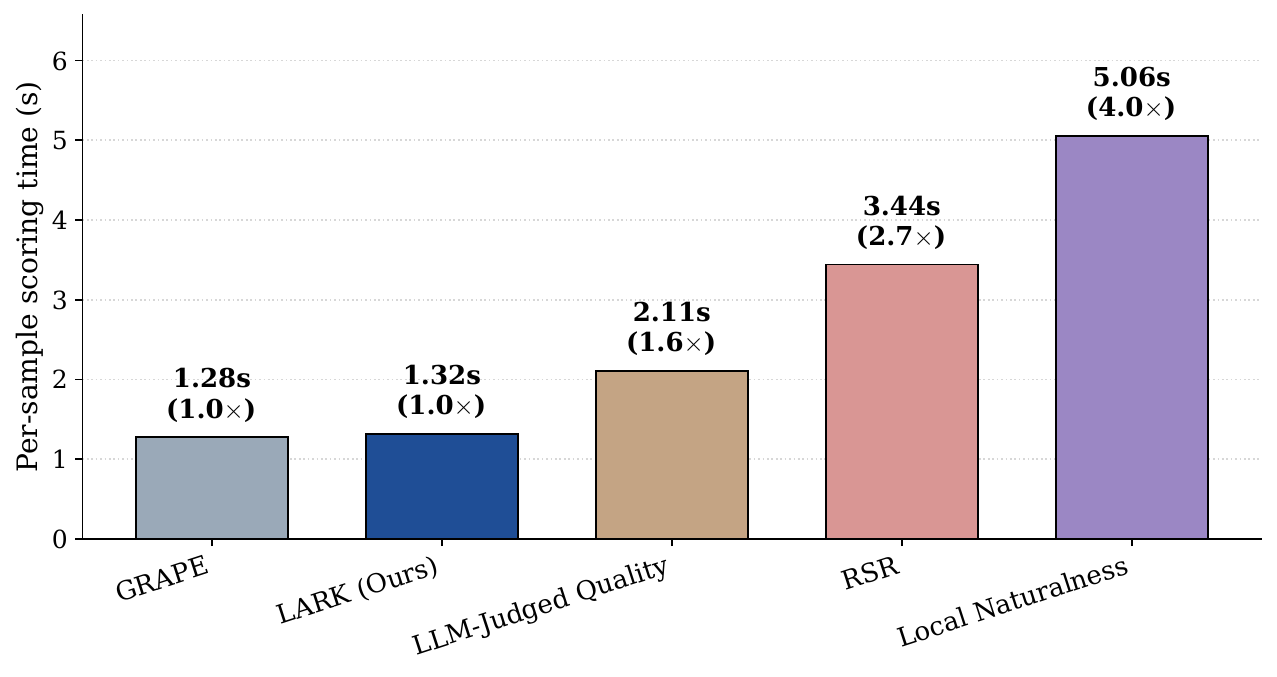}
\caption{\textbf{Per-sample wall-clock cost of trajectory scoring} on
$8$ A100 GPUs. The number in parentheses is the slowdown relative to
the fastest method. LARK is essentially as fast as GRAPE,
$1.6\times$ faster than LLM-judged Quality, $2.7\times$ faster than
RSR, and $4.0\times$ faster than Local Naturalness. Methods that
require no forward pass through the student (Random, Token Length,
Rule-based Quality) are omitted because their cost is negligible.}
\label{fig:computational_cost}
\end{figure}

\subsection{Selection Budget Scaling}
\label{app:budget_scaling}

We study how LARK behaves as the number of selected trajectories per
question changes. This analysis sweeps the selection budget
$B \in \{1, 3, 5, 10, 20\}$ on Qwen-2.5-7B and evaluates the resulting
student on AMC Acc@5. Because this sweep is conducted on the same
500-problem subset used for our diagnostic experiments, rather than the
full 5{,}000-problem training set used in
Table~\ref{tab:main_weighted_math_merged}, these results are intended
to show the relative trend across budgets rather than to replace the
main evaluation results.

For each budget $B$, we apply the LARK selection rule to choose
$B$ trajectories per problem, fine-tune Qwen-2.5-7B with the same SFT
protocol described in Appendix~\ref{app:exp_details}, and report AMC
Acc@5 in Figure~\ref{fig:budget_scaling}. LARK remains effective across
the entire budget range: accuracy increases from $49.4\%$ at $B{=}1$
to $65.1\%$ at $B{=}3$, reaches its highest value of $68.7\%$ at
$B{=}10$, and remains strong at $66.3\%$ when $B{=}20$.

The trend is not strictly monotonic. Moving from $B{=}1$ to $B{=}3$
brings a large improvement, indicating that using multiple selected
trajectories can provide complementary supervision. However, increasing
the budget beyond $B{=}3$ gives smaller and less stable gains. In
particular, $B{=}5$ underperforms $B{=}3$, while $B{=}10$ gives the
best result and $B{=}20$ slightly decreases. This pattern suggests that
a small number of carefully selected trajectories already captures most
of the useful training signal, while adding more lower-ranked
trajectories can introduce weaker or less informative supervision.

\begin{figure}[t]
\centering
\includegraphics[width=0.55\linewidth]{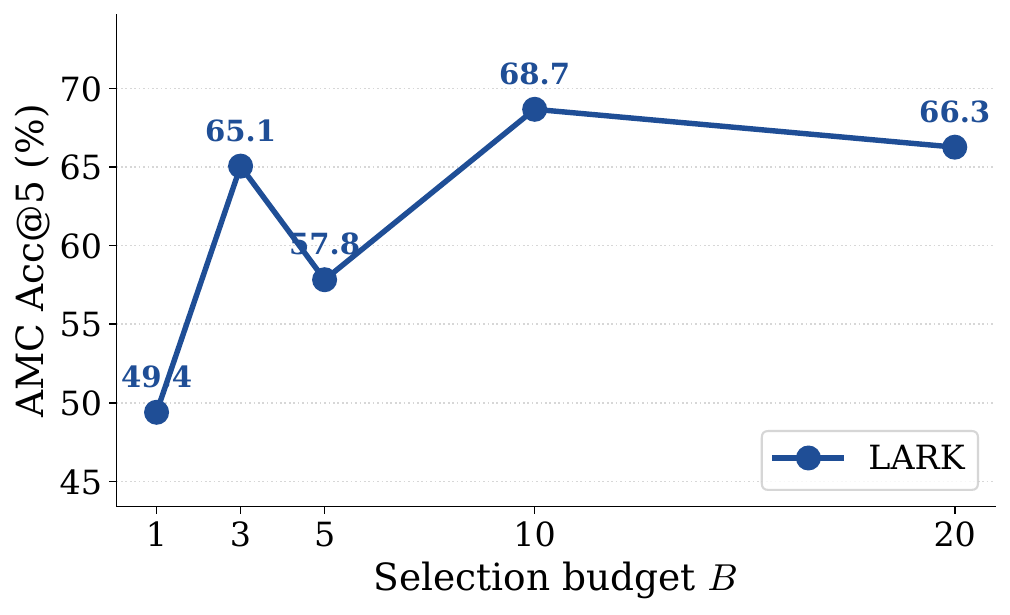}
\caption{\textbf{Selection budget scaling of LARK} on Qwen-2.5-7B,
evaluated on AMC after fine-tuning on the 500-problem subset. LARK
achieves strong performance across
$B \in \{1, 3, 5, 10, 20\}$, with the best result obtained at
$B{=}10$.}
\label{fig:budget_scaling}
\end{figure}

Overall, the budget sweep shows that LARK is not sensitive to a single
fixed choice of $B$. The method performs well with a small budget
($B{=}3$) and remains competitive as the budget increases. This supports
the use of a small top-$B$ selection strategy in the main experiments,
where the goal is to improve distillation efficiency while avoiding
unnecessary supervision from less informative trajectories.
\subsection{Case Study: A Single-Problem Walkthrough}
\label{app:case_study}

While Appendices~\ref{app:complexity} and~\ref{app:budget_scaling} support the efficiency and budget-scaling analyses of LARK at the aggregate level, this appendix zooms into a single training problem to illustrate the end-to-end behavior of the selection rule. We walk through (i) the problem and its candidate pool, (ii) the LARK score landscape, (iii) the top-1 trajectory selected by LARK, and (iv) the closed-form $\chi^2$-$B$ weights at $B=3$.

\paragraph{Problem and candidate pool.}
We illustrate on problem $\text{pid}=839$ from the candidate pool of Appendix~\ref{app:exp_pool}:

\begin{prompt}[title={User Question (pid$=839$)}]
If $\mathbf{a}$, $\mathbf{b}$, and $\mathbf{c}$ are vectors such that $\|\mathbf{a}\| = \|\mathbf{b}\| = 1$, $\|\mathbf{a} + \mathbf{b}\| = \sqrt{3}$, and
\[
\mathbf{c} - \mathbf{a} - 2 \mathbf{b} = 3 (\mathbf{a} \times \mathbf{b}),
\]
then find $\mathbf{b} \cdot \mathbf{c}$.
\end{prompt}

\noindent The ground-truth answer is $\boxed{\dfrac{5}{2}}$. The pool contains $K = 33$ candidate trajectories produced by 11 teacher models with 3 rollouts each; all reach the correct final answer per the construction of Appendix~\ref{app:exp_pool}.

\paragraph{Score landscape.}
Figure~\ref{fig:case_study_scores} shows the LARK score $\hat g_k$ for all 33 candidates, grouped by teacher. The top-1 trajectory is the second rollout of \texttt{phi4-reason-plus} (highlighted in red, $\hat g = 0.01782$); the same teacher's third rollout takes rank 2 ($\hat g = 0.01687$). Notably, \texttt{phi4-reason-plus} rollout 1 ranks only 25 ($\hat g = 0.00960$): the within-teacher spread on this teacher alone is $1.86\times$, exceeding the spread between several pairs of teachers. The same pattern recurs across the pool (e.g., \texttt{nemotron-super-v15} rollouts span $0.0052$ to $0.0078$), which underscores that LARK selects at the trajectory level rather than at the teacher level. Teachers commonly perceived as strong on aggregate benchmarks (\texttt{deepseek-r1-0528}, \texttt{qwen3-235b-2507}, \texttt{qwq-32b}) do not contribute any top-3 trajectory on this problem, with their rollouts clustering near the median.

\begin{figure}[h]
\centering
\includegraphics[width=\linewidth]{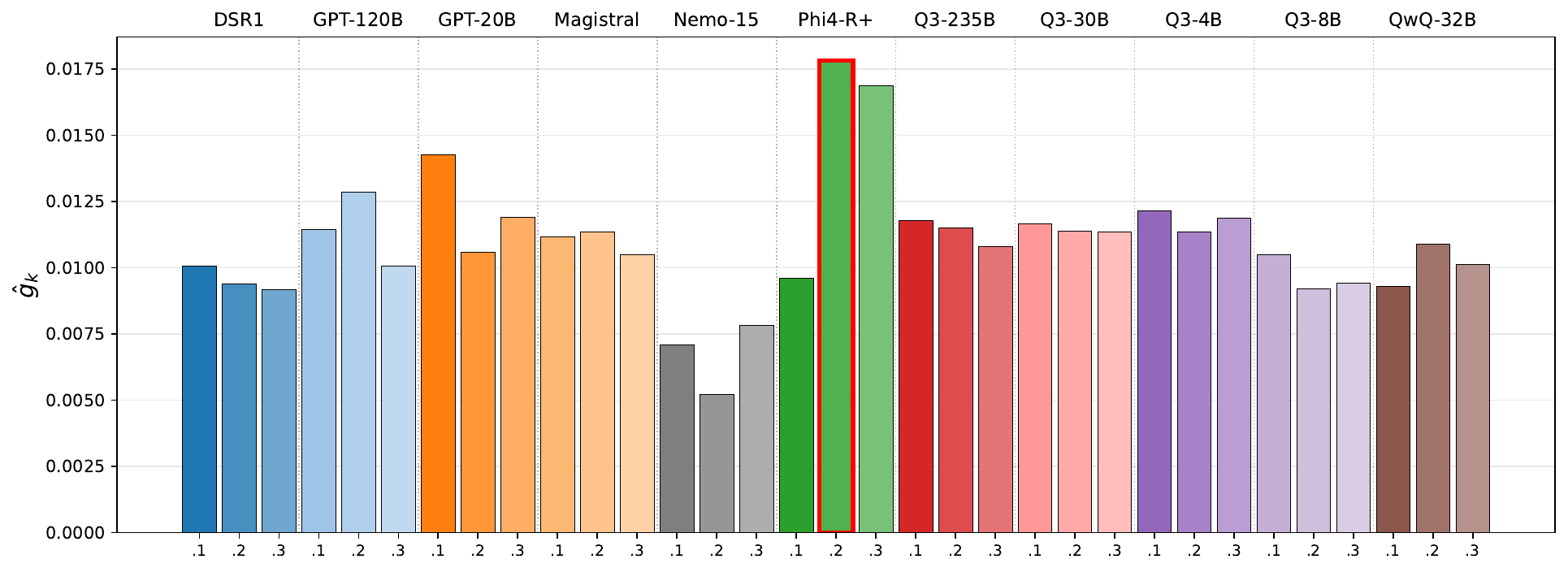}
\caption{$\hat g_k$ for all 33 candidate trajectories on problem $\text{pid}=839$, grouped by teacher. Each teacher contributes 3 rollouts (indexed .1/.2/.3 along the $x$-axis). The top-1 trajectory selected by LARK (\texttt{phi4-reason-plus} rollout 2, $\hat g = 0.01782$) is highlighted with a red border. Within-teacher variation is comparable to between-teacher variation.}
\label{fig:case_study_scores}
\end{figure}

\paragraph{Top-1 trajectory selected by LARK.}
We quote the top-1 trajectory (\texttt{phi4-reason-plus} rollout 2, $\hat g = 0.01782$) in full:

\begin{prompt}[title={Top-1 trajectory selected by LARK (\texttt{phi4-reason-plus} rollout 2)}]
\small
\texttt{<think>}\\
We are given:
Vectors $a$, $b$, $c$. Conditions:
$\|a\|=\|b\|=1$, and $\|a+b\|=\sqrt{3}$. Also $c - a - 2b = 3(a \times b)$.
We need find $b \cdot c$.

Step: First, note: $\|a+b\|^2 = (a+b)^2 = a^2 + 2 a \cdot b + b^2 = 1 + 2\cos\theta + 1 = 2 + 2\cos\theta = 3$. So $2 + 2\cos\theta = 3 \Rightarrow \cos\theta = 1/2$ so $a \cdot b = 1/2$.

Now, note: $a \times b$ is perpendicular to both $a$ and $b$. So it is orthogonal to $a$, $b$. Then, $c = a + 2b + 3(a \times b)$. Dotting with $b$: $b \cdot c = b \cdot a + 2 b \cdot b + 3 b \cdot (a \times b) = a \cdot b + 2 \cdot 1 + 3 \cdot 0 = 1/2 + 2 = 5/2$. So answer: $5/2$.

Thus final answer: $5/2$. We'll produce answer in box: \texttt{\textbackslash{}boxed\{5/2\}}.\\
\texttt{</think>}\\
\texttt{\textbackslash{}boxed\{\textbackslash{}frac\{5\}\{2\}\}}
\end{prompt}

\noindent The trajectory is substantially more compact than the average candidate in this pool. Importantly, its length-normalized cross-entropy $\ell_k$ under $\theta_{\text{ref}}$ is \emph{not} the smallest in the pool; what raises $\hat g_k$ appears to be the structure of the residual rather than low loss alone. A qualitative inspection suggests that high-residual tokens around numeric expressions and key algebraic steps carry a large portion of the cumulative Brier mass, matching the heuristic reading of $\hat g_k$ in Remark~\ref{rem:ghat-meaning}.

\paragraph{$\chi^2$-$B$ weighting at $B=3$.}
With $B=3$, the closed-form $\chi^2$-$B$ rule gives the weights from the score margins above the threshold $\hat g_{(B+1)}$. The four relevant scores on this problem are:
\begin{align*}
\hat g_{(1)} &= 0.01782 \quad (\texttt{phi4-reason-plus}.2), \\
\hat g_{(2)} &= 0.01687 \quad (\texttt{phi4-reason-plus}.3), \\
\hat g_{(3)} &= 0.01426 \quad (\texttt{gptoss-20b-high}.1), \\
\hat g_{(4)} &= 0.01285 \quad (\texttt{gptoss-120b-high}.2;\ \text{threshold}).
\end{align*}
The score margins above the threshold are
\begin{equation*}
\hat g_{(1)} - \hat g_{(4)} = 0.00497, \qquad
\hat g_{(2)} - \hat g_{(4)} = 0.00402, \qquad
\hat g_{(3)} - \hat g_{(4)} = 0.00141,
\end{equation*}
summing to $\sum_{j=1}^{3} (\hat g_{(j)} - \hat g_{(4)}) = 0.01040$. The closed-form temperature is therefore $\tau^*(B) = 0.01040 / K \approx 3.15 \times 10^{-4}$, and the LARK weights are
\begin{equation*}
\hat q_{(1)} = 0.478, \qquad \hat q_{(2)} = 0.387, \qquad \hat q_{(3)} = 0.136.
\end{equation*}
The contrast with hard top-$B$ truncation (uniform weight $1/B \approx 0.333$ each) is substantial: LARK assigns $43\%$ more weight to the rank-1 trajectory than uniform top-3 would, and only $41\%$ as much weight to the rank-3 trajectory. The asymmetry directly reflects the score gap between the two \texttt{phi4-reason-plus} rollouts and the remainder of the pool. The resulting per-problem contribution to the SFT loss is
\begin{equation*}
\mathcal{L}_{\text{SFT}}^{(\text{pid}=839)}(\theta) 
= 0.478\,\ell(\theta; \mathbf y_{16}) 
+ 0.387\,\ell(\theta; \mathbf y_{17}) 
+ 0.136\,\ell(\theta; \mathbf y_{6}),
\end{equation*}
with all other 30 candidates carrying zero weight on this problem.

\paragraph{Summary.}
The walkthrough illustrates three points anchored at a concrete problem: trajectory-level (rather than teacher-level) variation in $\hat g_k$ can be the dominant source of selection signal; the LARK-selected trajectory is not the one with the lowest $\ell_k$ but the one with structured residual signal; and the $\chi^2$-$B$ weighting deviates substantially from uniform top-$B$ truncation when the top-ranked trajectory is sharply separated from the threshold.
\section{Limitations}
\label{app:limitations}

LARK provides a practical and theoretically motivated approach for
trajectory-level selection in reasoning distillation. Its forward-pass
score makes the method efficient to apply to large candidate pools, and
the closed-form weighting rule avoids additional tuning while preserving
the benefits of multi-trajectory supervision. The empirical results
across multiple student models and benchmarks suggest that learnability
is a useful principle for selecting reasoning trajectories.

There are two natural limitations of the present study. First, our
experiments use a correctness-verified candidate pool, where all retained
trajectories reach the correct final answer. This setting isolates the
trajectory-selection problem from the separate problem of correctness
verification, but future work could study LARK in noisier pools that
include incorrect or partially correct teacher trajectories. Second, due
to the computational cost of full-parameter SFT, our reported standard
deviations are computed over independent decoding seeds while keeping
the training configuration fixed. Running multiple independent SFT
seeds would provide an even more complete estimate of training-time
variance.


\end{document}